\documentclass[a4paper]{article}
\pdfoutput=1

\usepackage[T1]{fontenc}
\usepackage[utf8]{inputenc}
\usepackage[english]{babel}
\usepackage{lmodern}

\usepackage[final]{microtype}

\usepackage{amsmath,amssymb,amsthm}

\usepackage{mathrsfs}

\usepackage[final]{graphicx}

\usepackage{subcaption}

\usepackage[a4paper,top=2.5cm,right=2.0cm,bottom=2.0cm,left=2.0cm,centering,pdftex]{geometry}

\usepackage{authblk}

\usepackage[strict]{changepage}

\usepackage[numbers,sort]{natbib}

\usepackage{xspace}

\usepackage{xifthen}

\usepackage{multirow}

\usepackage[final]{hyperref}
\hypersetup{%
pdfauthor={Thomas Lukasiewicz, Enrico Malizia},%
pdftitle={Complexity Results for Preference Aggregation over (m)CP-nets: Pareto and Majority Voting},%
pdfsubject={Social Choice / Voting, Computational Complexity of Reasoning, Preferences},%
pdfkeywords={Preferences, Combinatorial preferences, Preference aggregation, CP-nets, mCP-nets, Voting, Pareto voting, Majority voting, Computational complexity, Polynomial Hierarchy}%
}

\usepackage{doi}

\usepackage[capitalise,noabbrev]{cleveref}

\newcommand{\Wlog}{w.l.o.g.\ }

\renewcommand{\emptyset}{\varnothing}
\newcommand{\ol}{\overline}
\newcommand{\valtrue}{\textnormal{\texttt{true}}\xspace}
\newcommand{\valfalse}{\textnormal{\texttt{false}}\xspace}
\newcommand{\CPNet}{CP\nobreakdash-net\xspace}
\newcommand{\MCPNet}{\ensuremath{m}CP\nobreakdash-net\xspace}
\newcommand{\CPNets}{CP\nobreakdash-nets\xspace}
\newcommand{\MCPNets}{\ensuremath{m}CP\nobreakdash-nets\xspace}
\newcommand{\SetFeat}{\mathcal}

\newcommand{\Pref}{\succ}
\newcommand{\Ferp}{\prec}
\newcommand{\Incomp}{\bowtie}

\newcommand{\PrefNet}[1]{\Pref_{#1}}
\newcommand{\FerpNet}[1]{\Ferp_{#1}}
\newcommand{\IncompNet}[1]{\Incomp_{#1}}

\newcommand{\PrefRanking}[1]{\Pref_{#1}}
\newcommand{\FerpRanking}[1]{\Ferp_{#1}}
\newcommand{\IncompRanking}[1]{\Incomp_{#1}}

\newcommand{\PrefParetoMNet}[1]{\Pref^{\mathit{p}}_{#1}}
\newcommand{\PrefParetoProfile}[1]{\Pref^{\mathit{p}}_{#1}}
\newcommand{\PrefMajorityMNet}[1]{\Pref^{\mathit{maj}}_{#1}}
\newcommand{\PrefMajorityProfile}[1]{\Pref^{\mathit{maj}}_{#1}}

\newcommand{\Profile}{\mathscr}
\newcommand{\Ranking}{\mathcal}
\newcommand{\RepScheme}{\mathscr}
\newcommand{\Net}{\mathit}
\newcommand{\MNet}{\mathcal}
\newcommand{\OptOutNet}[1]{\mathit{o}_{#1}}
\newcommand{\OutNet}[1]{\mathcal{O}_{#1}}
\newcommand{\OutMNet}[1]{\mathcal{O}_{#1}}
\newcommand{\GraphNet}[1]{\mathcal{G}_{#1}}
\newcommand{\ExtGraphNet}[1]{G_{#1}}
\newcommand{\FeatNet}[1]{\mathcal{F}_{#1}}
\newcommand{\FeatMNet}[1]{\mathcal{F}_{#1}}
\newcommand{\LinkNet}[1]{\mathcal{E}_{#1}}
\newcommand{\NodeExtNet}[1]{V_{#1}}
\newcommand{\EdgeExtNet}[1]{E_{#1}}

\newcommand{\CPTNet}[1]{\CPTNetFeat{#1}{}}
\newcommand{\CPTNetFeat}[2]{\ifthenelse{\equal{#2}{}}{\mathit{CPT}_{#1}}{\mathit{CPT}_{#1}^{#2}}}
\newcommand{\Dom}{\mathit{Dom}}
\newcommand{\DomNet}[1]{\Dom_{#1}}
\newcommand{\DomMNet}[1]{\Dom_{#1}}
\newcommand{\Par}{\mathit{Par}}
\newcommand{\ParNet}[1]{\Par_{#1}}
\newcommand{\SetOrdNet}[1]{\mathit{Ord}_{#1}}
\newcommand{\ImpFlipVarNet}[2]{\xrightarrow{}_{#2[#1]}}
\newcommand{\ImpFlipVar}[1]{\rightarrow_{#1}}
\newcommand{\ImpFlipNet}[1]{\rightarrow_{#1}}
\newcommand{\ImpFlip}{\ImpFlipVar{}}

\newcommand{\Optimality}{\textsc{Optimality-Testing}\xspace}
\newcommand{\Dominance}{\textsc{Dominance}\xspace}
\newcommand{\Incomparability}{\textsc{Incomparability}\xspace}

\newcommand{\ParetoQuery}{\textsc{Pareto-Dominance}\xspace}
\newcommand{\IsParetoOptimal}{\textsc{Is-Pareto-Optimal}\xspace}
\newcommand{\ExistsParetoOptimal}{\textsc{Exists-Pareto-Optimal}\xspace}
\newcommand{\IsParetoOptimum}{\textsc{Is-Pareto-Optimum}\xspace}
\newcommand{\ExistsParetoOptimum}{\textsc{Exists-Pareto-Optimum}\xspace}

\newcommand{\MajorityQuery}{\textsc{Majority-Dominance}\xspace}
\newcommand{\IsMajorityOptimal}{\textsc{Is-Majority-Optimal}\xspace}
\newcommand{\ExistsMajorityOptimal}{\textsc{Exists-Majority-Optimal}\xspace}
\newcommand{\IsMajorityOptimum}{\textsc{Is-Majority-Optimum}\xspace}
\newcommand{\ExistsMajorityOptimum}{\textsc{Exists-Majority-Optimum}\xspace}

\newcommand{\Sat}{\textsc{Sat}\xspace}
\newcommand{\Unsat}{\textsc{Unsat}\xspace}
\newcommand{\Taut}{\textsc{Taut}\xspace}

\newcommand{\QbfECNF}[1]{\ensuremath{\mathrm{QBF}_{{#1},\exists}^{\mathit{CNF}}}\xspace}

\newcommand{\QbfADNF}[1]{\ensuremath{\mathrm{QBF}_{{#1},\forall}^{\mathit{DNF}}}\xspace}

\newcommand{\tuple}[1]{\langle #1 \rangle}

\newcommand{\NetCNF}{\Net{F}}
\newcommand{\NetCNFSumm}{\NetCNF_{\mathit{s}}}
\newcommand{\MNetIsParOpt}{\MNet{M}_{\mathrm{ipo}}}
\newcommand{\NetIPO}[1]{\Net{N}_{#1}^\mathrm{ipo}}
\newcommand{\NetDirect}{\Net{D}}
\newcommand{\NetInterAND}{\Net{H}_{\mathrm{C}}}
\newcommand{\NetInterOR}{\Net{H}_{\mathrm{D}}}
\newcommand{\MNetExistsWeakCond}{\MNet{M}_{\mathrm{eml}}}
\newcommand{\NetEWC}[1]{\Net{N}_{#1}^\mathrm{eml}}
\newcommand{\MNetIsStrongCond}{\MNet{M}_{\mathrm{imm}}}
\newcommand{\NetISC}[1]{\Net{N}_{#1}^\mathrm{imm}}
\newcommand{\MNetMaj}{\MNet{M}_{\mathit{NoWin}}}
\newcommand{\SetAgentsPrefMNet}[1]{S^{\Pref}_{#1}}
\newcommand{\SetAgentsFerpMNet}[1]{S^{\Ferp}_{#1}}
\newcommand{\SetAgentsIncompMNet}[1]{S^{\Incomp}_{#1}}
\newcommand{\SetAgentsPrefProfile}[1]{S^{\Pref}_{#1}}
\newcommand{\SetAgentsFerpProfile}[1]{S^{\Ferp}_{#1}}
\newcommand{\SetAgentsIncompProfile}[1]{S^{\Incomp}_{#1}}
\newcommand{\Change}{\mathit{Change}}
\newcommand{\Up}{\mathit{Up}}
\newcommand{\Down}{\mathit{Down}}
\newcommand{\WitnessSet}{\mathit{Witn}}
\newcommand{\NonWitnessSet}{\ol{\WitnessSet}}

\newcommand{\LogSpace}{\textnormal{{LOGSPACE}}\xspace}
\newcommand{\PTIME}{\textnormal{{P}}\xspace}

\newcommand{\FP}{\textnormal{{FP}}\xspace}
\newcommand{\NP}{\textnormal{{NP}}\xspace}
\newcommand{\NPh}{\textnormal{{NP\nobreakdash-\hspace{0pt}hard}}\xspace}
\newcommand{\NPc}{\textnormal{{NP\nobreakdash-\hspace{0pt}complete}}\xspace}

\newcommand{\CoNP}{\textnormal{co\nobreakdash-NP}\xspace}
\newcommand{\CoNPh}{\textnormal{{co\nobreakdash-NP\nobreakdash-\hspace{0pt}hard}}\xspace}
\newcommand{\CoNPc}{\textnormal{{co\nobreakdash-NP\nobreakdash-\hspace{0pt}complete}}\xspace}

\newcommand{\DP}[1]{\textnormal{\ensuremath{\mathrm{D}^\mathrm{P}_{#1}}}\xspace}

\newcommand{\BH}[1]{{\ifthenelse{\equal{#1}{}}{\ensuremath{\mathrm{BH}}}{\ensuremath{\mathrm{BH}(#1)}}}\xspace}
\newcommand{\BHThree}[1]{{\ifthenelse{\equal{#1}{}}{\ensuremath{\mathrm{BH}_3}}{\ensuremath{\mathrm{BH}_3(#1)}}}\xspace}
\newcommand{\CoDP}[1]{\textnormal{co\nobreakdash-\ensuremath{\mathrm{D}^\mathrm{P}_{#1}}}\xspace}

\newcommand{\DeltaP}[1]{\textnormal{\ensuremath{\Delta^\mathrm{P}_{#1}}}\xspace}

\newcommand{\ThetaP}[1]{\textnormal{\ensuremath{\Theta^\mathrm{P}_{#1}}}\xspace}

\newcommand{\SigmaP}[1]{\textnormal{\ensuremath{\Sigma^\mathrm{P}_{#1}}}\xspace}
\newcommand{\SigmaPh}[1]{\textnormal{\ensuremath{\Sigma^\mathrm{P}_{#1}}\nobreakdash-\hspace{0pt}hard}\xspace}
\newcommand{\SigmaPc}[1]{\textnormal{\ensuremath{\Sigma^\mathrm{P}_{#1}}\nobreakdash-\hspace{0pt}complete}\xspace}
\newcommand{\PiP}[1]{\textnormal{\ensuremath{\Pi^\mathrm{P}_{#1}}}\xspace}
\newcommand{\PiPh}[1]{\textnormal{\ensuremath{\Pi^\mathrm{P}_{#1}}\nobreakdash-\hspace{0pt}hard}\xspace}
\newcommand{\PiPc}[1]{\textnormal{\ensuremath{\Pi^\mathrm{P}_{#1}}\nobreakdash-\hspace{0pt}complete}\xspace}
\newcommand{\PSpace}{\textnormal{{PSPACE}}\xspace}
\newcommand{\ExpTime}{\textnormal{{EXPTIME}}\xspace}

\theoremstyle{plain}
\newtheorem{theorem}{Theorem}[section]
\newtheorem{lemma}[theorem]{Lemma}

\newtheorem{corollary}[theorem]{Corollary}

\theoremstyle{definition}

\newtheorem{example}[theorem]{Example}

\title{Complexity Results for Preference Aggregation over (\texorpdfstring{$m$}{m})\CPNets:\\ Pareto and Majority Voting\footnote{Preliminary results of this paper have appeared in the Proceedings of the 30th AAAI Conference on Artificial Intelligence (\mbox{AAAI-16})~\cite{LukasiewiczMaliziaAAAI2016}.}}

\date{}

\author[1]{Thomas Lukasiewicz}
\author[2]{Enrico Malizia}
\affil[1]{%
Department of Computer Science\\
University of Oxford, UK
}
\affil[2]{%
Department of Computer Science\\
University of Exeter, UK
}

\begin{document}

\maketitle

\begin{abstract}
Aggregating preferences over combinatorial domains has many applications in artificial intelligence (AI).
Given the inherent exponential nature of preferences over combinatorial domains, compact representation languages are needed to represent them, and ($m$)\CPNets are among the most studied ones.
Sequential and global voting are two different ways of aggregating preferences represented via \CPNets.
In sequential voting, agents' preferences are aggregated feature-by-feature.
For this reason, sequential voting may exhibit voting paradoxes, i.e., the possibility to select sub-optimal outcomes when preferences have specific feature dependencies.
To avoid paradoxes in sequential voting, one has often assumed the (quite) restrictive constraint of $\mathcal{O}$-legality, which imposes a shared common topological order among all the agents' \CPNets.
On the contrary, in global voting, \CPNets are considered as a whole during the preference aggregation process.
For this reason, global voting is immune from paradoxes, and hence there is no need to impose restrictions over the \CPNets' structure when preferences are aggregated via global voting.
Sequential voting over $\mathcal{O}$-legal \CPNets has extensively been investigated, and $\mathcal{O}$-legality of \CPNets has often been required in other studies.
On the other hand, global voting over non-$\mathcal{O}$-legal \CPNets has not  carefully been analyzed,
despite it was explicitly stated in the literature that a theoretical comparison between global and sequential voting was highly promising and a precise complexity analysis for global voting has been asked for multiple times.
In quite few works, only very partial results on the complexity of global voting over \CPNets have been given.
In this paper, we start to fill this gap by carrying out a thorough computational complexity analysis of global voting tasks, for Pareto and majority voting, over not necessarily $\mathcal{O}$-legal acyclic binary polynomially connected ($m$)\CPNets.
We show that all these problems belong to various levels of the polynomial hierarchy, and some of them are even in \PTIME or \LogSpace.
Our results are a notable achievement, given that the previously known upper bound for most of these problems was the complexity class \ExpTime.
We provide various exact complexity results showing tight lower bounds and matching upper bounds for problems that (up to date) did not have any explicit non-obvious lower bound.
\end{abstract}

\section{Introduction}\label{sec_intro} The problem of managing and aggregating agent preferences has attracted extensive interest in the computer science community~\cite{SurveyCompSocChoice}, because methods for representing and reasoning about preferences are very important in artificial intelligence (AI) applications, such as recommender systems~\cite{HandbookRecSys}, (group) product configuration~\cite{ConfigurationBook,Brafman2002,Stein2014},
(group)~planning~\cite{Brafman2005,Shaparau2006,Russell2012,Son2006}, (group) preference-based constraint satisfaction~\cite{Boutilier2004a,Boerkoel2010,Brafman2010}, and (group) preference-based query answering/information retrieval~\cite{Lukasiewicz2014,Lukasiewicz2013,DiNoia2015,Borgwardt2016}.

In computer science, the study of preference aggregation has often been based on the solid ground of social choice theory, which is the branch of economics analyzing methods for collective decision making~\cite{Arrow2002,Arrow2011}.
Having a well-founded theory and practice on how to properly and efficiently manage and aggregate preferences of real software agents, and hence support the growth and use of these technologies, has been one of the main drivers for investigating social choice theory from a computational perspective.
In social choice theory, the actual ways of representing agent preferences are rarely taken into consideration, also because the sets of candidates usually considered are relatively small in size.
For this reason, most of the insights obtained in the computational social choice literature about the computational properties of preference aggregation functions (or voting procedures) have assumed that agent preferences over the set of candidates are extensively listed (see~\cite{SurveyCompSocChoice} and references therein).
Although this is perfectly reasonable when we reason about, e.g., (political) elections among a not too numerous set of human candidates, this is not feasible when the voting domain (i.e., the set of candidates) has a \emph{combinatorial structure}~\cite{Lang2004,Lang_handbook,Lang_ai_magazine}.
By combinatorial structure, we mean that the set of candidates (or \emph{outcomes}) is the Cartesian product of finite value domains for each of a set of \emph{features} (also called \emph{variables}, or \emph{issues}, or \emph{attributes}).
The problem of aggregating agents' preferences over combinatorial domains (or \emph{multi-issue} domains) is called a \emph{combinatorial vote}~\cite{Lang2002,Lang2004}.

Interestingly, voting over combinatorial domains is rather common.
For example, in~2012, on the day of the US presidential election, voters in California had to vote also for eleven referenda~\cite{Lang_handbook}.
As another example, it may be the case that the inhabitants of a town have to make a joint decision about different related issues regarding their community, which could be whether and where to build new public facilities (such as a swimming pool or a library), or whether to levy new taxes.
Note that these voting scenarios are often also called \emph{multiple elections} or \emph{multiple referenda}~\cite{Brams1998,Xia2011,Xia2007,Lang2009,Lang_handbook,Lang_ai_magazine}.
Other examples are group product configurations and group planning~\cite{Stein2014,Lang_handbook}.
As for the latter, consider, e.g., a situation in which multiple autonomous agents have to agree upon a shared plan of actions to reach a goal that is preferred by the group as a whole,
such as a group of autonomous robots coordinating during the exploration of a remote area/planet.
Each robot has a specific task to accomplish, and the group as a whole coordinates to achieve a common goal.
That is, the robots have their own specific preferences over a vast amount of variables/features emerging from the contingency of the situation to complete their individual tasks,
however, their individual preferences have to be blended in all together, so that the course of action of a robotic agent does not interfere~with the tasks of the other agents, and the overall mission is successful.
These examples show the great relevance of dealing with combinatorial votes, and hence the pressing necessity of finding ways to represent agent~preferences over multi-issue domains and algorithms for aggregating them.

Combinatorial domains contain an exponential number of outcomes in the number of features, and hence compact representations for combinatorial preferences are needed~\cite{Lang2004,Lang_handbook}
(see also \cref{sec_representation_preferences} for more background).
The graphical model of \emph{\CPNets}~\cite{Boutilier2004} is among the most studied of these representations, as proven by a vast literature on them.
In \CPNets, the vertices of a graph represent features, and an edge from vertex $A$ to vertex $B$ models the influence of the value of feature $A$ on the choice of the value of feature $B$.
Intuitively, this model captures preferences like ``if the rest of the dinner is the same, with a fish dish ($A$'s value), I~prefer a white wine ($B$'s value)'',
also called \emph{conditional~ceteris paribus preferences}; a more detailed example is given below.

\begin{example}
Assume that we want to model one's preferences for a dinner with a main dish and a wine.
In the \CPNet in \cref{fig:dinner_example_cpnet}, an edge from vertex $\mathit{Main}$ to vertex $\mathit{Wine}$  models that the value of feature $\mathit{Main}$ influences the choice of the value of feature $\mathit{Wine}$. More
precisely,
 $m$ and $f$ are the possible values of feature $\mathit{Main}$, and they denote ``$m$eat'' and ``$f$ish'', respectively, while $r$ and $w$ are the possible values of feature $\mathit{Wine}$, and they denote ``$r$ed (wine)'' and ``$w$hite (wine)'', respectively.
The table associated with feature $\mathit{Wine}$ specifies that when a meat dish is chosen, then a red wine is preferred to a white one, and when a fish main is chosen, then a white wine is preferred to a red wine.
The table associated with feature $\mathit{Main}$ indicates that a meat dish is preferred to a fish one.
These tables are called \emph{CP tables}.
A~\CPNet like this one can represent the above conditional ceteris paribus preference ``given that the rest of the dinner does not change, with a meat dish ($\mathit{Main}$'s value), I~prefer a red wine ($\mathit{Wine}$'s value)''.

\begin{figure}[t]
  \centering%
  \begin{subfigure}[c]{0.5\textwidth}
    \centering
    \includegraphics[width=0.6\textwidth]{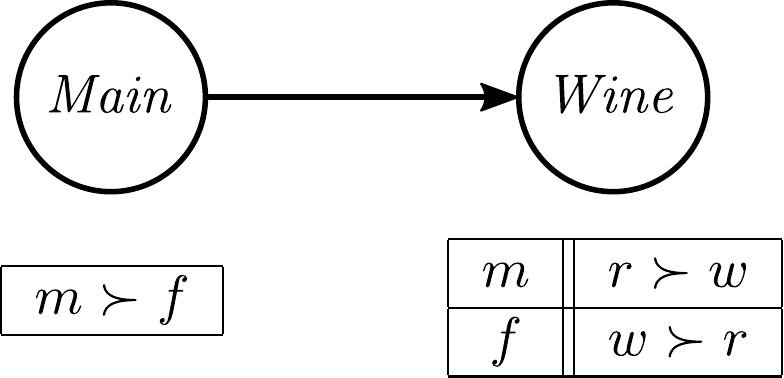}
    \caption{{\footnotesize A \CPNet modeling dinner preferences.}}
    \label{fig:dinner_example_cpnet}
  \end{subfigure}%
  \begin{subfigure}[c]{0.3\textwidth}
    \centering
    \includegraphics[width=0.5\textwidth]{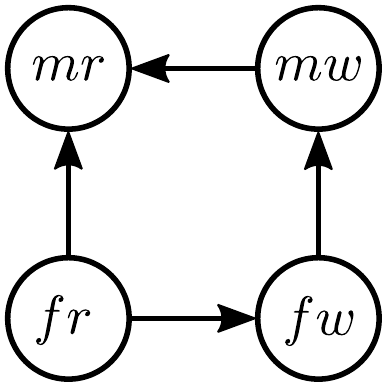}
   \vspace*{-1ex}
   \caption{{\footnotesize The CP-net's extended preference graph.}}
    \label{fig:dinner_example_ext_graph}
  \end{subfigure}
  \caption{A \CPNet and its preference graph.}\label{fig:dinner_example}
\end{figure}

Every \CPNet has an associated \emph{extended preference graph}, whose vertices are all the possible outcomes of the domain,
and whose  edges connect outcomes differing on only one value. More precisely, there is a directed edge from an outcome to another, if the latter is preferred to the former according to the preferences encoded in the tables of the \CPNet.
\cref{fig:dinner_example_ext_graph} shows the extended preference graph of the \CPNet in \cref{fig:dinner_example_cpnet},
having as vertices all the possible combinations for the dinner, and there is, e.g., an edge from $\mathit{mw}$ to $\mathit{mr}$, because the combination meat and red wine is preferred to the combination meat and white wine.
The preferences encoded in a \CPNet{} are the transitive closure of its extended preference graph.
Intuitively, an outcome $\alpha$ is preferred to an outcome $\beta$ according to the preferences of a \CPNet, if there is a directed path from $\beta$ to $\alpha$ in the extended preference graph. \hfill$\lhd$
\end{example}

\CPNets are also used to model preferences of \emph{groups} of individuals, obtaining a multi-agent model, called \emph{$m$\CPNets}~\cite{Rossi2004},
which is a set, or \emph{profile}, of \CPNets, one for each agent.
The preference semantics of \MCPNets is defined via voting schemes:
through its own individual \CPNet, every agent votes whether an outcome is preferred to another.
Various voting schemes were proposed for \MCPNets~\cite{Rossi2004,Li2015}, and different voting schemes give rise to different dominance semantics for \MCPNets.
In this paper, we consider \emph{Pareto} and \emph{majority voting} as they were defined in~\cite{Rossi2004}.
In the voting schemes proposed for \MCPNets, the voting protocol adopted, i.e., the actual way in which votes are collected \cite{Conitzer2005}, is \emph{global voting}~\cite{Lang2007,Lang_handbook}.
In this protocol, the results of the voting procedure are computed by having as input the \CPNets as a whole (see \cref{sec72} for related works on different voting protocols over \CPNets).

\begin{example}
Consider again the dinner scenario, and assume that there are three agents (Alice, Bob, and Chuck), expressing their preferences via \CPNets (see \cref{fig:dinner_aggregation}).
In Pareto voting, an outcome $\alpha$ dominates an outcome $\beta$, if \emph{all} agents prefers $\alpha$ to $\beta$.
In majority voting, an outcome $\alpha$ dominates an outcome $\beta$, if the \emph{majority} of agents prefers $\alpha$ to~$\beta$.

\begin{figure}
  \centering%
  \includegraphics[width=0.75\textwidth]{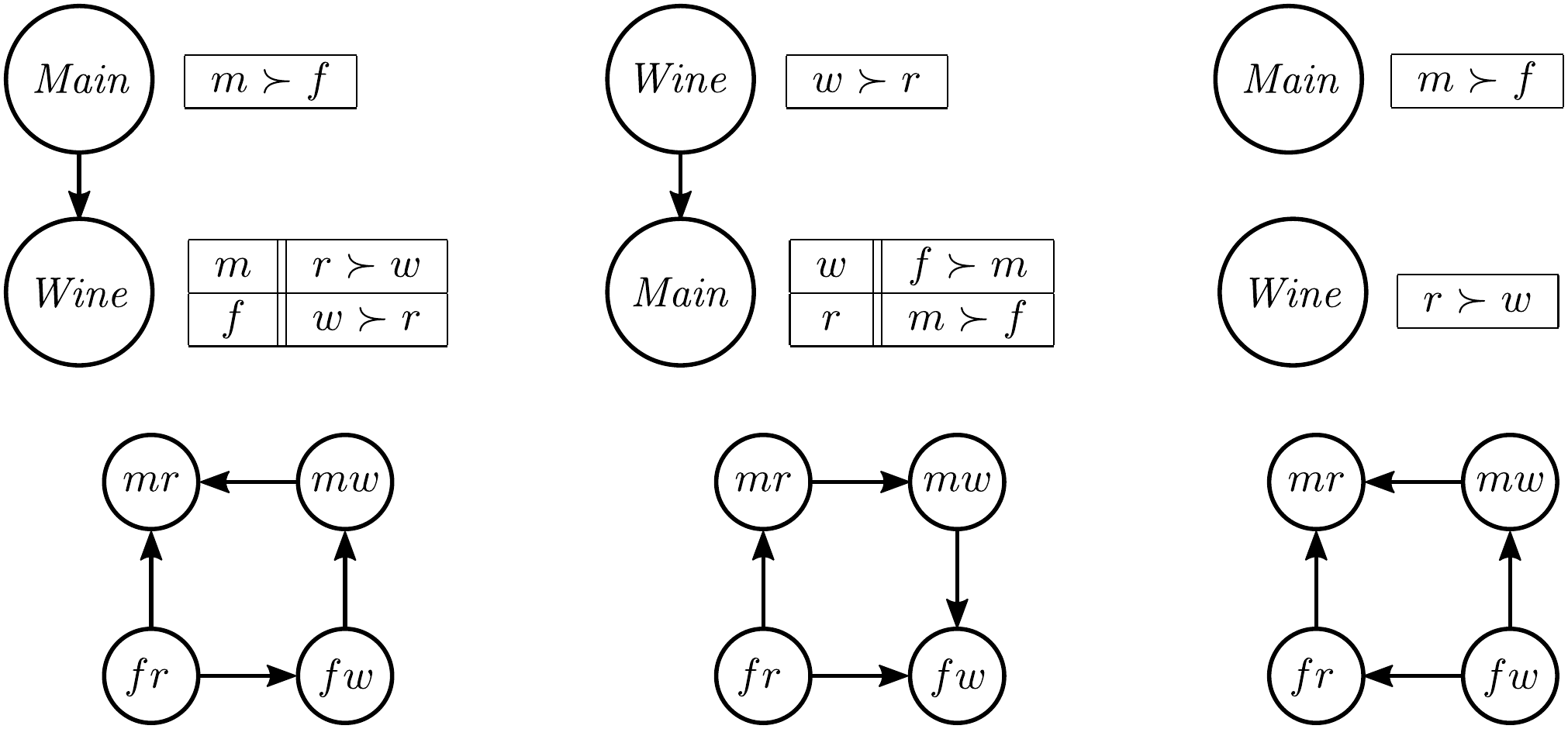}
  \caption{Dinner preferences of  Alice, Bob, and Chuck (in this order) modeled via \CPNets (above) and their extended preference graphs (below).}\label{fig:dinner_aggregation}
\end{figure}

The outcome $\mathit{fr}$ is not \emph{Pareto optimal}, because there is an outcome (namely $mr$), which is preferred to $\mathit{fr}$ by \emph{all} the agents.
The outcome $mw$, instead, is Pareto optimal, because there is no outcome Pareto dominating $mw$.
Hence, from a Pareto perspective, $mw$ is better than $\mathit{fr}$.
The outcome $mw$, however, is not \emph{majority optimal}, because
$mr$ majority dominates $mw$ (Alice and Chuck prefer $mr$ to $mw$).
On the other hand, $mr$ is majority optimal, because there is no outcome majority dominating $mr$.
Hence, from a majority perspective, $mr$ is better than $mw$.
Moreover, again according~to~the majority voting scheme, $mr$ is a very good outcome, because $mr$ is also \emph{majority optimum}, which means that $mr$ majority dominates all other outcomes.
On the contrary, in this example, there is no Pareto optimum outcome, i.e., there is no outcome Pareto dominating all other outcomes.\hfill$\lhd$
\end{example}

In the literature, a comparison between sequential voting (which is another voting protocol; see \cref{sec72}) and global voting over \CPNets was explicitly asked for and stated to be highly promising~\cite{Lang2007}.
However, global voting over \CPNets has not been as thoroughly investigated as sequential voting.
In fact, unlike \CPNets, which 
were extensively analyzed, a precise complexity analysis of \MCPNets has been missing for a long time, as explicitly mentioned several times in the literature~\cite{Lang2007,Li2010,Li2010a,Li2011,Li2015,Sikdar2017}---since the dominance semantics for \MCPNets is global voting over \CPNets, in the following, we use them interchangeably.
Furthermore, it was conjectured that the complexity of computing majority optimal and majority optimum outcomes in \MCPNets is harder than \NP and \CoNP~\cite{Li2010,Li2011}.

\paragraph{Contributions}
The aim of this paper is to explore the complexity of \MCPNets (and hence of global voting over \CPNets).
In particular, we focus on acyclic binary polynomially connected \MCPNets (see \cref{sec_prelim} for these notions) built with standard \CPNets, i.e., the constituent \CPNets of an \MCPNet rank all the features, and they are not partial \CPNets (which instead were allowed in the original definition of \MCPNets~\cite{Rossi2004}). 
Unlike what is often assumed in the literature, in this work, we do not restrict the profiles of \CPNets to be $\mathcal{O}$-legal 
(which means that there is a topological order common to all the \CPNets  of the profile;
see \cref{sec73}).
We carry out a thorough complexity analysis for the (a)~Pareto and (b)~majority voting schemes, as defined in~\cite{Rossi2004}, of deciding (1)~dominance, (2)~optimal and  (3) optimum outcomes, and (4)~the existence of optimal and (5) optimum outcomes.
Deciding the dominance for a voting scheme $s$ means deciding, given two outcomes, whether one dominates the other according to $s$.
Deciding whether an outcome is optimal or optimum for a voting scheme $s$ means deciding whether the outcome is not dominated or dominates all others, respectively, according to $s$.
Deciding the existence of optimal and optimum outcomes is the natural extension of the previous problems.

A summary of the complexity results obtained in this paper is provided in \cref{fig:summary_decision_results_mcp_nets}.
More precisely, deciding dominance and optimal outcomes is complete for NP and co-NP, respectively, for both Pareto and majority voting,
while deciding the existence of optimal outcomes can be done in constant time for Pareto voting and is complete for $\SigmaP2$ for majority voting.
Furthermore, deciding optimum outcomes and their existence is in $\LogSpace$ and $\PTIME$ for Pareto voting, and
complete for $\PiP2$ and between $\PiP2$ and $\DP2$ for majority voting, respectively.

It thus turns out that Pareto voting is the easiest  voting scheme to evaluate among the two analyzed here. More precisely, both Pareto and majority dominance are \NPc, however, only the complexity of majority dominance carries over to deciding optimal and optimum outcomes and their existence, and causes a substantial increase of their complexity, e.g., deciding the existence of majority~optimal~and~optimum outcomes is hard for \SigmaP{2} and \PiP{2}, respectively.
This is due to the fact that majority voting is structurally more complex than Pareto voting.
Intuitively, Pareto voting is based on unanimity, hence, to disprove Pareto dominance between two outcomes, it suffices to find \emph{one} agent that does not agree with the dominance relationship.
This particular structure of Pareto voting makes the other tasks not more difficult than the dominance test or even tractable.
Our results hence prove the conjecture posed in~\cite{Li2010,Li2011} about majority voting~tasks over ($m$)\CPNets being harder than \NP and \CoNP.

\begin{figure}
\centering%
\begin{tabular}{|c|l||c|}
\hline
& \textbf{Problem} & \textbf{Complexity}\\
\hline
\hline
\parbox[t]{2.5mm}{\multirow{5}{*}{\rotatebox[origin=c]{90}{\textsc{Pareto}}}}
& \ParetoQuery & \NPc \\
& \IsParetoOptimal & \CoNPc \\
& \ExistsParetoOptimal & $\Theta(1)$\textsuperscript{*}\\
& \IsParetoOptimum & in \LogSpace \\
& \ExistsParetoOptimum & in \PTIME \\
\hline
\parbox[t]{2.5mm}{\multirow{5}{*}{\rotatebox[origin=c]{90}{\textsc{Majority}}}}
& \MajorityQuery & \NPc \\
& \IsMajorityOptimal & \CoNPc \\
& \ExistsMajorityOptimal & \SigmaPc2 \\
& \IsMajorityOptimum & \PiPc2 \\
& \ExistsMajorityOptimum & \PiPh2, in $\DP2$ \\
\hline
\end{tabular}
\caption{Summary of the complexity results obtained in this paper for global voting over \CPNets. Membership results above \PTIME are valid for any representation scheme whose dominance test is feasible in \NP. \ \textsuperscript{*}\!{A different proof is available in~\cite{Rossi2004}.}}
\label{fig:summary_decision_results_mcp_nets}
\end{figure}

We show completeness results for most cases, and we provide tight lower bounds for problems that (up to date) did not have any explicit lower bound transcending the obvious hardness due to the dominance test over the underlying \CPNets.
Many of our results are intractability results, where the problems are put at various levels of the polynomial hierarchy.
However, although intractability is usually ``bad'' news, these results are quite interesting, as for most of these tasks, only \ExpTime upper bounds were known in the literature to date~\cite{Rossi2004}.
Even more interestingly, some of these problems are actually tractable, as they are in \PTIME or even \LogSpace, which is a huge leap from \ExpTime.

Our hardness results are given for binary acyclic polynomially connected ($m$)\CPNets.
This means that our hardness results extend to classes of ($m$)\CPNets encompassing the \CPNets considered here, and in particular also to general \MCPNets with partial \CPNets or multi-valued features.
More generally, the hardness results proven here extend to any representation scheme as ``expressive and succinct'' as the class of \CPNets used in the proofs (see \cref{sec_framework_representations}).
Moreover, the membership results above \PTIME that we prove here extend to any ``\NP-representation'' scheme
(see \cref{sec_framework_representations}). 
Our hardness results on the existence of optimal and optimum outcomes provide also lower bounds for the computational problems.
Indeed, actually computing optimal or optimum outcomes cannot be easier than the bounds shown here, because otherwise it would be possible to decide their existence more efficiently.

\paragraph{Organization of the paper}
The rest of this paper is organized as follows.
\Cref{sec_prelim} provides some preliminaries.
In \cref{sec_complexity_cpnets},
we prove some basic complexity results for \CPNets.
\Cref{sec5,sec6} analyze the complexity of Pareto and majority voting, respectively:
first, we analyze the complexity of dominance testing;
then, we study the complexity of deciding whether an outcome is optimal and whether there exists an optimal outcome;
and we conclude by dealing with the complexity of deciding whether an outcome is optimum and whether there exists an optimum outcome.
In~\cref{sec_related}, we discuss related works.
\cref{sec_concl_future}~summarizes~the main results and
gives an outlook on future research.
For several results, we give only proof sketches 
in the body of the paper, while detailed proofs are provided in \cref{sec_detailed_proofs}.

\section{Preliminaries}\label{sec_prelim}
In this section, we give some preliminaries, briefly recalling from the literature preference relations and aggregation, conditional preference nets (\CPNets), \CPNets for groups of $m$ agents (\MCPNets), and the complexity classes that we will encounter in our complexity results. We also define a formal framework for preference representation schemes, because our membership results will be given for generic representations whose dominance test is feasible in \NP.

\subsection{Preference relations and aggregation}
Before dwelling upon the details of \CPNets, which is the specific preference representation analyzed in this paper, we now give an introductory overview of the general concepts of preferences and their aggregation.

In this paper, a \emph{preference relation} $\Ranking{R}$ over a set of outcomes $\mathcal{O}$ is a strict order over $\mathcal{O}$, i.e., $\Ranking{R}$ is a binary relation over $\mathcal{O}$
that is irreflexive (i.e., $\tuple{\alpha,\alpha}\notin\Ranking{R}$), asymmetric (i.e., if~$\tuple{\alpha,\beta}\in\Ranking{R}$, then $\tuple{\beta,\alpha}\notin\Ranking{R}$), and transitive (i.e., if~$\tuple{\alpha,\beta}\in\Ranking{R}$ and $\tuple{\beta,\gamma}\in\Ranking{R}$, then $\tuple{\alpha,\gamma}\in\Ranking{R}$).
A \emph{preference ranking}  $\Ranking{R}$ is a preference relation that is total (i.e., either $\tuple{\alpha,\beta}\in\Ranking{R}$ or $\tuple{\beta,\alpha}\in\Ranking{R}$ for any two different outcomes $\alpha$ and $\beta$).
Usually, given two outcomes $\alpha$ and~$\beta$, their preference relationship stated in $\Ranking{R}$ is denoted by $\alpha\PrefRanking{\Ranking{R}}\beta$, instead of $\tuple{\alpha,\beta}\in\Ranking{R}$, which means that, in $\Ranking{R}$, $\alpha$ is strictly preferred to $\beta$, or $\alpha$ \emph{dominates}~$\beta$.
On the other hand, $\alpha\not\PrefRanking{\Ranking{R}}\beta$ means that $\tuple{\alpha,\beta}\notin\Ranking{R}$, and $\alpha\Incomp_\Ranking{R}$ means that $\tuple{\alpha,\beta}\notin\Ranking{R}$
\emph{and}~$\tuple{\beta,\alpha}\notin\Ranking{R}$, i.e., $\alpha$ and $\beta$ are \emph{incomparable} in $\Ranking{R}$.
Observe that in a preference ranking, it cannot be the case that two outcomes are incomparable.
Given a preference relation $\Ranking{R}$, an outcome $\alpha$ is \emph{optimal} in $\Ranking{R}$ if there is no outcome $\beta$ such that~$\beta\PrefRanking{\Ranking{R}}\alpha$.
We say that $\alpha$ is \emph{optimum} in $\Ranking{R}$, if for all outcomes $\beta$ such that $\beta\neq\alpha$, it holds that $\alpha\PrefRanking{\Ranking{R}}\beta$.
Clearly, if there is an optimum outcome in $\Ranking{R}$, then it is unique.
For notational convenience, if the preference relation $\Ranking{R}$ is clear from the context, we do not explicitly mention $\Ranking{R}$ as a subscript in the notations above.
In the following, if not stated otherwise, when we speak of preferences structures, we mean preference relations.

In preference aggregation, we
deal with preferences of multiple agents.
A preference \emph{profile} $\Profile{P}=\tuple{\Ranking{R}_1,\dots,\Ranking{R}_m}$ is a set of $m$ preference relations.
We assume that all the preferences $\Ranking{R}_i$ of $\Profile{P}$ are defined over the same set of outcomes, i.e., the agents express their preferences over the same set of candidates.
In this paper, we  focus on voting procedures based on comparisons of pairs of outcomes (see, e.g.,~\cite{BaumeisterRoth:Voting} for a classification of different kinds of preference aggregation procedures).
For this reason, we need to define the following sets of agents.
For a profile $\Profile{P} = \tuple{\Ranking{R}_1,\dots,\Ranking{R}_m}$, we denote by $\SetAgentsPrefProfile{\Profile{P}}(\alpha,\beta) = \{i \mid \alpha\PrefRanking{\Ranking{R}_i}\beta\}$, $\SetAgentsFerpProfile{\Profile{P}}(\alpha,\beta) = \{i \mid \alpha\FerpRanking{\Ranking{R}_i}\beta\}$, and $\SetAgentsIncompProfile{\Profile{P}}(\alpha,\beta) = \{i \mid \alpha\IncompRanking{\Ranking{R}_i}\beta\}$, the sets of agents preferring $\alpha$ to $\beta$, preferring $\beta$ to $\alpha$, and for which $\alpha$ and $\beta$ are incomparable, respectively.

The voting schemes considered in this paper are Pareto and majority. The definition of their dominance semantics over preference profiles, reported below, is a generalization of the respective definition over \MCPNets given in~\cite{Rossi2004}.

\begin{description}
  \item[Pareto:] An outcome $\beta$ \emph{Pareto dominates} an outcome $\alpha$, denoted $\beta\PrefParetoProfile{\Profile{P}}\alpha$, if \emph{all} agents prefer $\beta$ to $\alpha$, i.e., $|\SetAgentsPrefProfile{\Profile{P}}(\beta,\alpha)| = m$.

  \item[Majority] An outcome $\beta$ \emph{majority dominates} an outcome $\alpha$, denoted $\beta\PrefMajorityProfile{\Profile{P}}\alpha$, if the \emph{majority} of the agents prefer $\beta$ to~$\alpha$, i.e., $|\SetAgentsPrefProfile{\Profile{P}}(\beta,\alpha)| > |\SetAgentsFerpProfile{\Profile{P}}(\beta,\alpha)| + |\SetAgentsIncompProfile{\Profile{P}}(\beta,\alpha)|$.
\end{description}

For a preference profile $\Profile{P}$ and a voting scheme $s$, if outcome $\beta$ does not $s$ dominate outcome $\alpha$, we denote this by $\beta\not\Pref^{\mathit{s}}_{\Profile{P}}\alpha$.
An outcome $\alpha$ is \emph{$s$ optimal} in $\Profile{P}$, if for all $\beta\neq\alpha$, it holds that $\beta\not\Pref^{\mathit{s}}_{\Profile{P}}\alpha$, while $\alpha$ is \emph{$s$ optimum} in $\Profile{P}$, if for all $\beta\neq\alpha$, it holds that $\alpha\Pref^{\mathit{s}}_{\Profile{P}}\beta$.
Note that optimum outcomes, if they exist, are unique.

\subsection{CP-nets}\label{sec_cpnets}

We now focus on \CPNets, which is the preference representation that we will more closely investigate in this work.
As mentioned in the introduction, the set of outcomes of a preference relation is often defined as the Cartesian product of finite value domains for each of a set of {features}.
Conditional preference nets (\CPNets)~\cite{Boutilier2004} are a formalism to encode conditional ceteris paribus preferences over such combinatorial domains.
The distinctive element of~\CPNets is that a directed graph, whose vertices represent the features of a combinatorial domain,
is used to intuitively model the conditional part of conditional ceteris paribus preference statements.
Below, we recall the syntax, semantics, and some properties of \CPNets; see \cref{sec_representation_preferences} for more on conditional ceteris paribus preferences and preference representations in general.

\paragraph{Syntax of CP-nets}
A \emph{\CPNet} $\Net{N}$ is a triple $\tuple{\GraphNet{\Net{N}},\DomNet{\Net{N}},{(\CPTNetFeat{\Net{N}}{F})}_{F\in\FeatNet{\Net{N}}}}$, where $\GraphNet{\Net{N}}=\linebreak[0]\tuple{\FeatNet{\Net{N}},\LinkNet{\Net{N}}}$ is a directed graph whose vertices $\FeatNet{\Net{N}}$ represent the \emph{features} of a combinatorial domain, and $\DomNet{\Net{N}}$ and ${(\CPTNetFeat{\Net{N}}{F})}_{F\in\FeatNet{\Net{N}}}$ are a function and a family of functions, respectively.
The function $\DomNet{\Net{N}}$ associates a \emph{(value) domain} $\DomNet{\Net{N}}(F)$ with every feature $F$, while the functions $\CPTNetFeat{\Net{N}}{F}$ are the \emph{CP tables} for every feature $F$,
which are defined below.
The value domain of a feature $F$ is the set of all values that $F$ may assume in the possible outcomes.
In this paper, we assume features to be \emph{binary}, i.e., the domain of each feature $F$ contains exactly two values,
 usually denoted $\ol{f}$ and $f$, and called the \emph{overlined} and the \emph{non-overlined} value (of $F$), respectively.
For a set of features $\SetFeat{S}$, $\DomNet{\Net{N}}(\SetFeat{S})=\times_{F\in\SetFeat{S}}\DomNet{\Net{N}}(F)$ denotes the Cartesian product of the domains of the features in $\SetFeat{S}$.
Thus, an \emph{outcome} is an element of $\OutNet{\Net{N}}=\DomNet{\Net{N}}(\FeatNet{\Net{N}})$.
Given a feature $F$ and an outcome $\alpha$, we denote by $\alpha[F]$ the value of $F$ in $\alpha$, while, given a set of features $\SetFeat{F}$, $\alpha[\SetFeat{F}]$ is the projection of $\alpha$ over $\SetFeat{F}$.
For two outcomes $\alpha$ and $\beta$, and a set of features $\SetFeat{F}$, we denote by $\alpha[\SetFeat{F}] = \beta[\SetFeat{F}]$ that $\alpha[F] = \beta[F]$ for all $F\in\SetFeat{F}$; we write $\alpha[\SetFeat{F}] \neq \beta[\SetFeat{F}]$, when this is not the case, i.e., when there is at least one feature $F\in\SetFeat{F}$ such that $\alpha[F] \neq \beta[F]$.
The CP tables encode preferences over feature values.
Intuitively, the CP table of a feature $F$ specifies how the values of the parent features of $F$ influence the preferences over the values of $F$.
More formally, for a feature $F$, we denote by $\ParNet{\Net{N}}(F)=\{G\in\FeatNet{\Net{N}}\mid \tuple{G,F}\in\LinkNet{\Net{N}}\}$ the set of all features in $\GraphNet{\Net{N}}$ from which there is an edge to $F$.
We call $\ParNet{\Net{N}}(F)$ the set of the \emph{parents} of $F$ (in $\Net{N}$).
We denote by $\SetOrdNet{\Net{N}}(F)$ the set of all the (strict) preference rankings over the elements of $\DomNet{\Net{N}}(F)$.
Each function $\CPTNetFeat{\Net{N}}{F}\colon \DomNet{\Net{N}}(\ParNet{\Net{N}}(F)) \rightarrow \SetOrdNet{\Net{N}}(F)$ maps every element of $\DomNet{\Net{N}}(\ParNet{\Net{N}}(F))$ to a (strict) preference ranking over the domain of $F$.
If $\ParNet{\Net{N}}(F)\,{=}\,\emptyset$, then $\CPTNetFeat{\Net{N}}{F}$ is a single (strict) preference ranking over $\DomNet{\Net{N}}(F)$.
Note that indifferences between feature values are not admitted in (classical) \CPNets.
Each function $\CPTNetFeat{\Net{N}}{F}$ is represented via a two-column table, in which, given a row, the element in the first column is the input value of the function $\CPTNetFeat{\Net{N}}{F}$, and the element in the second column is the associated (strict) preference ranking  over $\DomNet{\Net{N}}(F)$.
Since $\CPTNetFeat{\Net{N}}{F}$ is total, in the table representing the function, there is a row for any combination of values of the parent features, i.e., for a feature $F$, there are $2^{|\DomNet{\Net{N}}(\ParNet{\Net{N}}(F))|}$ rows in the table of $F$.

In the following, when we define CP tables, we often use a logical notation to identify for which specific values of the parent features, a particular row in the CP table has to be considered.
Although this is the notation on which generalized propositional \CPNets~\cite{Goldsmith2008} are based on, it is used here only for notational convenience.
In this paper, we always assume that CP tables are explicitly represented in the input instances.
In the CP tables, $\ol{f}\Pref f$ denotes $\ol{f}$ being preferred to $f$.
We denote by $\|\Net{N}\|$ the size of \CPNet $\Net{N}$, i.e., the space in terms of bits required to represent the whole net $\Net{N}$ (which includes features, edges, feature domains, and CP tables).

\paragraph{Semantics of CP-nets}
The preference semantics of \CPNets
can be defined in several different but equivalent ways~\cite{Boutilier2004}.
A first definition has a model-theoretic flavour~\cite[Definitions~2 and~3]{Boutilier2004}.
Intuitively, a preference ranking $\Ranking{R}$ violates a \CPNet $\Net{N}$, if there are two outcomes $\alpha$ and $\beta$ that according to the CP tables of $\Net{N}$ should be ranked $\beta\Pref\alpha$, but they are not ranked in such a way in $\Ranking{R}$
(i.e., $\alpha\PrefRanking{\Ranking{R}}\beta$, since $\Ranking{R}$ is total). 
Formally, a preference ranking $\Ranking{R}$ \emph{violates} a \CPNet $\Net{N}$, if there are two distinct outcomes $\alpha,\beta$ and a feature $F$ such that $\alpha[\FeatNet{\Net{N}}\setminus\{F\}] = \beta[\FeatNet{\Net{N}}\setminus\{F\}]$ (i.e., $\alpha$ and $\beta$ differ only on the value of~$F$), $\beta[F] \Pref \alpha[F]$ in the order $\CPTNetFeat{\Net{N}}{F}(\alpha[\ParNet{\Net{N}}(F)])$, and $\alpha\PrefRanking{\Ranking{R}}\beta$.
A preference ranking $\Ranking{R}$ \emph{satisfies} a \CPNet $\Net{N}$, if $\Ranking{R}$ does not violate $\Net{N}$.
Given two outcomes $\alpha$ and $\beta$, a \CPNet $\Net{N}$ \emph{entails} the preference $\alpha\Pref\beta$, denoted $\Net{N}\models\alpha\Pref\beta$, if
$\alpha\PrefRanking{\Ranking{R}}\beta$ for every preference ranking $\Ranking{R}$ over $\OutNet{\Net{N}}$
that satisfies $\Net{N}$.
The preference semantics of \CPNets can be equivalently defined via the concept of improving (or alternatively worsening) flip~\cite[Definition~4]{Boutilier2004}:
let $F$ be a feature, and let $\alpha$ be an outcome.
Intuitively, flipping the value of $F$ in $\alpha$ from $\alpha[F]$ to a different one is an improving flip, if the new value of $F$ is preferred, given the values in $\alpha$ of the parent features of $F$.
More formally, flipping $F$ from $\alpha[F]$ to a different value $f'$ is an \emph{improving flip}, if $f'\Pref \alpha[F]$ holds in  $\CPTNetFeat{\Net{N}}{F}(\alpha[\ParNet{\Net{N}}(F)])$.
Given two outcomes $\alpha$ and $\beta$ differing only on the value of a feature $F$, there is an improving flip from $\alpha$ to $\beta$, denoted  $\alpha\ImpFlipVarNet{F}{\Net{N}}\beta$, if  flipping the value of $F$ from $\alpha[F]$ to $\beta[F]$ is an improving flip.
In the following, we often omit the feature $F$ 
and simply write $\alpha\ImpFlipNet{\Net{N}}\beta$; and
when we say that we flip a feature, then we  often mean that the flipping is improving. The \emph{(extended) preference graph of} $\Net{N}$ is the pair  $\ExtGraphNet{\Net{N}}=\tuple{\NodeExtNet{\Net{N}},\EdgeExtNet{\Net{N}}}$, where the nodes $\NodeExtNet{\Net{N}}$ are all the possible outcomes of $\Net{N}$, and, given two outcomes $\alpha,\beta\in\NodeExtNet{\Net{N}}$, the directed edge from $\alpha$ to $\beta$ belongs to $\EdgeExtNet{\Net{N}}$ if and only if $\alpha\ImpFlipNet{\Net{N}}\beta$.

It can be shown that, for a \CPNet $\Net{N}$ and two outcomes $\alpha$ and $\beta$, $\Net{N}\models\beta\Pref\alpha$ if and only if there is a sequence of improving flips from $\alpha$ to $\beta$~\cite[Theorems~7 and~8]{Boutilier2004}.
Therefore, for an agent whose preferences are encoded through a \CPNet $\Net{N}$, we say that the agent \emph{prefers} $\beta$ to $\alpha$, or that $\beta$ \emph{dominates} $\alpha$ (in $\Net{N}$), denoted $\beta\PrefNet{\Net{N}}\alpha$, if $\Net{N}$ entails $\beta\Pref\alpha$, or, equivalently, if there is an improving flipping sequence from $\alpha$ to $\beta$.
If for two outcomes $\alpha$ and $\beta$, neither $\alpha\PrefNet{\Net{N}}\beta$ nor~$\beta\PrefNet{\Net{N}}\alpha$, then $\alpha$ and $\beta$ are 
\emph{incomparable} (in $\Net{N}$), denoted $\alpha\IncompNet{\Net{N}}\beta$ (which is equivalent to the existence of preference rankings $\Ranking{R}_1$ and $\Ranking{R}_2$ that both satisfy $\Net{N}$ such that $\alpha\PrefRanking{\Ranking{R}_1}\beta$ and $\beta\PrefRanking{\Ranking{R}_2}\alpha$).

Note here that, since there are no indifferences between features values in (classical) \CPNets, for any two outcomes~$\alpha$ and $\beta$, either one dominates the other, or they are incomparable.

\begin{example}\label{ex:first_example}
\begin{figure}
  \centering%
  \begin{subfigure}[c]{0.4\textwidth}
    \includegraphics[width=\textwidth]{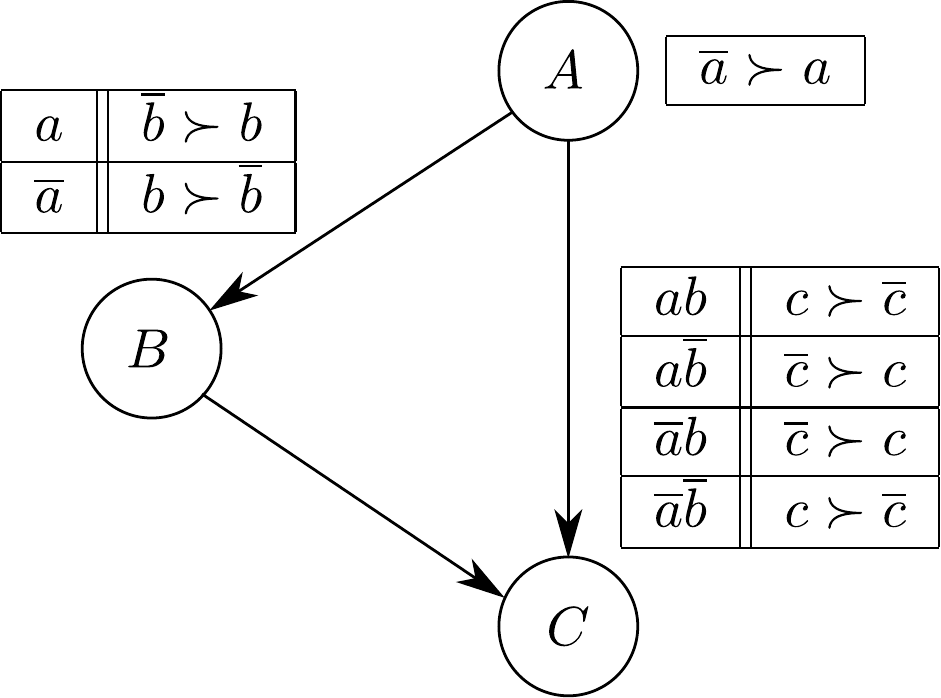}
    \caption{{\footnotesize A \CPNet $\Net{N}$ with three features.}}
    \label{fig:first_example_net}
  \end{subfigure}%
  \hspace{2cm}
  \begin{subfigure}[c]{0.4\textwidth}
    \includegraphics[width=\textwidth]{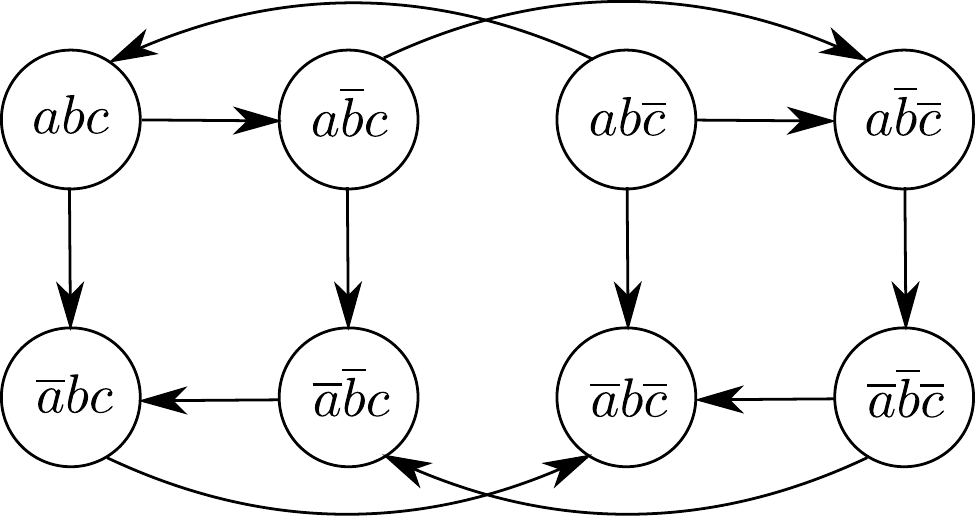}
    \caption{{\footnotesize The CP-net's extended preference graph $\ExtGraphNet{\Net{N}}$.}}
    \label{fig:first_example_ext_graph}
  \end{subfigure}
  \caption{A \CPNet and its extended preference graph.}\label{fig:first_example}
\end{figure}

Consider the \CPNet $\Net{N}$ shown in \cref{fig:first_example}. For the
 outcomes $\alpha=abc$ and $\beta=\ol{a}bc$, it holds that
$\beta\PrefNet{\Net{N}}\alpha$, because $\alpha\ImpFlipVarNet{A}{\Net{N}}\beta$. For
 the outcomes $\alpha=abc$ and $\beta=ab\ol{c}$, it holds that
 $\beta\not\PrefNet{\Net{N}}\alpha$, because there is no path from $\alpha$ to $\beta$ in $\ExtGraphNet{\Net{N}}$.
However, $\alpha\PrefNet{\Net{N}}\beta$, because $\beta\ImpFlipVarNet{C}{\Net{N}}\alpha$, and hence it is not the case that $\alpha\IncompNet{\Net{N}}\beta$.
Consider now the outcomes $\alpha=abc$ and $\beta=\ol{a}\ol{b}\ol{c}$.
Then, $\beta\PrefNet{\Net{N}}\alpha$ by the improving flipping sequence $abc\ImpFlip a\ol{b}c\ImpFlip a\ol{b}\ol{c}\ImpFlip \ol{a}\ol{b}\ol{c}$. \hfill$\lhd$
\end{example}

\paragraph{Properties of CP-nets}
A \CPNet is \emph{binary}, if all its features are binary.
The \emph{indegree} of a \CPNet $\Net{N}$ is the maximum number of edges entering in a node of the graph of $\Net{N}$.
A \CPNet $\Net{N}$ is \emph{singly connected}, if, for any two distinct features~$G$ and $F$, there is at most one path from $G$ to $F$  in $\GraphNet{\Net{N}}$.
A class $\mathscr{F}$ of \CPNets is \emph{polynomially connected}, if there exists a polynomial $p$ such that, for any \CPNet $\Net{N}\in\mathscr{F}$ and for any two features $G$ and $F$ of $\Net{N}$, there are at most $p(\|\Net{N}\|)$ distinct paths from $G$ to $F$ in $\GraphNet{\Net{N}}$.
A \CPNet $\Net{N}$ is \emph{acyclic}, if $\GraphNet{\Net{N}}$ is acyclic.
It is well known that acyclic \CPNets $\Net{N}$ always have a preference ranking satisfying $\Net{N}$, their extended preference graph $\ExtGraphNet{\Net{N}}$ is acyclic, the preferences encoded by $\Net{N}$ are consistent (i.e., there is no outcome $\alpha$ such that $\alpha\PrefNet{\Net{N}}\alpha$), and there is a unique optimum outcome $\OptOutNet{\Net{N}}$ dominating all other outcomes (and, clearly, not dominated by any other), which can be computed in polynomial time~\cite{Boutilier2004}.

It is known that dominance testing, i.e., deciding, for any two given outcomes $\alpha$ and $\beta$, whether $\beta\Pref\alpha$, is feasible in~\NP over polynomially connected classes of binary acyclic \CPNets~\cite{Boutilier2004}.
However, it is an open problem whether dominance testing is feasible in \NP over non-polynomially-connected classes of binary acyclic \CPNets.
Also, the complexity of  dominance testing for non-binary \CPNets is currently still open.
Whereas dominance testing for the class of acyclic binary singly connected \CPNets whose indegree is at most six is \NPh~\cite{Boutilier2004}---we  improve this result in \cref{sec_complexity_cpnets}, requiring only indegree three.
Dominance testing is feasible in polynomial time on acyclic binary \CPNets whose graph is a tree or a polytree~\cite{Boutilier2004},
and it is \PSpace-complete for cyclic \CPNets~\cite{Goldsmith2008}.

In the rest of this paper, we consider only binary acyclic (and often polynomially connected) classes of \CPNets.
When the \CPNet $\Net{N}$ is clear from the context, we often omit the subscript ``$\Net{N}$'' from the notations introduced above.

\subsection{\texorpdfstring{$m$}{m}CP-nets}\label{sec_mcpnets}
In this section, we focus on \MCPNets~\cite{Rossi2004}, which are a formalism to reason about conditional ceteris paribus preferences when a group of multiple agents is considered.
Intuitively, an \MCPNet is a profile of~$m$ (individual) \CPNets, one for each agent of the group.
The original definition of \MCPNets also allows for partial \CPNets.
Here, we consider only \MCPNets consisting of a collection of standard \CPNets.
The difference is that we do not allow for non-ranked features in agents' \CPNets, and hence there is no distinction between private, shared, and visible features (see~\cite{Rossi2004} for definitions), i.e., all features are ranked in all the individual \CPNets of an \MCPNet.

As underlined in~\cite{Rossi2004}, the ``$m$'' of an \MCPNet stands for multiple agents and also indicates that the preferences of $m$ agents are modeled, so a $3$\CPNet is an \MCPNet with $m=3$.
Formally, an $m$\CPNet $\MNet{M}=\tuple{\Net{N}_1,\dots,\Net{N}_m}$ consists of $m$ \CPNets $N_1,\ldots,N_m$, all of them defined over the same set of features, which, in turn, have the same domains.
If $\MNet{M}$ is an \MCPNet, we denote by $\FeatMNet{\MNet{M}}$ the set of all features of $\MNet{M}$, and by $\DomMNet{\MNet{M}}(F)$ the domain of feature $F$ in $\MNet{M}$.
Given this notation, $\FeatNet{\Net{N}_i}=\FeatMNet{\MNet{M}}$, for all $1\leq i\leq m$, and $\DomNet{\Net{N_i}}(F)=\DomMNet{\MNet{M}}(F)$, for all features $F\in\FeatMNet{\MNet{M}}$ and all $1\leq i\leq m$.
Although the features of the individual \CPNets are the same, their graphical structures may be different, i.e., the edges between the features in the various individual \CPNets may vary.
We underline here that, unlike in other papers in the literature, we do \emph{not} impose that the individual \CPNets of the agents share a common topological order (i.e., we do not restrict the profiles of \CPNets to be $\mathcal{O}$-legal);
see \cref{sec_related} for more on $\mathcal{O}$-legality.

An \emph{outcome} for an \MCPNet is an assignment to all the features of the \CPNets, and given an \MCPNet $\MNet{M}$, we denote by $\OutMNet{\MNet{M}}$ the set of all the outcomes in $\MNet{M}$.
The preference semantics of \MCPNets is defined through global voting over \CPNets.
In particular, via its own individual \CPNet, each agent votes whether an outcome dominates another, and hence different ways of collecting votes (i.e., different voting schemes) give rise to different group dominance semantics for an \MCPNet.
Let $\MNet{M}=\tuple{\Net{N}_1,\dots,\Net{N}_m}$ be an \MCPNet, and let $\alpha$ and $\beta$ be two outcomes.
With a notation similar to the one defined above, $\SetAgentsPrefMNet{\MNet{M}}(\alpha,\beta)=\{i \mid \alpha\PrefNet{\Net{N}_i}\beta\}$, $\SetAgentsFerpMNet{\MNet{M}}(\alpha,\beta)=\{i \mid \alpha\FerpNet{\Net{N}_i}\beta\}$, and $\SetAgentsIncompMNet{\MNet{M}}(\alpha,\beta)=\{i \mid \alpha\IncompNet{\Net{N}_i}\beta\}$ are the sets of the agents of $\MNet{M}$ preferring $\alpha$ to $\beta$, preferring $\beta$ to~$\alpha$, and for which $\alpha$ and $\beta$ are incomparable, respectively.

In~\cite{Rossi2004,Li2015}, various voting schemes were proposed and analyzed to define multi-agent dominance semantics for \MCPNets.
In this paper, we focus on two of them, namely, Pareto and majority voting, whose dominance semantics definitions are the natural specializations to \MCPNets of Pareto and majority dominance semantics defined above.
Consider an \MCPNet $\MNet{M}=\tuple{\Net{N_1},\dots,\Net{N_m}}$, and let $\alpha$ and $\beta$ be two outcomes. Then:
\begin{description}
\item[Pareto:] $\beta$ \emph{Pareto dominates} $\alpha$, denoted  $\beta\PrefParetoMNet{\MNet{M}}\alpha$, if \emph{all} the agents of $\MNet{M}$ prefer $\beta$ to $\alpha$, i.e., $|\SetAgentsPrefMNet{\MNet{M}}(\beta,\alpha)|=m$.

\item[Majority:] $\beta$ \emph{majority dominates} $\alpha$, denoted  $\beta\PrefMajorityMNet{\MNet{M}}\alpha$, if the \emph{majority} of the agents of $\MNet{M}$ prefers $\beta$ to $\alpha$, i.e.,\\ $|\SetAgentsPrefMNet{\MNet{M}}(\beta,\alpha)|>|\SetAgentsFerpMNet{\MNet{M}}(\beta,\alpha)|+|\SetAgentsIncompMNet{\MNet{M}}(\beta,\alpha)|$.
\end{description}

For a voting scheme $s$, $s$ optimal and $s$ optimum outcomes in \MCPNets are defined in the natural way.

An \MCPNet is \emph{acyclic}, \emph{binary}, and \emph{singly connected}, if all its \CPNets are acyclic, binary, and singly connected, respectively.
A class $\mathscr{F}$ of \MCPNets is \emph{polynomially connected}, if the set of \CPNets{} constituting the \MCPNets in $\mathscr{F}$ is a polynomially connected class of \CPNets.
The \emph{indegree} of an \MCPNet is the maximum indegree of its constituent individual \CPNets.
Unless stated otherwise, we consider only polynomially connected classes of acyclic binary \MCPNets.
When the \MCPNet $\MNet{M}$ is clear from the context, we often omit the subscript ``$\MNet{M}$'' from the above notations.

\subsection{Computational complexity}
\label{sec_some_note_before_start}
We now give some notions from computational complexity theory, which will be required for the complexity analysis carried out in this paper.
First, we briefly recall the complexity classes that we will encounter in this paper (along with some closely related ones), and then we recall the notion of polynomial-time reductions among decision problems,
	and~some~decision problems that are hard for some of these complexity classes.
	We assume that the reader has some elementary background in computational complexity theory, including the notions of Boolean formulas and quantified Boolean formulas, Turing machines,
	 and hardness and completeness of a problem for a complexity class, as can be found, e.g., in \cite{Johnson1990,Papadimitriou1994}.

\paragraph{Complexity classes} The class $\PTIME$ is the set of all decision problems that can be solved by a deterministic Turing machine in polynomial time with respect to the input size, i.e., with respect to the length of the string that encodes the input instance.
For a given input string $s$, its size is usually denoted by $\|s\|$.
The class of decision problems that can be solved by nondeterministic Turing machines in polynomial time is denoted by~$\NP$.
They enjoy a remarkable property: any ``yes''-instance $s$ has a \emph{certificate} for being a ``yes''-instance, which has polynomial length and can be checked in deterministic polynomial time (in $\|s\|$).
For example, deciding whether a Boolean formula $\phi(X)$ over the Boolean variables $X=\{x_1,\dots,x_n\}$ is satisfiable, i.e., whether there exists some truth assignment to these variables making $\phi$ true, is a well-known problem in $\NP$;
in fact, any satisfying truth assignment for $\phi$ is clearly a certificate that $\phi$ is a ``yes''-instance, i.e., that $\phi$ is satisfiable.

For a complexity class $C$, we denote by \textrm{co}-$C$ the complementary class to $C$, i.e., the class containing the complementary languages of those in $C$.
For example, the problem of deciding whether a Boolean formula $\phi$ is \emph{not} satisfiable is in $\CoNP$.
The class $\PTIME$ is contained in both $\NP$ and $\CoNP$, i.e., $\PTIME\subseteq \NP\cap\CoNP$.

By \LogSpace, we denote the set of decision problems that can be solved by deterministic Turing machines in logarithmic space.
For such machines, it is assumed that the input tape is read-only, and that these machine have a read/write tape, called work tape, for intermediate computations.
The logarithmic space bound is given on the space available on the work tape.
The class \LogSpace is contained in \PTIME.

The class $\DP{}$, defined originally in~\cite{Papadimitriou1984}, is the class of problems that are a ``conjunction'' of two problems, one from $\NP$ and one from $\CoNP$, i.e., $\DP{}=\{L \mid L = L' \cap L'', L'\in\NP, L''\in\CoNP\}$.
The class $\CoDP{}$ is the class of problems whose complements are in $\DP{}$, equivalently, it can be defined as the class of problems that are a ``disjunction'' of two problems, one from $\NP$ and one from $\CoNP$, i.e., $\CoDP{}=\{L \mid L = L' \cup L'', L'\in\NP, L''\in\CoNP\}$.

The classes $\SigmaP{k}$, $\PiP{k}$, and $\DeltaP{k}$, forming the \emph{polynomial hierarchy (PH)}~\cite{Stockmeyer1976}, are defined as follows: $\SigmaP{0} = \PiP{0} = \DeltaP{0} = \PTIME$, and, for all $k\geq 1$, $\SigmaP{k}=\NP^{\SigmaP{k-1}}$, $\DeltaP{k}=\PTIME^{\SigmaP{k-1}}$, and $\PiP{k}=\textnormal{co\nobreakdash-}\SigmaP{k}$.
Here, $\SigmaP{k}$ (resp., $\DeltaP{k}$) is  the set of decision problems solvable by nondeterministic (resp., deterministic) polynomial-time Turing machines with an oracle to recognize, at unit cost, a language in $\SigmaP{k-1}$.
Note that $\SigmaP{1} = \NP^{\SigmaP{0}} = \NP^\PTIME = \NP$, $\PiP{1} = \textnormal{co\nobreakdash-}\SigmaP{0} = \CoNP$, and $\DeltaP{1} = \PTIME^{\SigmaP{0}} = \PTIME^\PTIME = \PTIME$.
Sometimes a bound is imposed on the number of calls that are allowed to be issued to the oracle.
For example, $\ThetaP{k}=\PTIME^{\SigmaP{k-1}[O(\log n)]}$~denotes the set of decision problems solvable by a deterministic polynomial-time Turing machine that is allowed to query a $\SigmaP{k-1}$ oracle at most logarithmically many times (in the size of the input).
By definition, $\ThetaP{k}\subseteq\DeltaP{k}$.\footnote{For the complexity class $\ThetaP{k}$, an interesting characterization has recently been provided:
$\ThetaP{k}$ is the class of languages involving the counting and comparison of the number of ``yes''-instances in two sets containing instances of $\SigmaP{k-1}$ or $\PiP{k-1}$ languages~\cite{LukasiewiczMalizia-GeneralThetaPK}.
This is quite useful for reductions in voting settings where votes have to be counted and compared.}

The classes \DP{} and \CoDP{} can be generalized to the classes $\DP{k} = \{L \mid L = L' \cap L'', L'\in\SigmaP{k}, L''\in\PiP{k}\}$ and $\CoDP{k} = \{L \mid L = L' \cup L'', L'\in\SigmaP{k}, L''\in\PiP{k}\}$, respectively, for $k \geq 1$, that are the conjunction and the disjunction, respectively, of $\SigmaP{k}$ and $\PiP{k}$; in particular, $\DP1 = \DP{}$. 
Note also that $\DP{k} \subseteq \ThetaP{k+1}$.

\begin{figure}
	\centering%
	\includegraphics[width=0.75\textwidth]{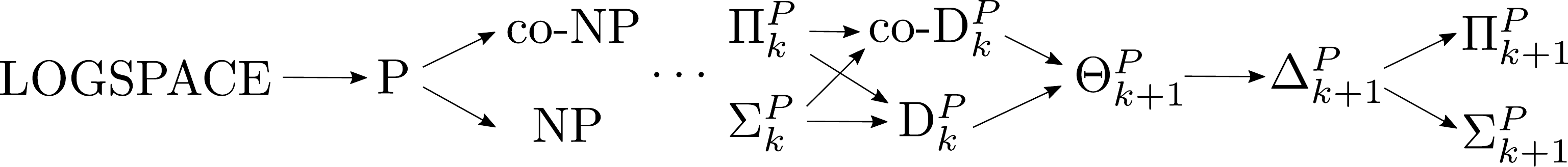}
	\caption{Inclusion relationships among complexity classes: an arrow from class $\mathcal{C}'$ to class $\mathcal{C}''$ means that $C' \subseteq C''$.}
	\label{fig:complexity_classes}
\end{figure}

Given their definitions, for all $k\geq 1$, the relationships among the mentioned classes are as follows (see \cref{fig:complexity_classes} for an illustration):
$(\SigmaP{k} \cup \PiP{k}) \subseteq \DP{k},\,\CoDP{k} \subseteq \ThetaP{k+1} \subseteq \DeltaP{k+1} \subseteq (\SigmaP{k+1} \cap \PiP{k+1})$ (see, e.g., \cite{Wagner1990,Wagner1988}).

\paragraph{Reductions and hard problems} 
A decision problem $L_1$ is \emph{(Karp) reducible} to a decision problem $L_2$, denoted $L_1 \leq L_2$, if there is a computable function~$h$, called \emph{(Karp) reduction}, such that, for every string $s$, $h(s)$ is defined, and $s$ is a ``yes''-instance of $L_1$ if and only if $h(s)$ is a ``yes''-instance of $L_2$.
A~decision problem~$L_1$ is \emph{polynomially (Karp) reducible} to a decision problem~$L_2$, denoted $L_1 \leq_p L_2$, if there is a polynomial-time (Karp) reduction from $L_1$ to $L_2$.
In this paper, we consider only 
Karp reductions.

To prove hardness for a complexity class, we show reductions from various problems known to be complete for the complexity classes that they belong to.
We next define such problems, so that we can later refer to them by name.

Deciding the satisfiability of Boolean formulas, denoted \Sat, is the prototypical \NPc problem, which remains \NPh even if only $3$CNF formulas are considered~\cite{Garey1979,Karp1972}, i.e., Boolean formulas in conjunctive normal form with three literals per clause.
The complementary problem \Unsat of deciding whether a given Boolean formula is \emph{not} satisfiable is \CoNPc.
It remains \CoNPh even if only $3$CNF formulas are considered, and it is the equivalent to the problem \Taut of deciding whether a $3$DNF formula is a tautology.
A $3$DNF formula is a Boolean formula in disjunctive normal form with three literals per term.
CNFs and DNFs are actually linked (see~\cite{Gottlob_Malizia2014,Gottlob_Malizia:DUAL_journal}).

The prototypical $\SigmaP{k}$- and $\PiPc{k}$ $\mathrm{QBF}_{Q_1,k}$ problems are defined as follows:
given a quantified Boolean~formula~(QBF) $\Phi = (Q_1 X_1)\linebreak[0] (Q_2 X_2)\linebreak[0] \dots\linebreak[0] (Q_k X_k)\linebreak[0] \phi(X_1,\linebreak[0] X_2,\linebreak[0] \dots,\linebreak[0] X_k)$, where
\begin{itemize}
\item $Q_1,Q_2,\ldots,Q_k$ is a sequence of $k$ alternating quantifiers $Q_i\in \{\exists,\forall\}$, and
\item $\phi(X_1,X_2,\dots,X_k)$ is a (non-quantified) Boolean formula over $k$ disjoint sets $X_1,X_2,\dots,X_k$ of Boolean variables,
\end{itemize}
decide whether $\Phi$ is valid. The problem $\mathrm{QBF}_{\exists,k}$ is $\SigmaPc{k}$ \cite{Stockmeyer1976,Wrathall1976},
while  $\mathrm{QBF}_{\forall,k}$ is \PiPc{k}~\cite{Stockmeyer1976,Wrathall1976}.
These problems remain hard for their respective classes even if $\phi(X_1,X_2,\dots,X_k)$ is in $3$CNF, when $Q_k = \exists$, and if $\phi(X_1,X_2,\dots,X_k)$ is in $3$DNF, when $Q_k = \forall$~\cite{Stockmeyer1976,Wrathall1976}.
We denote by QBF$^\mathit{CNF}_{k,\exists}$ (resp., QBF$^\mathit{DNF}_{k,\forall}$)%
\footnote{Note the difference in the subscripts of the notations $\mathrm{QBF}_{Q_1,k}$ and QBF$^\mathit{CNF}_{k,Q_k}$ (resp., QBF$^\mathit{DNF}_{k,Q_k}$). In the former notation, $Q_1$ is the \emph{first} quantifier of the sequence, and, for notational convenience, we place ``$Q_1$'' before ``$k$'' in the subscript. On the other hand, in the latter notation, $Q_k$ is the \emph{last} quantifier of the sequence, and, for notational convenience, we place ``$Q_k$'' after ``$k$'' in the subscript.} %
the problem of deciding the validity of formulas $\Phi = (Q_1 X_1)\linebreak[0] \dots \linebreak[0] (Q_k X_{k})\linebreak[0] \phi(X_1,\dots,\linebreak[0]X_{k})$, where $Q_k$ is $\exists$ (resp.,~$\forall$), and $\phi(X_1,\dots,X_{k})$ is in $3$CNF (resp., $3$DNF).
For odd $k$, QBF$^\mathit{CNF}_{k,\exists}$ (resp., QBF$^\mathit{DNF}_{k,\forall}$) is complete for $\SigmaP{k}$ (resp., $\PiP{k}$),
while, for even~$k$, QBF$^\mathit{CNF}_{k,\exists}$ (resp., QBF$^\mathit{DNF}_{k,\forall}$) is complete for $\PiP{k}$ (resp., $\SigmaP{k}$).
Observe that \QbfECNF1 (resp., \QbfADNF1) is equivalent to \Sat (resp., \Taut).

Sometimes, it is preferable that in QBF formulas the non-quantified formula is CNF rather than DNF, or vice-versa.
For example, to show \SigmaPh{2}{}ness it might be the case that we would prefer to start our reduction from formulas $\Phi = (\exists X)(\forall Y)\phi(X,Y)$ with $\phi(X,Y)$ being in CNF, rather than in DNF, as required by QBF$^\mathit{DNF}_{2,\forall}$.
To achieve this, we can exploit De~Morgan's laws.
Indeed, we have that $\Phi = (Q_1 X_1)\linebreak[0] \dots \linebreak[0] (Q_k X_{k})\linebreak[0] \phi(X_1,\dots,\linebreak[0]X_{k})$ is logically equivalent to $\Phi' = (Q_1 X_1)\linebreak[0] \dots \linebreak[0] (Q_k X_{k})\linebreak[0] \lnot \phi'(X_1,\dots,\linebreak[0]X_{k})$, where $\phi'(X_1,\dots,\linebreak[0]X_{k}) = \lnot \phi(X_1,\dots,\linebreak[0]X_{k})$.
We thus extend the notation above.
We denote by QBF$^{\mathit{DNF}}_{k,\exists,\lnot}$ (resp., QBF$^{\mathit{CNF}}_{k,\forall,\lnot}$) the problem of deciding the validity of formulas $\Phi = (Q_1 X_1)\linebreak[0] \dots \linebreak[0] (Q_k X_{k})\linebreak[0] \lnot \phi(X_1,\dots,\linebreak[0]X_{k})$, where $Q_k$ is $\exists$ (resp.,~$\forall$), and $\phi(X_1,\dots,X_{k})$ is in $3$DNF (resp., $3$CNF).
For odd $k$, QBF$^{\mathit{DNF}}_{k,\exists,\lnot}$ (resp., QBF$^{\mathit{CNF}}_{k,\forall,\lnot}$) is complete for $\SigmaP{k}$ (resp., $\PiP{k}$),
while, for even~$k$, QBF$^{\mathit{DNF}}_{k,\exists,\lnot}$ (resp., QBF$^{\mathit{CNF}}_{k,\forall,\lnot}$) is complete for $\PiP{k}$ (resp., $\SigmaP{k}$).

\subsection{A framework for preference representation schemes}\label{sec_framework_representations}
In this paper, most of the membership results that we will show hold for generic preference representation schemes.
For this reason,  we now introduce the general framework of representation schemes that we will refer to.
Inspired by the concept of compact representations in~\cite{core_bargset_kernel-rivista,nucleolus-rivista,Lang_handbook,Lang_ai_magazine}, we define preference representation schemes $\RepScheme{S}$ as suitable encodings for a class of preference relations, denoted $\mathcal{C}(\RepScheme{S})$.
Formally, a \emph{preference representation scheme} $\RepScheme{S}$ defines a computable representation function $\xi^{\RepScheme{S}}(\cdot)$ and a computable Boolean  function $\mathit{Pref}^\RepScheme{S}(\cdot,\cdot,\cdot)$ such that, for any relation $\Ranking{R} \in \mathcal{C}(\RepScheme{S})$, $\xi^\RepScheme{S}(\Ranking{R})$ is the encoding of $\Ranking{R}$ according to $\RepScheme{S}$, and $\mathit{Pref}^\RepScheme{S}(\xi^\RepScheme{S}(\Ranking{R}),\alpha,\beta)$ evaluates to $1$, if $\tuple{\alpha,\beta}\in\Ranking{R}$ (i.e., if $\alpha\Pref_{\Ranking{R}}\beta$), and to $0$, otherwise.
By $\|\xi^\RepScheme{S}(\Ranking{R})\|$, we denote the size of the representation of $\Ranking{R}$ via $\RepScheme{S}$.

Let $\RepScheme{S}_1$ and $\RepScheme{S}_2$ be two preference representation schemes.
We say that $\RepScheme{S}_2$ is \emph{at least as expressive} (and \emph{succinct}) \emph{as} $\RepScheme{S}_1$, denoted $\RepScheme{S}_1\precsim_e\RepScheme{S}_2$, if there exists a function $f$ in \FP (i.e., computable in deterministic polynomial time) that translates a preference relation $\xi^{\RepScheme{S}_1}(\Ranking{R})$ represented in $\RepScheme{S}_1$ into an equivalent preference relation $\xi^{\RepScheme{S}_2}(\Ranking{R})$ represented in $\RepScheme{S}_2$, i.e., into a preference relation over the same outcomes and with the same preference relationships between them.
More precisely, we require that $\xi^{\RepScheme{S}_2}(\Ranking{R}) = f(\xi^{\RepScheme{S}_1}(\Ranking{R}))$ and $\mathit{Pref}^{\RepScheme{S}_1}(\xi^{\RepScheme{S}_1}(\Ranking{R}),\alpha,\beta) = \mathit{Pref}^{\RepScheme{S}_2}(\xi^{\RepScheme{S}_2}(\Ranking{R}),\alpha,\beta)$, for each pair of outcomes $\alpha$ and $\beta$.
Observe that $f$ belonging to \FP entails that there exists a constant $c_f$ (depending on $f$) such that $\|\xi^{\RepScheme{S}_2}(\Ranking{R})\| \leq {\|\xi^{\RepScheme{S}_1}(\Ranking{R})\|}^{c_f}$, i.e., the size of $\xi^{\RepScheme{S}_2}(\Ranking{R})$ is polynomially bounded in the size of $\xi^{\RepScheme{S}_1}(\Ranking{R})$.%
\footnote{Note that the above definitions are slightly different from the ones in~\cite{Lang_handbook}: the counterpart of this paper's function $\mathit{Pref}$ in~\cite{Lang_handbook} is not required to be computable, and the transformation function $f$ in~\cite{Lang_handbook} is only required to be polynomially bounded, but not polynomially computable.}

A \emph{\PTIME-} and an \emph{\NP-representation} $\RepScheme{S}$ is a preference representation scheme whose function $\mathit{Pref}^\RepScheme{S}$ is in \PTIME and  \NP, respectively.
For example, the polynomially connected classes of acyclic binary \CPNets are \NP-representation schemes.

When a compact representation $\RepScheme{S}$ is clear from the context, we often simply write $\Ranking{R}$ instead of $\xi^\RepScheme{S}(\Ranking{R})$, and $\alpha\PrefRanking{\Ranking{R}}\beta$ and $\alpha\not\PrefRanking{\Ranking{R}}\beta$ instead of $\mathit{Pref}^\RepScheme{S}(\xi^\RepScheme{S}(\Ranking{R}),\alpha,\beta) = 1$ and $\mathit{Pref}^\RepScheme{S}(\xi^\RepScheme{S}(\Ranking{R}),\alpha,\beta) = 0$, respectively.
While doing so, we are identifying the preference relation with its actual representation.

\section{Complexity of basic tasks on CP-nets}\label{sec_complexity_cpnets}
To precisely characterize the complexity of Pareto and majority voting tasks in \cref{sec_complexity_pareto_voting,sec_complexity_majority_voting}, we need to understand how complex is deciding, given a \CPNet $\Net{N}$ and two outcomes $\alpha$ and $\beta$, whether $\alpha$ dominates $\beta$ or whether $\alpha$ and $\beta$ are incomparable.
Here, we prove that the former problem is \NPc, while the latter is \CoNPc.
To achieve this, after giving some preliminary definitions on how to encode Boolean formulas into \CPNets, we  show that deciding the satisfiability of Boolean formula can be reduced to the problem of deciding dominance between outcomes in \CPNets.
This allows us to prove the \NPh{}ness of the dominance test in \CPNets, and the \CoNPh{}ness of deciding incomparability is shown as a byproduct of this property.

\subsection{Preliminaries}\label{sec_formula_net}

We first introduce a notation 
mapping Boolean assignments to outcomes of CP-nets; this notation will frequently be used later in the paper.
In particular, to prove the hardness of voting tasks on \MCPNets, we often provide reductions from problems regarding the satisfiability (or validity) of (quantified) Boolean formulas.
For this reason, \MCPNets will often have sets of features associated with sets of Boolean variables.
For example, for a Boolean formula $\phi(X)$  over the set of Boolean variables $X=\{x_1,\dots,x_n\}$, we often define an \MCPNet $\MNet{M}(\phi)$ that has  $\SetFeat{V}=\bigcup \{\{V_i^T,V_i^F\}\mid x_i\in X\}$ as
	a subset of its set of features.
Then, for a (partial or complete) assignment $\sigma_X$ over $X$, an outcome $\alpha_{\sigma_X}$ of $\MNet{M}(\phi)$ encoding $\sigma_X$ over the features set $\SetFeat{V}$ is such that, for the features in~$\SetFeat{V}$, if $\sigma_X[x_i] = \valtrue$, then $\alpha_{\sigma_X}[V_i^T V_i^F] = \ol{v_i^T} v_i^F$; if $\sigma_X[x_i] = \valfalse$, then $\alpha_{\sigma_X}[V_i^T V_i^F] = v_i^T \ol{v_i^F}$; and if $\sigma_X[x_i]$ is undefined, then $\alpha_{\sigma_X}[V_i^T V_i^F] = v_i^T v_i^F$.
The values of the features $F\notin\SetFeat{V}$ in $\alpha_{\sigma_X}$ will be specified in each particular case.

We next define \emph{formula nets}, which will be used in hardness proofs, and which are intuitively \CPNets aiming at having a particular preference relationship between two outcomes depending on the satisfiability of associated Boolean formulas in CNF.
This will allow us to show that deciding dominance in \CPNets is \NPh.\footnote{The \NPh{}ness of dominance in \CPNets was proven already in~\cite{Boutilier2004}. Here, we show this result, because the construction proposed here, which allows us to prove a stricter result, is different from the one available in the literature, and it is required in multiple reductions in the rest of the paper.}

Formally, let $\phi(X)$ be a Boolean formula in $3$CNF defined over the set of Boolean variables $X=\{x_1,\dots,x_n\}$, and whose set of clauses is $C=\{c_1,\dots,c_m\}$.
We often omit the variable set from the notation of the Boolean formula, i.e., we write $\phi$ instead of $\phi(X)$, if this does not cause ambiguity.
We denote by $\ell_{j,k}$ the $k$-th literal of the $j$-th clause.
From $\phi(X)$, we build the \CPNet $\NetCNF(\phi)$ in the following way
(see \cref{fig:formula_net} for an example).

\begin{figure}
	\centering%
	\includegraphics[width=0.95\textwidth]{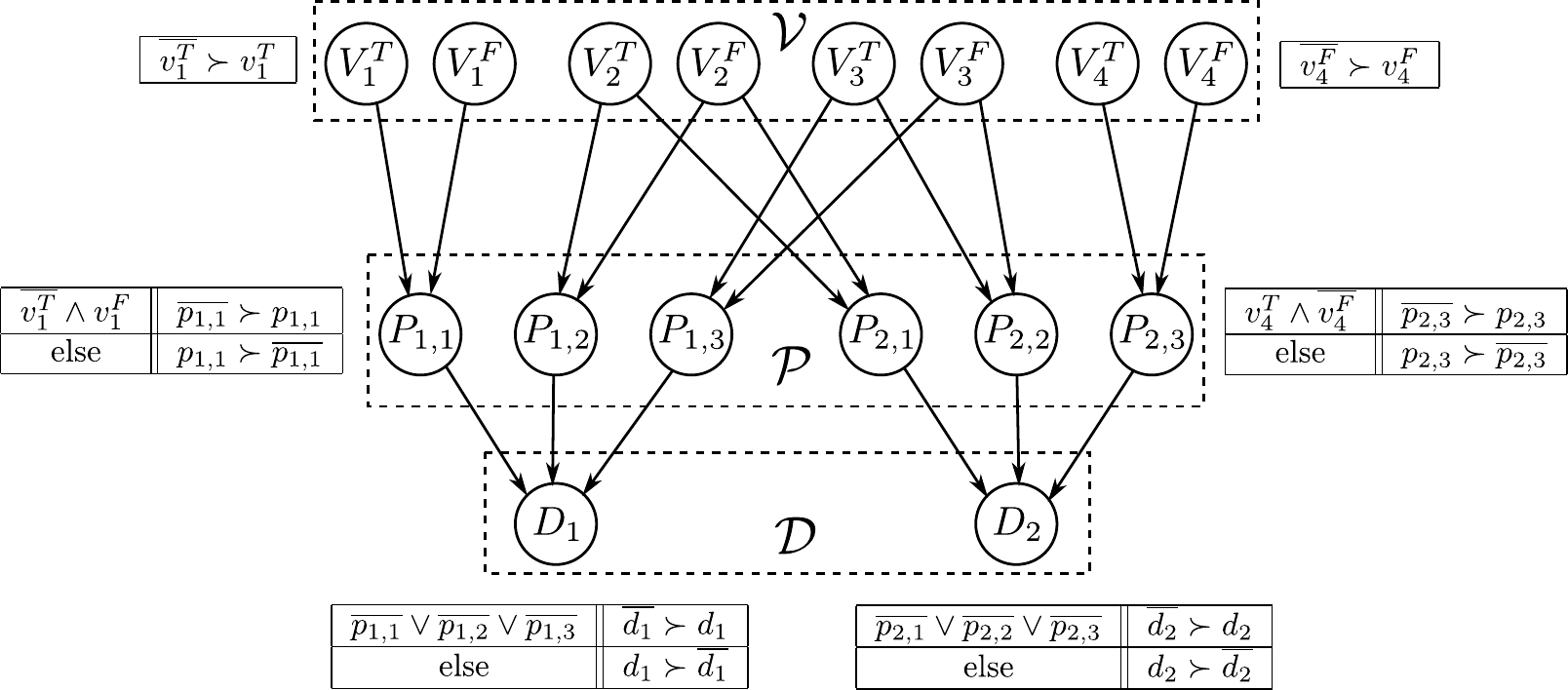}
	\caption{The \CPNet $\NetCNF(\phi)$, where $\phi(x_1,x_2,x_3,x_4)=(x_1\lor x_2 \lor \lnot x_3)\land (\lnot x_2 \lor x_3 \lor \lnot x_4)$. Not all the CP tables are reported in the figure.}\label{fig:formula_net}
\end{figure}

The features of $\NetCNF(\phi)$ are:
\begin{itemize}
	\item for each variable $x_i\in X$, there are features $V_i^T$ and $V_i^F$ (called \emph{variable features}), and we denote by $\SetFeat{V}_\phi$ the set of variable features;
	
	\item for each clause $c_j\in C$, there is a feature $D_j$ (called \emph{clause feature}), and we denote by $\SetFeat{D}_\phi$ the set of clause features; and
	
	\item for each literal $\ell_{j,k}$, there is a feature $P_{j,k}$ (called \emph{literal feature}), and we denote by $\SetFeat{P}_\phi$ the set of literal features.
\end{itemize}
All features are binary, with the usual notation for their values.
When the formula $\phi$ is clear from the context, we often omit the subscript ``$\phi$'' from the notation of the sets of features illustrated above.
The edges of $\NetCNF(\phi)$ are: for each literal $\ell_{j,k}=x_i$ or $\ell_{j,k}=\lnot x_i$, there are edges $\tuple{V_i^T,P_{j,k}}$, $\tuple{V_i^F,P_{j,k}}$, and $\tuple{P_{j,k},D_j}$.
The CP tables of $\NetCNF(\phi)$ are:
\begin{itemize}
	\item for each variable $x_i\in X$, features $V_i^T$ and $V_i^F$ have the CP tables

\smallskip 	\begin{tabular}{|l|}
		\hline
		$\ol{v_i^T}\Pref v_i^T$\\
		\hline
	\end{tabular} and
	\begin{tabular}{|l|}
		\hline
		$\ol{v_i^F}\Pref v_i^F$\\
		\hline
	\end{tabular}, respectively;
	
	\item for each literal $\ell_{j,k}$, if $\ell_{j,k}=x_i$, then feature $P_{j,k}$ has the CP table

\smallskip 	\begin{tabular}{|c||l|}
		\hline
		$\ol{v_i^T}\land v_i^F$ & $\ol{p_{j,k}}\Pref p_{j,k}$\\
		\hline
		else & $p_{j,k}\Pref \ol{p_{j,k}}$\\
		\hline
	\end{tabular};
	
	\smallskip
	otherwise  (i.e., $\ell_{j,k}=\lnot x_i$) $P_{j,k}$ has the CP table
	
	\smallskip
	\begin{tabular}{|c||l|}
		\hline
		$v_i^T\land \ol{v_i^F}$ & $\ol{p_{j,k}}\Pref p_{j,k}$\\
		\hline
		else & $p_{j,k}\Pref \ol{p_{j,k}}$\\
		\hline
	\end{tabular};
	
	\item for each clause $c_j\in C$, feature $D_j$ has the CP table
	
	\smallskip
	\begin{tabular}{|c||l|}
		\hline
		$\ol{p_{j,1}}\lor \ol{p_{j,2}}\lor \ol{p_{j,3}}$ & $\ol{d_j}\Pref d_j$\\
		\hline
		else & $d_j\Pref \ol{d_j}$\\
		\hline
	\end{tabular}.
\end{itemize}

Note that $\NetCNF(\phi)$ is binary, acyclic, singly connected, its indegree is three, and the \CPNet can be built in polynomial time in the size of $\phi$.

The following \lcnamecref{lemma_Pref_Net_CNF_new} and \lcnamecref{corol:Pref_Net_CNF} show an important property of formula nets, which is that $\phi$ is satisfiable if and only if a particular outcome dominates others in $\NetCNF(\phi)$.

\begin{lemma}\label{lemma_Pref_Net_CNF_new}
Let $\phi(X)$ be a Boolean formula in $3$CNF defined over a set $X$ of Boolean variables, and let $\sigma_{X}$ be an assignment on $X$.
Let $\alpha_{\sigma_X}$ be the outcome of $\NetCNF(\phi)$ encoding $\sigma_X$ on the feature set $\SetFeat{V}$, and assigning non-overlined values to all other features, and let $\ol{\beta}$ be the outcome assigning overlined values to all and only variable and clause features.
Then:
\begin{itemize}
\item[(1)] There is an extension of $\sigma_{X}$ to $X$ satisfying $\phi(X)$ if and only if $\ol{\beta}\PrefNet{\NetCNF(\phi)}\alpha_{\sigma_X}$;
\item[(2)] There is no extension of $\sigma_{X}$ to $X$ satisfying $\phi(X)$ if and only if  $\ol{\beta}\IncompNet{\NetCNF(\phi)}\alpha_{\sigma_X}$.
\end{itemize}
\end{lemma}

\begin{proof}[Proof (sketch).]
	The idea at the base of this proof is that the CP tables in $\NetCNF(\phi)$ are designed so that the features enact the role of variables, literals, and clauses of a CNF Boolean formula.
	Details of the proof are at page~\pageref{page_pointer:Pref_Net_CNF_new_detailed}.
\end{proof}

\begin{corollary}\label{corol:Pref_Net_CNF}
Let $\phi(X)$ be a Boolean formula in $3$CNF defined over a set $X$ of Boolean variables, and let $\alpha$ and $\ol{\beta}$ be two outcomes of $\OutNet{\NetCNF(\phi)}$ assigning non-overlined values to all features and overlined values to all and only variable and clause features, respectively.
	Then:
	\begin{itemize}
		\item $\phi$ is satisfiable if and only if $\ol{\beta}\PrefNet{\NetCNF(\phi)}\alpha$;
		\item $\phi$ is unsatisfiable if and only if $\ol{\beta}\IncompNet{\NetCNF(\phi)}\alpha$.
	\end{itemize}
\end{corollary}
\begin{proof}
	Observe that, when the empty assignment $\sigma_X$ is considered, outcomes $\alpha_{\sigma_X}$ and $\ol{\beta}_{\sigma_X}$ of the statement of \cref{lemma_Pref_Net_CNF_new} coincide with outcomes $\alpha$ and $\ol{\beta}$ of the statement of this corollary, respectively.
	Moreover, any assignment over $X$ is an extension of the empty one, and hence \cref{lemma_Pref_Net_CNF_new} applies.
\end{proof}

\subsection{Complexity of dominance, incomparability, and optimality on CP-nets}

As mentioned above, via formula nets, it is possible to show the \NPh{}ness and the \CoNPh{}ness of dominance and incomparability on \CPNets, respectively.
We start by showing the \NPh{}ness of dominance on \CPNets.
More formally, consider the following problem on \CPNets.

\medbreak

\begin{tabular}{rl}
  \textit{Problem:} & \Dominance\\
  \textit{Instance:} & A \CPNet $\Net{N}$, and two outcomes $\alpha,\beta\in\OutNet{\Net{N}}$.\\
  \textit{Question:} & Is $\beta\PrefNet{\Net{N}}\alpha$?
\end{tabular}

\medbreak

\Dominance is known to be feasible in \NP for some classes of instances, and in particular it is in \NP for polynomially connected classes of acyclic binary \CPNets~\cite{Boutilier2004}.
For these classes of \CPNets, it was shown that \Dominance is \NPh, and hardness holds even if the considered \CPNets are singly connected and the indegree of each feature in the net is at most six~\cite{Boutilier2004}.
Moreover, \Dominance is feasible in polynomial time on acyclic binary \CPNets whose graph is a tree or a polytree~\cite{Boutilier2004}.
However, the exact complexity of \Dominance for general (non polynomially connected classes of) acyclic binary \CPNets is still an open problem~\cite{Boutilier2004}, and in particular it is unknown whether it belongs to \NP or not.

First, we give an improved result on the hardness of dominance testing in \CPNets.
In particular, we show that the \NPh{}ness holds even if the indegree of the \CPNet is three, while the minimum indegree previously required to show the hardness was six~\cite{Boutilier2004}.

\begin{theorem}[improved over \cite{Boutilier2004}]\label{theo_dominance_cpnets_hardness}
Let $\Net{N}$ be a \CPNet, and let $\alpha,\beta\in\OutNet{\Net{N}}$ be two outcomes.
Deciding whether $\beta\PrefNet{\Net{N}}\alpha$ is \NPh.
Hardness holds even if $\Net{N}$ is acyclic, binary, singly connected, and its indegree is three.
\end{theorem}
\begin{proof}
To show that \Dominance is \NPh, we prove that \Sat $\leq_p$ \Dominance.
Let $\phi$ be a Boolean formula in $3$CNF.
Consider the \CPNet $\NetCNF(\phi)$ (defined in \cref{sec_formula_net}) and outcomes $\alpha$ and $\beta$ such that in $\alpha$ the values of all features are non-overlined, and in $\beta$ the values of all and only variable and clause features are overlined.
By \cref{corol:Pref_Net_CNF}, $\phi$ is satisfiable if and only if $\beta\PrefNet{\NetCNF(\phi)}\alpha$.
\end{proof}

We now focus on the problem of testing incomparability between outcomes in \CPNets and show its \CoNPc{}ness.
We  show hardness via the properties of formula nets.
More formally, consider the following problem.

\medbreak

\begin{tabular}{rl}
  \textit{Problem:} & \Incomparability\\
  \textit{Instance:} & A \CPNet $\Net{N}$, and two outcomes $\alpha,\beta\in\OutNet{\Net{N}}$.\\
  \textit{Question:} & Is $\alpha \IncompNet{\Net{N}} \beta$?
\end{tabular}

\medbreak

The following theorem shows that deciding whether two outcomes are incomparable in a preference relation represented via an \NP-representation scheme is in \CoNP.

\begin{theorem}\label{theo_incomparability_generic_membership}
Let $\Ranking{R}$ be a preference relation represented via an \NP-representation scheme, and let $\alpha$ and $\beta$ be two outcomes.
Deciding whether $\alpha\Incomp_{\Ranking{R}}\beta$ is feasible in \CoNP.
\end{theorem}
\begin{proof}
We show that disproving $\alpha\Incomp\beta$ is feasible in \NP.
In fact, if $\alpha\not\Incomp\beta$, then either $\alpha\Pref\beta$ or $\beta\Pref\alpha$.
In these cases, since $\Ranking{R}$ is represented via an \NP-representation, there is a polynomial certificate either witnessing $\alpha\Pref\beta$ or witnessing $\beta\Pref\alpha$. 
To conclude, observe that such a certificate can be checked in polynomial time.
\end{proof}

We now focus on \CPNets and show that deciding incomparability is \CoNPh.

\begin{theorem}\label{theo_incomparability_cpnets_hardness}
Let $\Net{N}$ be a \CPNet, and let $\alpha$ and $\beta$ be two outcomes.
Deciding whether $\alpha\IncompNet{\Net{N}}\beta$ is \CoNPh.
Hardness holds even if $\Net{N}$ is acyclic, binary, singly connected, and its indegree is three.
\end{theorem}
\begin{proof}
To show that \Incomparability is \CoNPh, we prove that \Unsat $\leq_p$ \Incomparability.
Let $\phi$ be a Boolean formula in $3$CNF.
Consider the \CPNet $\NetCNF(\phi)$ (defined in \cref{sec_formula_net}) and outcomes $\alpha$ and $\beta$ such that in $\alpha$ the values of all features are non-overlined, and in $\beta$ the values of all and only variable and clause features are overlined.
By \cref{corol:Pref_Net_CNF}, $\phi$ is unsatisfiable if and only if $\alpha\IncompNet{\NetCNF(\phi)}\beta$.
\end{proof}

By combining the previous two results, we obtain that testing incomparability over polynomially connected classes of acyclic \CPNets is \CoNPc.

\begin{corollary}
Let $\mathcal{C}$ be a polynomially connected class of acyclic \CPNets.
Let $\Net{N}\in\mathcal{C}$ be a \CPNet, and let $\alpha$ and $\beta$ be two outcomes.
Deciding whether $\alpha\IncompNet{\Net{N}}\beta$ is \CoNPc.
\end{corollary}

We emphasize here that checking the incomparability between two outcomes is different from deciding the ``ordering query'' defined in~\cite{Boutilier2004}.
For a \CPNet $\Net{N}$ and two outcomes $\alpha$ and $\beta$, an ordering query is deciding whether there is at least a preference raking $\Ranking{R}$ satisfying $\Net{N}$ such that $\alpha\PrefRanking{\Ranking{R}}\beta$.
As noticed in~\cite{Boutilier2004}, this is tantamount to decide whether $\Net{N}\not\models\beta\Pref\alpha$.
Since $\Net{N}\not\models\beta\Pref\alpha$ if and only if $\beta\not\PrefNet{\Net{N}}\alpha$, deciding an ordering query is actually \CoNPh, because dominance testing is \NPh. 
Therefore, it comes as no surprise that the polynomial algorithm proposed in~\cite{Boutilier2004} to decide ordering queries is actually ``partially complete'' (as said in~\cite{Boutilier2004}).
In fact, given the \CoNPh{}ness of the ordering query problem, there is no sound and complete deterministic polynomial-time algorithm for this problem (unless a major breakthrough in complexity theory occurs, showing that $\PTIME = \NP$).

We conclude this section by looking at the complexity of deciding whether an outcome is optimal in a \CPNet.
This result is needed in \cref{sec_complexity_pareto_voting} to characterize the complexity of one of the voting tasks of Pareto voting.
Here, we show that this problem can be decided in \LogSpace.
More formally, consider the following problem.

\medbreak

\begin{tabular}{rl}
  \textit{Problem:} & \Optimality\\
  \textit{Instance:} & A \CPNet $\Net{N}$, and an outcome $\alpha\in\OutNet{\Net{N}}$.\\
  \textit{Question:} & Is $\alpha$ optimal in $\Net{N}$?
\end{tabular}

\medbreak

Recall that for acyclic \CPNets, there is an outcome that is optimum~\cite{Boutilier2004}.
Clearly, this outcome is also the only optimal one in a \CPNet.
Therefore, checking whether an outcome is optimal is tantamount to checking whether the outcome is optimum.
It was shown that, given an acyclic \CPNet $\Net{N}$, computing the unique optimal outcome $\OptOutNet{\Net{N}}$ of $\Net{N}$ is feasible in deterministic polynomial time, more precisely in linear time, through the ``forward sweep'' procedure~\cite{Boutilier2004}.
Hence, as pointed out in~\cite{Rossi2004}, a simple procedure to decide whether a given outcome $\alpha$ is optimal in an acyclic \CPNet~$\Net{N}$ is to compute $\OptOutNet{\Net{N}}$ (in polynomial time) and then to compare $\alpha$ to $\OptOutNet{\Net{N}}$.
However,  this problem actually belongs to a complexity class that is a below \PTIME when acyclic \CPNets are considered.

\begin{theorem}\label{theo_individual_optimality_membership}
Let $\Net{N}$ be an acyclic \CPNet, and let $\alpha\in\OutNet{\Net{N}}$ be an outcome.
Deciding whether $\alpha$ is optimal in $\Net{N}$ is feasible in \LogSpace.
\end{theorem}
\begin{proof}
Being $\Net{N}$ acyclic, if an outcome $\alpha$ is not optimal in $\Net{N}$, then there is an improving flipping sequence from $\alpha$ to the optimum outcome, and hence there is at least a feature whose value can be flipped in $\alpha$ to obtain a better outcome.
Therefore, to decide whether $\alpha$ is optimal, it suffices to consider in turn all features $F$ and check whether it is possible to perform an improving flip according to the CP table of $F$.
If no feature can be flipped to improve the outcome, then~$\alpha$ is optimal.
Clearly, this procedure requires only logarithmic space to be carried out.
\end{proof}

Now that we have analyzed the complexity of dominance, incomparability, and optimality in \CPNets, we can devote our focus to the complexity of Pareto and majority voting on \CPNets in the next two sections.

\section{Complexity of Pareto voting on \texorpdfstring{$m$}{m}CP-nets}\label{sec_complexity_pareto_voting}\label{sec5}

In this section, we characterize the complexity of Pareto voting tasks on \MCPNets.
In particular, after giving some preliminaries on specific structures of \CPNets that we will use in our reductions, we analyze the complexity of Pareto dominance, which is proven \NPc.
Then, we devote our analysis to the problems related to Pareto optimal outcomes, namely, deciding whether an outcome is Pareto optimal, and deciding whether an \MCPNet has a Pareto optimal outcome.
We prove the former \CoNPc, while for the latter, we are able to show that every \MCPNet has a Pareto optimal outcome, which implies that the problem is trivial (i.e., feasible in constant time).
To conclude, we study the complexity of problems on Pareto optimum outcomes, namely, deciding whether an outcome is Pareto optimum, and deciding whether an \MCPNet has a Pareto optimum outcome.
Both problems are proven to be tractable, and in particular we  show the former to be in feasible in \LogSpace and the latter to be feasible in \PTIME.
We recall that the Pareto voting semantics is based on the concept of unanimity, i.e., given an \MCPNet $\MNet{M}$ and two outcomes $\alpha,\beta\in\OutMNet{\MNet{M}}$, it holds that~$\beta\PrefParetoMNet{\MNet{M}}\alpha$, if all agents prefer $\beta$ to $\alpha$, i.e., $|\SetAgentsPrefMNet{\MNet{M}}(\beta,\alpha)|=m$.

\subsection{Preliminaries}
\label{sec_interconnecting_net}
We now introduce a specific structure of \CPNets that will be used in the forthcoming reductions.
In particular, the \emph{conjunctive} $\NetInterAND(\cdot)$ and \emph{disjunctive $\NetInterOR(\cdot)$ interconnecting nets} are CP-nets whose role is intuitively to link different parts of bigger \CPNets:
given a set $\SetFeat{S}$ of $m$ features, the aim of the \CPNets $\NetInterAND(m)$ and $\NetInterOR(m)$ is to propagate the information that all features of $\SetFeat{S}$ and at least one of the features of $\SetFeat{S}$, respectively, have been flipped to their overlined value.

\begin{figure}
	\centering%
	\includegraphics[width=0.8\textwidth]{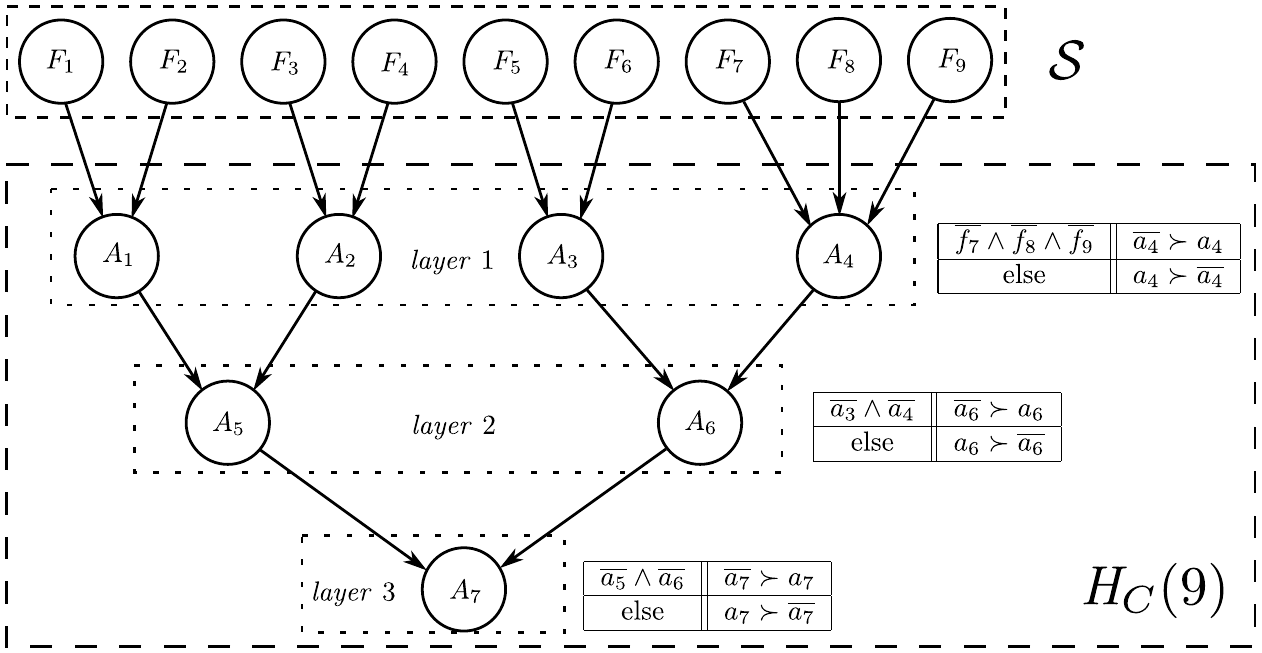}
	\caption{An interconnecting net $\NetInterAND(9)$. Not all the CP tables are reported in the figure.}\label{fig:interconnect_H}
\end{figure}

We first introduce the conjunctive interconnecting \CPNet $\NetInterAND(m)$.
See \cref{fig:interconnect_H} for an example of a $\NetInterAND(9)$ inter\-connecting net.
Such a \CPNet is partitioned into layers, and it is a kind of ``inverted pyramid''.
In particular, $\NetInterAND(m)$ is an acyclic DAG in which each feature of a layer has two or three distinct parents in the previous layer, and at most one child in the next layer.
Features belonging to the same layer have no parents in common, and in every layer at most one feature has three parents.
The first layer of $\NetInterAND(m)$ is attached to a set $\SetFeat{S}$ of $m$ different features of the \CPNet that we want to interconnect.
In the first layer, the above described connection properties hold relative to the features of $\SetFeat{S}$.
The layer with a unique feature, which we call \emph{apex}, is the last layer of the \CPNet.
All the features of $\NetInterAND(m)$ are named $A_i$ with a proper increasing index $i$, and their values are, as usual, $\{a_i,\ol{a_i}\}$.
The CP table for a feature $A_i$ states that value $\ol{a_i}$ is preferred to $a_i$ whenever the value of every parent of $A_i$ is overlined.
Otherwise, the value $a_i$ is preferred to $\ol{a_i}$.
Let $\alpha$ be the outcome in which the values of all the features $A_i$ are non-overlined.
It is not difficult to see that, whenever all the features in $\SetFeat{S}$ have overlined values in $\alpha$, there is an improving flipping sequence starting from $\alpha$ changing the values of all $A_i$ (and hence also the value of the apex) to their overlined values.

The disjunctive interconnecting net $\NetInterOR(m)$ is similar to $\NetInterAND(m)$.
The features and the structure of $\NetInterOR(m)$ are the same as those in $\NetInterAND(m)$.
The only variations are on the CP tables.
Being $\NetInterOR(m)$ a disjunctive net, given a feature $A_i$ of $\NetInterOR(m)$, value $\ol{a_i}$ is preferred to value $a_i$, whenever at least one of the parents of $A_i$ has been flipped to its overlined value.
Let $\alpha$ be the outcome in which the values of all the features $A_i$ are non-overlined.
It is easy to see that, whenever at least a feature in $\SetFeat{S}$ has an overlined value in $\alpha$, there is an improving flipping sequence starting from $\alpha$ and reaching an outcome in which the apex has an overlined value.

Note that $\NetInterAND(m)$ and $\NetInterOR(m)$ are binary, acyclic, singly connected, their indegree is at most three, and the \CPNets can be built in polynomial time in $|\SetFeat{S}|=m$, as the number of their features is polynomial in $m$ (in particular, strictly less than $m$), and each feature has a bounded number of parents which translates into CP tables of bounded sizes.

\subsection{Complexity of Pareto dominance on \texorpdfstring{$m$}{m}CP-nets}
First, we analyze the problem of deciding Pareto dominance on \CPNets, which is shown \NPc.
More formally, consider the following problem.

\medbreak

\begin{tabular}{rl}
  \textit{Problem:} & \ParetoQuery\\
  \textit{Instance:} & An \MCPNet $\MNet{M}$, and two outcomes $\alpha,\beta\in\OutMNet{\MNet{M}}$.\\
  \textit{Question:} & Is $\beta \PrefParetoMNet{\MNet{M}} \alpha$?
\end{tabular}

\medbreak

The following \lcnamecref{theo_pareto_query_generic_membership} shows that deciding Pareto dominance over preference profiles represented via an \NP-representation scheme is in \NP.

\begin{theorem}\label{theo_pareto_query_generic_membership}
Let $\Profile{P}$ be a preference profile defined over the same combinatorial domain and represented via an \NP-representation scheme, and let $\alpha$ and $\beta$ be two outcomes.
Then, deciding whether $\beta \PrefParetoProfile{\Profile{P}} \alpha$ is feasible in \NP.
\end{theorem}
\begin{proof}
To show that this problem resides in \NP, we exhibit a concise certificate for it.
Let $\Profile{P}=\tuple{\Ranking{R}_1,\dots,\Ranking{R}_m}$.
If $\beta \PrefParetoProfile{\Profile{P}} \alpha$, then, for all $1\leq i\leq m$, it holds that $\beta \PrefRanking{\Ranking{R}_i}\alpha$.
Since, for every agent $i$, $\Ranking{R}_i$ is represented via an \NP-representation scheme, there is a concise certificate witnessing that $\beta \PrefRanking{\Ranking{R}_i}\alpha$.
Therefore, in order to decide whether $\beta \PrefParetoProfile{\Profile{P}} \alpha$, it suffices to guess and subsequently check the polynomial witnesses that, for each agent $i$, it holds that $\beta \PrefRanking{\Ranking{R}_i} \alpha$.
The overall guess requires only polynomial space, and it can be checked in polynomial time.
\end{proof}

Observe that on $1$\CPNets, Pareto dominance is equivalent to dominance (on simple \CPNets).
Therefore, the following result, which follows directly from \cref{theo_dominance_cpnets_hardness}, shows that deciding Pareto dominance in \MCPNets is \NPh.

\begin{theorem}\label{theo_pareto_query_mcpnets_hardness}
Let $\MNet{M}$ be an \MCPNet, and let $\alpha$ and $\beta$ be two outcomes.
Then, deciding whether $\beta\PrefParetoMNet{\MNet{M}}\alpha$ is \NPh.
Hardness holds even on classes of singly connected acyclic binary \MCPNets with indegree at most three and at most one agent.
\end{theorem}

By combining the two previous results, we immediately obtain that Pareto dominance over polynomially connected classes of acyclic \MCPNets is \NPc.

\begin{corollary}
Let $\mathcal{C}$ be a polynomially connected class of acyclic \MCPNets.
Let $\MNet{M}\in\mathcal{C}$ be an \MCPNet, and let $\alpha$ and $\beta$ be two outcomes.
Then, deciding whether $\beta\PrefParetoMNet{\MNet{M}}\alpha$ is \NPc.
\end{corollary}

\subsection{Complexity of Pareto optimality on \texorpdfstring{$m$}{m}CP-nets}

Here, we devote our analysis to problems on Pareto optimal outcomes.
In particular, the problems analyzed are deciding whether an outcome is Pareto optimal, and deciding whether an \MCPNet has a Pareto optimal outcome.
We first focus on deciding Pareto optimality of outcomes in \MCPNets.
We  show that this problem is \CoNPc.
More formally, consider the following problem.

\medbreak

\begin{tabular}{rl}
  \textit{Problem:} & \IsParetoOptimal\\
  \textit{Instance:} & An \MCPNet $\MNet{M}$, and an outcome $\alpha\in\OutMNet{\MNet{M}}$.\\
  \textit{Question:} & Is $\alpha$ Pareto optimal in $\MNet{M}$?
\end{tabular}

\medbreak

The following result shows that, for a preference profile represented via an \NP-representation scheme, deciding the Pareto optimality of an outcome is feasible in \CoNP.

\begin{theorem}\label{theo_is_pareto_optimal_generic_membership}
Let $\Profile{P}$ be a preference profile defined over the same combinatorial domain and represented via an \NP-representation scheme, and let $\alpha$ be an outcome.
Then, deciding whether $\alpha$ is Pareto optimal in $\Profile{P}$ is feasible in \CoNP.
\end{theorem}
\begin{proof}
We show that disproving $\alpha$ being Pareto optimal is feasible in \NP.
If $\alpha$ is not Pareto optimal in $\Profile{P}$, then there is an outcome $\beta$ such that $\beta\PrefParetoProfile{\Profile{P}}\alpha$.
Therefore, we can guess such an outcome $\beta$ along with the witness that $\beta\PrefParetoProfile{\Profile{P}}\alpha$.
This guess requires only polynomial space and can be checked in polynomial time (see the proof of \cref{theo_pareto_query_generic_membership}).
\end{proof}

To prove the \CoNPh{}ness of \IsParetoOptimal, we use a reduction from \Unsat.
Consider the following construction.
Let $\phi(X)$ be a Boolean formula in $3$CNF defined over the set of Boolean variables $X=\{x_1,\dots,x_n\}$, and whose set of clauses is $C=\{c_1,\dots,c_m\}$.
From $\phi$, we build the $2$\CPNet $\MNetIsParOpt(\phi)=\tuple{\NetIPO1,\NetIPO2}$ in the following way.
The \CPNets $\NetIPO1$ and $\NetIPO2$ are built similarly, and we discuss first $\NetIPO1$.
See \cref{fig:is_pareto_optimal_hardness} for a schematic representation of the interconnections between the building blocks of $\NetIPO1$ and $\NetIPO2$.

\begin{figure}
  \centering%
  \includegraphics[width=0.7\textwidth]{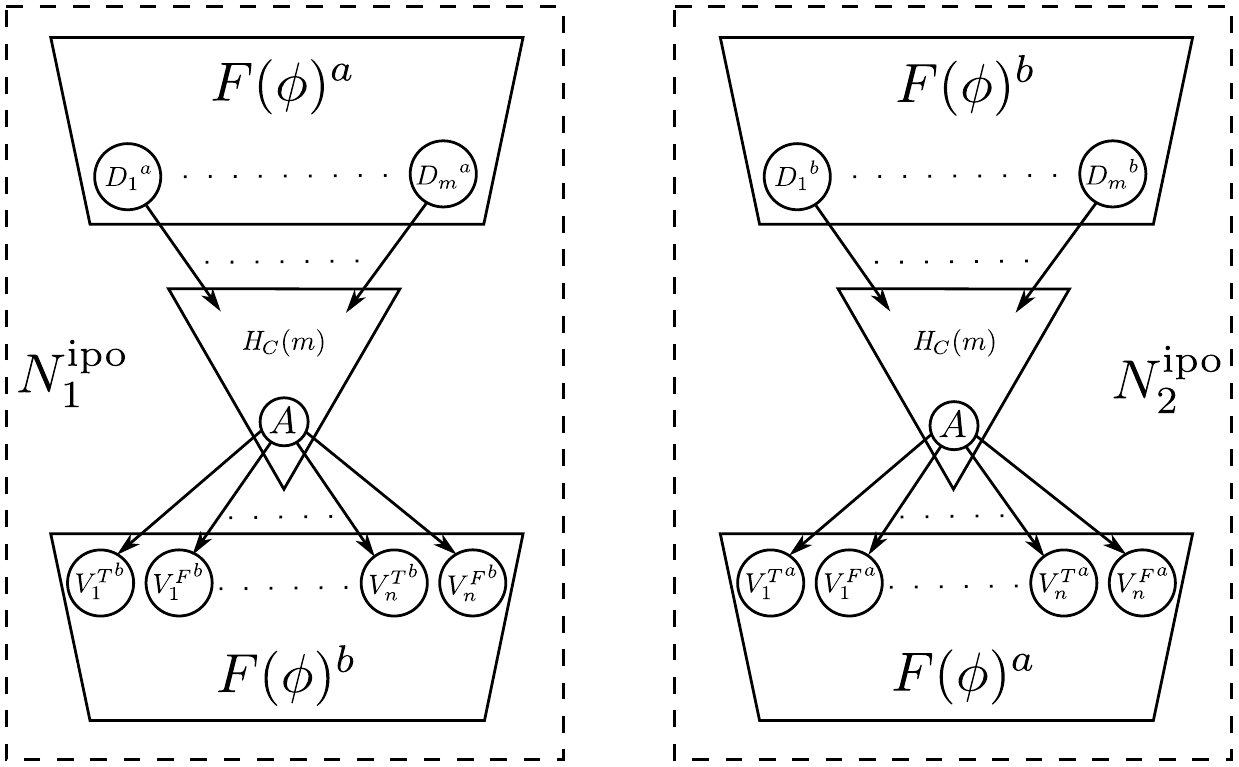}
  \caption{A schematic representation of the building blocks of the $2$\CPNet $\MNetIsParOpt(\phi)=\tuple{\NetIPO1,\NetIPO2}$.}\label{fig:is_pareto_optimal_hardness}
\end{figure}

In $\NetIPO1$, there are two complete copies of the net $\NetCNF(\phi)$ (defined in \cref{sec_formula_net}), with its features, edges, and CP tables.
To distinguish these two copies, we append two different superscript to them and obtain $\NetCNF(\phi)^a$ and $\NetCNF(\phi)^b$.
All features of these two nets have the corresponding superscript $a$ or $b$ to make them different and distinguish them from the others.
The clause features of $\NetCNF(\phi)^a$, i.e., ${D_1}^a,\dots,{D_m}^a$, are attached to a conjunctive interconnecting net $\NetInterAND(m)$ (defined in \cref{sec_interconnecting_net}).
The apex $A$ of $\NetInterAND(m)$ is attached to \emph{all} the variable features of $\NetCNF(\phi)^b$.
Since the variable features of $\NetCNF(\phi)^b$ now have one parent (more precisely the same parent, i.e., the apex of $\NetInterAND(m)$), the CP tables of these features are a bit different from those in $\NetCNF(\phi)$.
Variable features $F$ of $\NetCNF(\phi)^b$ have the CP tables

\smallskip
\begin{tabular}{|c||l|}
\hline
$\ol{a}$ & $\ol{f}\Pref f$\\
\hline
$a$ & $f\Pref \ol{f}$\\
\hline
\end{tabular}.

The \CPNet $\NetIPO2$ is similar to $\NetIPO1$, with the only difference that the conjunctive interconnecting \CPNet $\NetInterAND(m)$ attaches the clause features of $\NetCNF(\phi)^b$ to the variable features of $\NetCNF(\phi)^a$.
The CP tables in $\NetIPO2$ of the variable features of~$\NetCNF(\phi)^a$ are therefore modified accordingly.

Observe that $\MNetIsParOpt(\phi)$ is acyclic, binary, its indegree is three, and can be computed in polynomial time from $\phi$.
Moreover, the class of \MCPNets ${\{\MNetIsParOpt(\phi)\}}_{\phi}$ derived from formulas $\phi$ of the specified kind and according to the reduction shown above is polynomially connected.
The following result shows
that $\phi(X)$ is satisfiable if and only if a particular outcome of $\MNetIsParOpt(\phi)$ is not Pareto optimal.

\begin{lemma}\label{lemma_properties_of_IsParetoOptimal_reduction}
Let $\phi(X)$ be a $3$CNF Boolean formula, and let $\alpha$ be the outcome of $\MNetIsParOpt(\phi)$ assigning non-overlined values to all features.
Then, $\phi(X)$ is satisfiable if and only if $\alpha$ is not Pareto optimal in $\MNetIsParOpt(\phi)$.
\end{lemma}
\begin{proof}[Proof (sketch).]
The key point of the proof is that the features of the formula net downline of the interconnecting net can be flipped to their overlined values if and only if the formula $\phi$ is satisfiable.
Details of the proof are at page~\pageref{page_pointer:properties_of_IsParetoOptimal_reduction_detailed}.
\end{proof}

The above property implies that, in \MCPNets, deciding the Pareto optimality of an outcome is \CoNPh.

\begin{theorem}\label{theo_is_pareto_optimal_mcpnets_hardness}
Let $\MNet{M}$ be an \MCPNet, and let $\alpha\in\OutMNet{\MNet{M}}$ be an outcome.
Deciding whether $\alpha$ is Pareto optimal is \CoNPh.
Hardness holds even on polynomially connected classes of acyclic binary \MCPNets with indegree at most three and at most two agents.
\end{theorem}
\begin{proof}
We prove that \IsParetoOptimal is \CoNPh by showing a reduction from \Unsat.
Let $\phi(X)$ be a $3$CNF formula, and consider the $2$\CPNet $\MNetIsParOpt(\phi)=\tuple{\NetIPO1,\NetIPO2}$.
Consider the outcome $\alpha\in\OutMNet{\MNetIsParOpt(\phi)}$ in which the values of all features are non-overlined.
By \cref{lemma_properties_of_IsParetoOptimal_reduction}, $\phi(X)$ is unsatisfiable (and hence a ``yes''-instance of \Unsat) if and only if $\alpha$ is Pareto optimal in $\MNetIsParOpt(\phi)$.
\end{proof}

Notice here that the presence of at least two agents in an \MCPNet is an essential source of complexity for the problem \IsParetoOptimal.
In fact, if there were only one agent, then deciding whether an outcome $\alpha$ is Pareto optimal would be tantamount to checking whether $\alpha$ is optimal for that only agent, and we have seen already that this task can be carried out in \LogSpace (see \cref{theo_individual_optimality_membership}).

By combining the two above results, we  immediately conclude that deciding the Pareto optimality of an outcome over polynomially connected classes of acyclic \MCPNets is \CoNPc.

\begin{corollary}
Let $\mathcal{C}$ be a polynomially connected class of acyclic \MCPNets.
Let $\MNet{M}\in\mathcal{C}$ be an \MCPNet, and let $\alpha\in\OutMNet{\MNet{M}}$ be an outcome.
Deciding whether $\alpha$ is Pareto optimal is \CoNPc.
\end{corollary}

We now focus on the problem of deciding the existence of Pareto optimal outcomes in \MCPNets.
We  show that every \MCPNets has a Pareto optimal outcome, which implies that the problem is trivial (i.e., feasible in constant time).
More formally, consider the following problem.

\medbreak

\begin{tabular}{rl}
  \textit{Problem:} & \ExistsParetoOptimal\\
  \textit{Instance:} & An \MCPNet $\MNet{M}$.\\
  \textit{Question:} & Does $\MNet{M}$ have a Pareto optimal outcome?
\end{tabular}

\medbreak

The following lemma shows that an acyclic \CPNet has always a Pareto optimal outcome.
Its proof is different from the
proof of a similar result that appeared in \cite{Rossi2004}.

\begin{lemma}[different proof from~\cite{Rossi2004}]\label{lemma_always_exists_pareto_optimal_mcpnets}
Let $\MNet{M}$ be an acyclic \MCPNet.
Then, $\MNet{M}$ has (always) a Pareto optimal outcome.
\end{lemma}
\begin{proof}
Let $\MNet{M}=\tuple{\Net{N}_1,\dots,\Net{N}_m}$ be an $m$\CPNet, and assume by contradiction that $\MNet{M}$ has no Pareto optimal outcome.
Consider any net $\Net{N}_i$ of $\MNet{M}$.
By the fact that $\Net{N}_i$ is acyclic, it follows that there is a unique optimum outcome $\OptOutNet{\Net{N}_i}$ in $\Net{N}_i$.
Since we assume that in $\MNet{M}$ there are no Pareto optimal outcomes, there must be an outcome $\beta\neq\OptOutNet{\Net{N}_i}$ such that $\beta\PrefParetoMNet{\MNet{M}}\OptOutNet{\Net{N}_i}$, i.e., $\beta$ is preferred to $\OptOutNet{\Net{N}_i}$ by \emph{all} agents of $\MNet{M}$.
However, there is no outcome $\beta$ such that $\beta\PrefNet{\Net{N}_i}\OptOutNet{\Net{N}_i}$, because $\OptOutNet{\Net{N}_i}$ is the optimum outcome in $\Net{N}_i$: a contradiction.
Therefore, there must be a Pareto optimal outcome in $\MNet{M}$.
\end{proof}

\subsection{Complexity of Pareto optimums on \texorpdfstring{$m$}{m}CP-nets}

We now focus on Pareto optimum outcomes.
In particular, the problems analyzed are deciding whether an outcome is Pareto optimum, and deciding whether an \MCPNet has a Pareto optimum outcome.
To study the complexity of these problems, we first prove an intermediate property stating that an \MCPNet has a Pareto optimum outcome if and only if all its individual \CPNets have the very same optimum outcome.

\begin{lemma}\label{lemma_pareto_optimum_iff_same_individual_optimal}
Let $\MNet{M}$ be an acyclic $m$\CPNet.
Then, $\MNet{M}$ has a Pareto optimum outcome if and only if all the individual \CPNets of $\MNet{M}$ have the very same optimum outcome (that, in which case, is also the Pareto optimum outcome of $\MNet{M}$).
\end{lemma}
\begin{proof}
We show first that if $\MNet{M}=\tuple{\Net{N}_1,\dots,\Net{N}_m}$ has a Pareto optimum outcome $\alpha$, then all the individual \CPNets have the very same optimum outcome, which is $\alpha$.
By definition, a Pareto optimum outcome is unique and Pareto dominates all other outcomes.
This means that, for any outcome $\beta \neq \alpha$, and for all $i$, $\alpha \PrefNet{\Net{N}_i} \beta$.
We know that in each individual \CPNet the only outcome dominating all the others is the individual optimum.
Therefore, $\alpha$ equals all the individual optimum outcomes.
The other direction of the proof is easy.
\end{proof}

Based on the property above, we  characterize the complexity of problems on Pareto optimum outcomes.
We first focus on deciding whether outcomes are Pareto optimum in \MCPNets.
We show that this problem is solvable in \LogSpace.
More formally, consider the following problem.

\medbreak

\begin{tabular}{rl}
  \textit{Problem:} & \IsParetoOptimum\\
  \textit{Instance:} & An \MCPNet $\MNet{M}$ and an outcome $\alpha$.\\
  \textit{Question:} & Is $\alpha$ Pareto optimum in $\MNet{M}$?
\end{tabular}

\medbreak

The following theorem shows that, in acyclic \MCPNets, deciding that an outcome is Pareto optimum is feasible in \LogSpace.

\begin{theorem}\label{theo_is_pareto_optimum_mcpnets_membership}
Let $\MNet{M}$ be an acyclic $m$\CPNet, and let $\alpha\in\OutMNet{\MNet{M}}$ be an outcome.
Deciding whether $\alpha$ is Pareto optimum in $\MNet{M}$ is feasible in \LogSpace.
\end{theorem}
\begin{proof}
From \cref{lemma_pareto_optimum_iff_same_individual_optimal}, we know that $\alpha$ is Pareto optimum in $\MNet{M}$ if and only if $\alpha$ equals all the individual optimum outcomes of the individual \CPNets of $\MNet{M}$.
Therefore, in order to check whether $\alpha$ is actually Pareto optimum it suffices to check, for each individual \CPNet in turn, whether $\alpha$ is the individual optimal outcome for that agent.
If $\alpha$ is different from even just one of the individual optimum outcomes, then $\alpha$ is not Pareto optimum.
Remember that checking the individual optimality of $\alpha$ is feasible in \LogSpace (see \cref{theo_individual_optimality_membership}), and hence, by reusing of work space, we can check in \LogSpace whether $\alpha$ equals all the individual optimum outcomes.
\end{proof}

To conclude, we study the complexity of deciding whether an \MCPNet has a Pareto optimum outcome.
We show that this problem is feasible in \PTIME.
More formally, consider the following problem.

\medbreak

\begin{tabular}{rl}
  \textit{Problem:} & \ExistsParetoOptimum\\
  \textit{Instance:} & An \MCPNet $\MNet{M}$.\\
  \textit{Question:} & Does $\MNet{M}$ have a Pareto optimum outcome?
\end{tabular}

\medbreak

The following theorem states that deciding whether an acyclic \MCPNet has a Pareto optimum outcome is in~\PTIME.

\begin{theorem}\label{theo_exists_pareto_optimum_mcpnets_membership}
Let $\MNet{M}$ be an acyclic $m$\CPNet.
Deciding whether $\MNet{M}$ has a Pareto optimum outcome is feasible in \PTIME.
\end{theorem}
\begin{proof}
From \cref{lemma_pareto_optimum_iff_same_individual_optimal}, we know that $\MNet{M}$ has a Pareto optimum outcome if and only if all the individual \CPNets have the very same individual optimum outcome.
Hence, in order to decide whether $\MNet{M}$ has a Pareto optimum outcome, it suffices to compute the individual optimum outcome of the first agent.
We can do this in polynomial time~\cite{Boutilier2004}.
After this, we compare the just computed individual optimum outcome with all the other individual optimum outcomes.
This can be carried out, by reusing of working space, in logarithmic space (see \cref{theo_is_pareto_optimum_mcpnets_membership}), and hence in polynomial time (by the inclusion $\LogSpace\subseteq\PTIME$).
If all the individual optimum outcomes are equal, then we answer yes, otherwise no.
Observe that the overall procedure is feasible in \PTIME.
\end{proof}

\section{Complexity of majority voting on \texorpdfstring{$m$}{m}CP-nets}\label{sec_complexity_majority_voting}\label{sec6}

In this section, we characterize the complexity of majority voting tasks on \MCPNets.
We start with a preliminary section by showing that there are \MCPNets without majority optimal and optimum outcomes, which implies that deciding the existence of majority optimal and optimum outcomes is not a trivial problem.
In the preliminary section, we also define some \CPNets that will be used in the reductions exhibited in this section.
More specifically, we will introduce directs nets, which are \CPNets having a designated outcome as optimal.
Moreover, we  introduce summarized formula nets, which, similarly to formula nets, are \CPNets encoding Boolean formulas, but they associate the satisfiability of formulas with outcomes having specific values on just two features.
Then, we  analyze the complexity of deciding majority dominance in \MCPNets, which is shown \NPc.
Subsequently, we devote our analysis to the problems related to majority optimal outcomes, namely, deciding whether an outcome is majority optimal, and deciding whether an \MCPNet has a majority optimal outcome.
We prove the former \CoNPc and the latter~\SigmaPc{2}.
To conclude, we study the complexity of problems on majority optimum outcomes, namely, deciding whether an outcome is majority optimum, and deciding whether an \MCPNet has a majority optimum outcome.
We  prove the former \PiPc{2} and the latter is shown \PiPh{2} and belonging to \DP{2}.

Recall that, given an \MCPNet~$\MNet{M}$ and two different outcomes $\alpha,\beta\in\OutMNet{\MNet{M}}$, it holds that $\beta\PrefMajorityMNet{\MNet{M}}\alpha$, if the majority of agents prefer $\beta$ to $\alpha$, i.e., $|\SetAgentsPrefMNet{\MNet{M}}(\beta,\alpha)| > |\SetAgentsFerpMNet{\MNet{M}}(\beta,\alpha)| + |\SetAgentsIncompMNet{\MNet{M}}(\beta,\alpha)|$.
What we call majority optimal and majority optimum outcomes, in some works (see, e.g.,~\cite{Li2011,Felsenthal2014}) are named weak and (strong) Condorcet winners, respectively.
However, in the literature (see, e.g.,~\cite{BrandtEtAlMAS2013,BaumeisterRoth:Voting,SurveyCompSocChoice}), the nomenclature of weak/strong Condorcet winner has also been used with a slightly different meaning.
To avoid any confusion, in this paper, we prefer to stick to the concepts of majority optimal and majority optimum outcomes, which we introduced in \cref{sec_prelim}.

\subsection{Preliminaries} \label{sec_direct_net}\label{sec_summ_formula_net}

We first show that there are \MCPNets that do not have any majority optimal outcome, and hence neither a majority optimum outcome.
Thus, deciding whether an \MCPNet has majority optimal or optimum outcomes is a non-trivial problem.

\begin{theorem}\label{theo_no_weak_strong_condorcet_winners}
There are acyclic binary singly-connected \MCPNets not having majority optimal and majority optimum outcomes.
\end{theorem}
\begin{proof}
Consider the acyclic binary singly connected $4$\CPNet $\MNetMaj=\tuple{\Net{N}_1,\Net{N}_2,\Net{N}_3,\Net{N}_4}$ defined in \cref{fig:no_majority_optimal}.

\begin{figure}
  \centering%
  \includegraphics[width=0.9\textwidth]{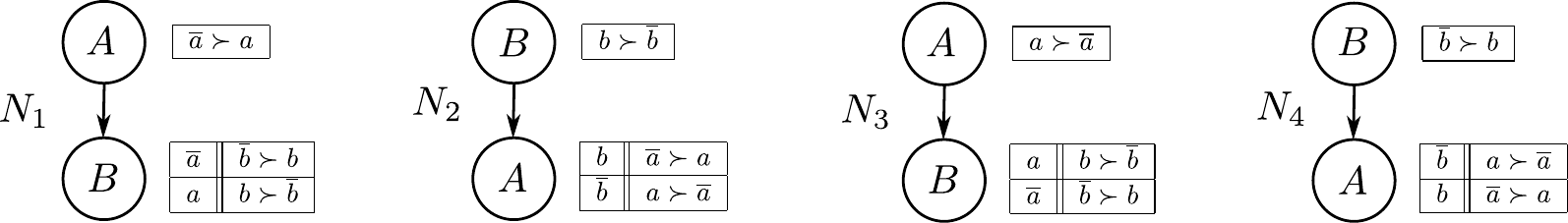}
  \caption{The $4$\CPNet $\MNetMaj$ of \cref{theo_no_weak_strong_condorcet_winners}.}
  \label{fig:no_majority_optimal}
\end{figure}

The preferences encoded in the four nets are: $\ol{a}\ol{b}\PrefNet{\Net{N}_1}\ol{a}b\PrefNet{\Net{N}_1}ab\PrefNet{\Net{N}_1}a\ol{b}$;
$\ol{a}b\PrefNet{\Net{N}_2}ab\PrefNet{\Net{N}_2}a\ol{b}\PrefNet{\Net{N}_2}\ol{a}\ol{b}$;
$ab\PrefNet{\Net{N}_3}a\ol{b}\PrefNet{\Net{N}_3}\ol{a}\ol{b}\PrefNet{\Net{N}_3}\ol{a}b$; and
$a\ol{b}\PrefNet{\Net{N}_4}\ol{a}\ol{b}\PrefNet{\Net{N}_4}\ol{a}b\PrefNet{\Net{N}_4}ab$.
Observe that:
$ab$ is not majority optimal, because $\ol{a}b\PrefMajorityMNet{\MNetMaj}ab$;
$\ol{a}b$ is not majority optimal, because $\ol{a}\ol{b}\PrefMajorityMNet{\MNetMaj}\ol{a}b$;
$a\ol{b}$ is not majority optimal, because $ab\PrefMajorityMNet{\MNetMaj}a\ol{b}$;
and $\ol{a}\ol{b}$ is not majority optimal, because $a\ol{b}\PrefMajorityMNet{\MNetMaj}\ol{a}\ol{b}$.
This implies that $\MNetMaj$ does not have any majority optimal outcome, and hence also either a majority optimum outcome.
\end{proof}

We now define two \CPNets that will be used in the reductions of this section on majority voting.
First, we define \emph{direct nets},
which intuitively are \CPNets that have a specific desired optimum outcome (recall that an acyclic \CPNet has a unique optimal outcome, which is also optimum).
Let $\SetFeat{S}$ be a set of binary features defined over the usual values, and let $\alpha\in\Dom(\SetFeat{S})$ be an outcome over $\SetFeat{S}$.
The \emph{direct net} $\NetDirect(\alpha)=\tuple{\GraphNet{\NetDirect(\alpha)},\DomNet{\NetDirect(\alpha)},\CPTNet{\NetDirect(\alpha)}}$ is the \CPNet such that $\FeatNet{\NetDirect(\alpha)}=\SetFeat{S}$, $\LinkNet{\NetDirect(\alpha)}=\emptyset$, the domain of each feature $F$ of $\NetDirect(\alpha)$ is the same as the domain of $F$ in $\SetFeat{S}$, and the CP tables of $\NetDirect(\alpha)$ are such that, given a feature $F$, if $\alpha[F]=f$, then the CP table for $F$ is
\begin{tabular}{|l|}
	\hline
	$f\Pref \ol{f}$\\
	\hline
\end{tabular},
otherwise (i.e., $\alpha[F]=\ol{f}$) the CP table for $F$ is
\begin{tabular}{|l|}
	\hline
	$\ol{f}\Pref f$\\
	\hline
\end{tabular}.
See \cref{fig:direct_net} for an example.
Clearly, given any outcome $\beta\in\OutNet{\NetDirect(\alpha)}$ such that $\beta\neq\alpha$, it holds that $\alpha\PrefNet{\NetDirect(\alpha)}\beta$.
Moreover, let $\beta$ and $\gamma$ be two different outcomes of $\NetDirect(\alpha)$ such that, for all features $F$ for which $\gamma[F]\neq\beta[F]$, it holds that $\gamma[F]=\alpha[F]$ (and hence $\beta[F] \neq \alpha[F]$).
Then, $\gamma\PrefNet{\NetDirect(\alpha)}\beta$.
Note that $\NetDirect(\alpha)$ is binary, acyclic, singly connected, its indegree is zero, and the net can be built in polynomial time from $\alpha$.

\begin{figure}
	\centering%
	\includegraphics[width=0.45\textwidth]{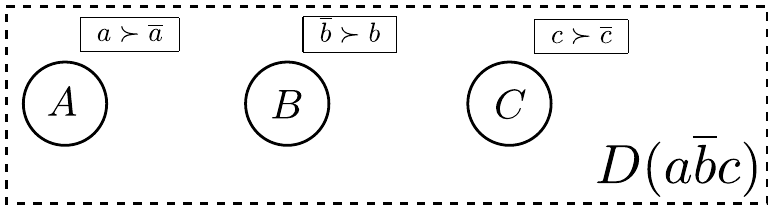}
	\caption{The direct net $\NetDirect(\alpha)$, with $\alpha=a\ol{b}c$.}
	\label{fig:direct_net}
\end{figure}

We finally introduce \emph{summarized formula nets}, which are similar to formula nets (of \cref{sec_formula_net}), with the advantage that these new nets put in relationship the satisfiability of Boolean formulas with the flip of only two features, instead of with the flip of all variable and clause features.
This advantage comes at cost of loosing the single connectedness property of the nets, which, instead, is satisfied in non-summarized formula nets.
Formally, let
$\phi(X)$ be a Boolean formula in $3$CNF defined over the set of Boolean variables $X=\{x_1,\dots,x_n\}$, and whose set of clauses is $C=\{c_1,\dots,c_m\}$.
From $\phi(X)$, we build the \CPNet $\NetCNFSumm(\phi)$ in the following way (see \cref{fig:summ_formula_net}).

\begin{figure}
	\centering%
	\includegraphics[width=0.3\textwidth]{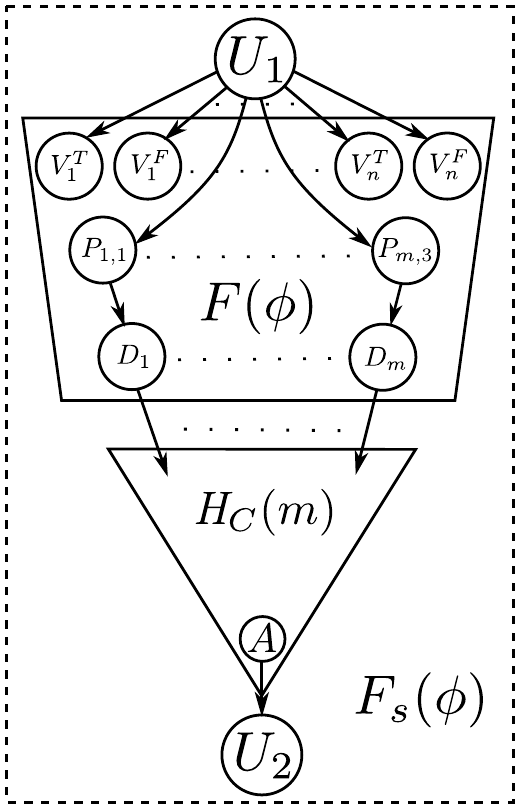}
	\caption{A schematic representation of CP-net $\NetCNFSumm(\phi)$.}
	\label{fig:summ_formula_net}
\end{figure}

The \CPNet $\NetCNFSumm(\phi)$ embeds a formula net $\NetCNF(\phi)$ (defined in \cref{sec_formula_net}) with its features and links.
Moreover, in $\NetCNFSumm(\phi)$, there is an interconnecting net $\NetInterAND(m)$ (defined in \cref{sec_interconnecting_net}), which is attached to all clause features of $\NetCNF(\phi)$.
We denote by $\SetFeat{A}$ the set of features belonging to the net $\NetInterAND(m)$ embedded in $\NetCNFSumm(\phi)$.
To conclude with the features of~$\NetCNFSumm(\phi)$, there are two more features: $U_1$ and $U_2$, where
$U_1$ has no parents and is linked to all variable and literal features of~$\NetCNF(\phi)$, while
$U_2$ is not the parent of any feature and its unique parent is the apex of the interconnecting net.

The CP tables of $\NetCNFSumm(\phi)$ are as follows:

\begin{itemize}
\item feature $U_1$ has the CP table
	\begin{tabular}{|l|}
		\hline
		$\ol{u_1}\Pref u_1$\\
		\hline
	\end{tabular};
	
	\item for each variable $x_i\in X$, features $V_i^T$ and $V_i^F$ have the CP tables
	
	\smallskip
	\begin{tabular}{|c||l|}
		\hline
		$u_1$ & $\ol{v_i^T}\Pref v_i^T$\\
		\hline
		$\ol{u_1}$ & $v_i^T\Pref \ol{v_i^T}$\\
		\hline
	\end{tabular} and
	\begin{tabular}{|c||l|}
		\hline
		$u_1$ & $\ol{v_i^F}\Pref v_i^F$\\
		\hline
		$\ol{u_1}$ & $v_i^F\Pref \ol{v_i^F}$\\
		\hline
	\end{tabular}, respectively;
	
	\item for each literal $\ell_{j,k}$, if $\ell_{j,k}=x_i$, then feature $P_{j,k}$ has the CP table
	
	\smallskip \begin{tabular}{|c||l|}
		\hline
		$u_1\land\ol{v_i^T}\land v_i^F$ & $\ol{p_{j,k}}\Pref p_{j,k}$\\
		\hline
		else & $p_{j,k}\Pref \ol{p_{j,k}}$\\
		\hline
	\end{tabular};
	
	otherwise, if $\ell_{j,k}=\lnot x_i$, then $P_{j,k}$ has the CP table

\smallskip 	\begin{tabular}{|c||l|}
		\hline
		$u_1\land v_i^T\land \ol{v_i^F}$ & $\ol{p_{j,k}}\Pref p_{j,k}$\\
		\hline
		else & $p_{j,k}\Pref \ol{p_{j,k}}$\\
		\hline
	\end{tabular};
	
	\item clause features have the same CP table as in $\NetCNF(\phi)$;
	
	\item features of the conjunctive interconnecting net $\NetInterAND(m)$ have the usual CP tables;
	
	\item feature $U_2$, if feature $A$ is the apex of the interconnecting net, has the CP table

\smallskip 	\begin{tabular}{|c||l|}
		\hline
		$\ol{a}$ & $\ol{u_2}\Pref u_2$\\
		\hline
		$a$ & $u_2\Pref \ol{u_2}$\\
		\hline
	\end{tabular}.
\end{itemize}

Note that $\NetCNFSumm(\phi)$
is binary, acyclic, its indegree is three, and the net can be built in polynomial time in the size of~$\phi$.
Moreover, the class of \MCPNets ${\{\MNetIsParOpt(\phi)\}}_{\phi}$ derived from formulas $\phi$ of the specified kind and according to the reduction shown above is polynomially connected.

Now we give an equivalent of \cref{lemma_Pref_Net_CNF_new} for $\NetCNFSumm(\phi)$.
In particular, the following \lcnamecref{lemma_Pref_Net_CNF_Summ_new} and \lcnamecref{corol:Pref_Net_CNF_Summ} show that $\phi$ is satisfiable if and only if a particular outcome dominates others in $\NetCNFSumm(\phi)$.

\begin{lemma}\label{lemma_Pref_Net_CNF_Summ_new}
	Let $\phi(X)$ be a Boolean formula in $3$CNF defined over a set $X$ of Boolean variables, and let $\sigma_X$ be an assignment on $X$.
	Let $\alpha_{\sigma_X}$ be the outcome of $\NetCNFSumm(\phi)$ encoding $\sigma_X$ on the feature set $\SetFeat{V}$, and assigning non-overlined values to all other features.
	Let $\ol{\beta}$ be an outcome of $\NetCNFSumm(\phi)$ such that $\ol{\beta}[U_1 U_2] = \ol{u_1}\ol{u_2}$, assigning any value to the features of $\SetFeat{V}$, and assigning non-overlined values to all other features.
	Then:
	\begin{itemize}
		\item[(1)] There is an extension of $\sigma_X$ to $X$ satisfying $\phi(X)$ if and only if $\ol{\beta}\PrefNet{\NetCNFSumm(\phi)}\alpha_{\sigma_X}$;
		\item[(2)] There is no extension of $\sigma_X$ to $X$ satisfying $\phi(X)$ if and only if $\ol{\beta}\IncompNet{\NetCNFSumm(\phi)}\alpha_{\sigma_X}$.
	\end{itemize}
\end{lemma}

\begin{proof}[Proof (sketch).]
	The intuition at the base of the proof of this property is that, by linking $U_1$ to variable and literal features, these cannot be flipped to their overlined values once $U_1$ is $\ol{u_1}$.
	Therefore, distinct literal features, attached to the same features $V_i^T$ and $V_i^F$, cannot be flipped to their overlined values according to contrasting values of $V_i^T$ and $V_i^F$.
	Details of the proof are at page~\pageref{page_pointer:Pref_Net_CNF_Summ_new_detailed}.
\end{proof}

\begin{corollary}\label{corol:Pref_Net_CNF_Summ}
	Let $\phi(X)$ be a Boolean formula in $3$CNF, and let $\alpha$ and $\ol{\beta}$ be two outcomes of $\OutNet{\NetCNFSumm(\phi)}$ such that in $\alpha$ the values of all features are non-overlined, and in $\ol{\beta}$ the values of only $U_1$ and of $U_2$ are overlined.
	Then:
	\begin{itemize}
		\item $\phi(X)$ is satisfiable if and only if $\ol{\beta}\PrefNet{\NetCNFSumm(\phi)}\alpha$, and
		\item $\phi(X)$ is unsatisfiable if and only if $\ol{\beta}\IncompNet{\NetCNFSumm(\phi)}\alpha$.
	\end{itemize}
\end{corollary}
\begin{proof}
	Observe that, when the empty assignment $\sigma_X$ is considered, outcome $\alpha_{\sigma_X}$ of the statement of \cref{lemma_Pref_Net_CNF_Summ_new} coincides with outcome $\alpha$ of the statement of this corollary.
	Moreover, any assignment over $X$ is an extension of the empty one, and hence \cref{lemma_Pref_Net_CNF_Summ_new} applies.
\end{proof}

\subsection{Complexity of majority dominance on \texorpdfstring{$m$}{m}CP-nets}

First, we analyze the problem of deciding majority dominance on \CPNets, which is shown \NPc.
More formally, consider the following problem.

\medbreak

\begin{tabular}{rl}
  \textit{Problem:} & \MajorityQuery\\
  \textit{Instance:} & An \MCPNet $\MNet{M}$, and two outcomes $\alpha,\beta\in\OutMNet{\MNet{M}}$.\\
  \textit{Question:} & Is $\beta \PrefMajorityMNet{\MNet{M}} \alpha$?
\end{tabular}

\medbreak

The following result shows that, for preference profiles represented via an \NP-representation scheme, deciding majority dominance is feasible in \NP.

\begin{theorem}\label{theo_majority_query_generic_membership}
Let $\Profile{P}$ be a preference profile defined over the same combinatorial domain and represented via an \NP-representation scheme, and let $\alpha$ and $\beta$ be two outcomes.
Then, deciding whether $\beta \PrefMajorityProfile{\Profile{P}} \alpha$ is feasible in \NP.
\end{theorem}
\begin{proof}
Let $\Profile{P}=\tuple{\Ranking{R}_1,\dots,\Ranking{R}_m}$.
Observe first that, since $|\SetAgentsPrefProfile{\Profile{P}}(\beta,\alpha)|+|\SetAgentsFerpProfile{\Profile{P}}(\beta,\alpha)|+|\SetAgentsIncompProfile{\Profile{P}}(\beta,\alpha)|=m$, it holds that  $|\SetAgentsPrefProfile{\Profile{P}}(\beta,\alpha)|>|\SetAgentsFerpProfile{\Profile{P}}(\beta,\alpha)|+|\SetAgentsIncompProfile{\Profile{P}}(\beta,\alpha)|$ if and only if $|\SetAgentsPrefProfile{\Profile{P}}(\beta,\alpha)|>\left\lfloor\frac{m}{2}\right\rfloor$.
If $\beta \PrefMajorityProfile{\Profile{P}} \alpha$, then, for more than half of the agents $i$, $\beta \PrefRanking{\Ranking{R}_i}\alpha$.
For such agents,
since the preferences are represented via an \NP-representation scheme, there is a polynomial witness that they prefer $\beta$ to $\alpha$.
Therefore, to show that $\beta \PrefMajorityProfile{\Profile{P}} \alpha$, it suffices to guess a set $S$ of players preferring $\beta$ to $\alpha$, along with the polynomial witness of their preference, and then check that $|S|>\left\lfloor\frac{m}{2}\right\rfloor$ and that the witnesses are valid.
The overall guess requires only polynomial space, and it can be checked in polynomial time.
\end{proof}

Observe that, on $1$\CPNets, majority dominance is equivalent to dominance on (simple) \CPNets.
Therefore, the following result, which follows directly from \cref{theo_dominance_cpnets_hardness}, states that on \MCPNets deciding majority dominance is \NPh.

\begin{theorem}\label{theo_majority_query_mcpnets_hardness}
Let $\MNet{M}$ be an \MCPNet, and let $\alpha,\beta\in\OutMNet{\MNet{M}}$ be two outcomes.
Then, deciding whether $\beta\PrefMajorityMNet{\MNet{M}}\alpha$ is \NPh.
Hardness holds even on classes of singly connected acyclic binary \MCPNets with indegree at most three and at most one agent.
\end{theorem}

By combining the two above results, we immediately obtain that deciding majority dominance over polynomially connected classes of acyclic \MCPNets is \NPc.

\begin{corollary}
Let $\mathcal{C}$ be a polynomially connected class of acyclic \MCPNets.
Let $\MNet{M}\in\mathcal{C}$ be an \MCPNet, and let $\alpha,\beta\in\OutMNet{\MNet{M}}$ be two outcomes.
Then, deciding whether $\beta\PrefMajorityMNet{\MNet{M}}\alpha$ is \NPc.
\end{corollary}

\subsection{Complexity of majority optimality on \texorpdfstring{$m$}{m}CP-nets}\label{sec_complexity_majority_optimal}

Here, we analyze the problems on majority optimal outcomes.
In particular, the problems considered are deciding whether an outcome is majority optimal, and deciding whether an \MCPNet has a majority optimal outcome.
We first focus on deciding majority optimality of outcomes in \MCPNets.
We  show that this problem is \CoNPc.
More formally, consider the following problem.

\medbreak

\begin{tabular}{rl}
  \textit{Problem:} & \IsMajorityOptimal\\
  \textit{Instance:} & An \MCPNet $\MNet{M}$, and an outcome $\alpha\in\OutMNet{\MNet{M}}$.\\
  \textit{Question:} & Is $\alpha$ majority optimal in $\MNet{M}$?
\end{tabular}

\medbreak
The following theorem shows that, on preference profiles represented via an \NP-representation scheme, deciding whether an outcome is majority optimal is feasible in \CoNP.

\begin{theorem}\label{theo_is_weak_condorcet_generic_membership}
Let $\Profile{P}$ be a preference profile defined over the same combinatorial domain and represented via an \NP-representation scheme, and let $\alpha$ be an outcome.
Then, deciding whether $\alpha$ is majority optimal in~$\Profile{P}$ is feasible in \CoNP.
\end{theorem}
\begin{proof}
We show that disproving $\alpha$ being majority optimal is feasible in \NP.
If $\alpha$ is not majority optimal in $\Profile{P}$, then there is an outcome $\beta$ such that $\beta\PrefMajorityProfile{\Profile{P}}\alpha$.
Therefore, we can guess such an outcome $\beta$ along with the witness that $\beta\PrefMajorityProfile{\Profile{P}}\alpha$ (i.e., the set of agents preferring $\beta$ to $\alpha$).
This guess requires only polynomial space, and can be checked in polynomial time (see the proof of \cref{theo_majority_query_generic_membership}).
\end{proof}

Observe that, on $2$\CPNets, majority dominance and Pareto dominance are equivalent.
Therefore, the following result, which follows directly from \cref{theo_is_pareto_optimal_mcpnets_hardness}, shows that on \MCPNets deciding majority optimality is \CoNPh.

\begin{theorem}\label{theo_is_weak_condorcet_generic_mcpnets_membership}
Let $\MNet{M}$ be an \MCPNet, and let $\alpha\in\OutMNet{\MNet{M}}$ be an outcome.
Then, deciding whether $\alpha$ is majority optimal is \CoNPh.
Hardness holds even on polynomially connected classes of acyclic binary \MCPNets with indegree at most three and at most two agents.
\end{theorem}

By combining the two previous results, it follows immediately that deciding the majority optimality of an outcome over polynomially connected classes of acyclic \MCPNets is \CoNPc.

\begin{corollary}
Let $\mathcal{C}$ be a polynomially connected class of acyclic \MCPNets.
Let $\MNet{M}\in\mathcal{C}$ be an \MCPNet, and let $\alpha\in\OutMNet{\MNet{M}}$ be an outcome.
Then, deciding whether $\alpha$ is majority optimal is \CoNPc.
\end{corollary}

We now focus on the problem of deciding the existence of majority optimal outcomes in \MCPNets.
We  show that this problem is \SigmaPc{2}.
More formally, consider the following problem.

\medbreak

\begin{tabular}{rl}
  \textit{Problem:} & \ExistsMajorityOptimal\\
  \textit{Instance:} & An \MCPNet $\MNet{M}$.\\
  \textit{Question:} & Does $\MNet{M}$ have a majority optimal outcome?
\end{tabular}

\medbreak

The following \lcnamecref{theo_exists_weak_condorcet_generic_membership} shows that deciding whether a preference profile represented via an \NP-representation scheme has a majority optimal outcome is feasible in \SigmaP{2}.

\begin{theorem}\label{theo_exists_weak_condorcet_generic_membership}
Let $\Profile{P}$ be a preference profile defined over the same combinatorial domain and represented via an \NP-representation scheme.
Then, deciding whether $\Profile{P}$ has a majority optimal outcome is feasible in \SigmaP2.
\end{theorem}
\begin{proof}
To show that $\Profile{P}$ has a majority optimal outcome, it suffices to guess an outcome $\alpha$ and then check that $\alpha$ is actually majority optimal.
Observe that guessing $\alpha$ requires an \NP machine, and the final check is feasible in \CoNP (see \cref{theo_is_weak_condorcet_generic_membership}), which can be carried out by an oracle.
Therefore, the overall procedure is feasible in \SigmaP2.
\end{proof}

To prove the \SigmaPh{2}{}ness of \ExistsMajorityOptimal, we use a reduction from QBF$^{\mathit{CNF}}_{2,\forall,\lnot}$.
Consider the following construction.
Let $\Phi = (\exists X)(\forall Y)\lnot\phi(X,Y)$ be a quantified formula where $\phi(X,Y)$ is a $3$CNF Boolean formula defined over two disjoint sets $X=\{x_1,\dots,x_{n_X}\}$ and $Y=\{y_1,\dots,y_{n_Y}\}$ of Boolean variables, and whose set of clauses is $C=\{c_1,\dots,c_m\}$.
From $\Phi$, we define the $6$\CPNet $\MNetExistsWeakCond(\Phi)=\tuple{\NetEWC1,\dots,\NetEWC6}$ as follows.

\smallskip

\noindent The features of $\MNetExistsWeakCond(\Phi)$ are:
\begin{itemize}
\item all the features of a net $\NetCNFSumm(\phi)$ (defined in \cref{sec_summ_formula_net}) in which, in this case, we distinguish two variable feature sets $\SetFeat{V}=\{V_i^T,V_i^F\mid x_i\in X\}$ and $\SetFeat{W}=\{W_i^T,W_i^F\mid y_i\in Y\}$ (recall that $\SetFeat{P}$ and $\SetFeat{D}$ are the sets of literal and clause features, respectively, and $\SetFeat{A}$ is the set of features of the conjunctive interconnecting net embedded in $\NetCNFSumm(\phi)$); for further reference, we call $A$ the apex of the interconnecting net;

\item all the features of set $\SetFeat{V'}=\{V_i'\mid x_i\in X\}$;

\item all the features of the set $\SetFeat{B}$ which are the features $B_i$ of a disjunctive interconnecting net $\NetInterOR(|\SetFeat{V'}|+|\SetFeat{W}|+|\SetFeat{P}|+|\SetFeat{D}|+|\SetFeat{A}|)$ and its apex is feature $B$ (observe that features $B_i$ are distinct from features $A_i$ of the conjunctive interconnecting net $\NetInterAND(m)$ embedded in $\NetCNFSumm(\phi)$).
\end{itemize}

To summarize, all the features of $\MNetExistsWeakCond(\Phi)$ are:
$\SetFeat{V} \cup \SetFeat{V'} \cup \SetFeat{W} \cup \SetFeat{P} \cup \SetFeat{D} \cup \SetFeat{A} \cup \SetFeat{B} \cup \{U_1,U_2\}$.

\begin{figure}
  \centering%
  \includegraphics[width=0.9\textwidth]{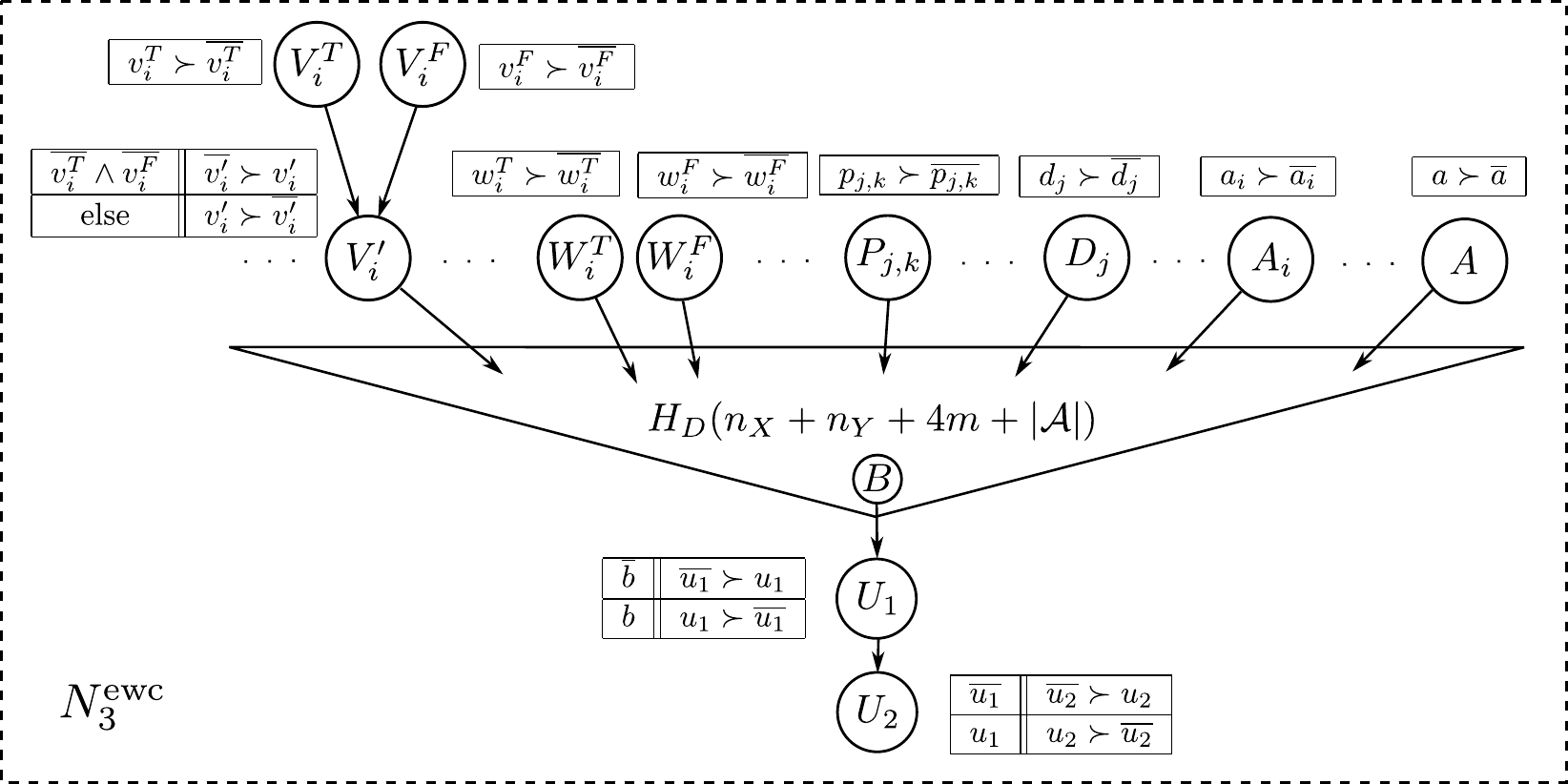}
  \caption{A schematic illustration of net $\NetEWC3$ of $\MNetExistsWeakCond(\Phi)$. The direct net embedded in $\NetEWC3$ is not reported in the figure.}
  \label{fig:ExistsWeakCond_Hardness}
\end{figure}

The \CPNets of $\MNetExistsWeakCond(\Phi)$ are:
\begin{itemize}
\item $\NetEWC1$ is composed by a net $\NetCNFSumm(\phi)$ with its features, links, CP tables, and a direct net $\NetDirect(\gamma)$ (see \cref{sec_direct_net}), where $\gamma$ is defined over the set of features $\SetFeat{V'}\cup\SetFeat{B}$ and assigns non-overlined values to all of them.

\item $\NetEWC2$ is similar to net $\NetEWC1$, with the only differences that features $U_1$ and $U_2$ are exchanged, and the CP tables of $\NetEWC2$ are adjusted to reflect this change.

\item $\NetEWC3=\tuple{\FeatNet{\NetEWC3},\LinkNet{\NetEWC3}}$ is as follows (see \cref{fig:ExistsWeakCond_Hardness} for a schematic representation of the links).
    Links of $\NetEWC3$ are the following:
    \begin{itemize}
    \item for each $x_i\in X$, $\{\{V_i^T,V_i'\},\{V_i^F,V_i'\}\} \subset \LinkNet{\NetEWC3}$;

    \item a disjunctive interconnecting $\NetInterAND(|\SetFeat{V'}|+|\SetFeat{W}|+|\SetFeat{P}|+|\SetFeat{D}|+|\SetFeat{A}|)$ over feature set $\SetFeat{B}$ which is connected to the features in $\SetFeat{V'}\cup \SetFeat{W}\cup \SetFeat{P}\cup \SetFeat{D}\cup \SetFeat{A}$;

    \item $\{\{B,U_1\},\{U_1,U_2\}\}\subset\LinkNet{\NetEWC3}$.
    \end{itemize}

    CP tables of $\NetEWC3$ are the following:
    \begin{itemize}
    \item features $F\in(\SetFeat{V}\cup\SetFeat{W}\cup\SetFeat{P}\cup\SetFeat{D}\cup\SetFeat{A})$ have the CP tables
    \begin{tabular}{|l|}
    \hline
    $f\Pref \ol{f}$\\
    \hline
    \end{tabular};

    \item for each variable $x_i\in X$, feature $V_i'\in\SetFeat{V'}$ has the CP table

   \smallskip  \begin{tabular}{|c||l|}
    \hline
    $\ol{v_i^T} \land \ol{v_i^F}$ & $\ol{v_i'}\Pref v_i'$\\
    \hline
    else & $v_i'\Pref \ol{v_i'}$\\
    \hline
    \end{tabular};

\smallskip
    \item features in $\SetFeat{B}$ of the interconnecting net have the usual CP tables;

    \item $U_1$ has the CP table

    \smallskip
    \begin{tabular}{|c||l|}
    \hline
    $\ol{b}$ & $\ol{u_1}\Pref u_1$\\
    \hline
    $b$ & $u_1\Pref \ol{u_1}$\\
    \hline
    \end{tabular};

\smallskip
    \item $U_2$ has the CP table

    \smallskip
    \begin{tabular}{|c||l|}
    \hline
    $\ol{u_1}$ & $\ol{u_2}\Pref u_2$\\
    \hline
    $u_1$ & $u_2\Pref \ol{u_2}$\\
    \hline
    \end{tabular}.
    \end{itemize}

\item $\NetEWC4$ is equal to net $\NetEWC3$;

\item $\NetEWC5$ is composed by a link from $U_1$ to $U_2$, and a direct net $\NetDirect(\gamma)$, where $\gamma$ is defined over all features but $U_1$ and $U_2$, and assigns non-overlined values to all of them.
    The other CP tables of $\NetEWC5$ are:

    \begin{itemize}
    \item feature $U_1$ has the CP table
        \begin{tabular}{|l|}
        \hline
        $\ol{u_1}\Pref u_1$\\
        \hline
        \end{tabular};

    \item feature $U_2$ has the CP table

    \smallskip
        \begin{tabular}{|c||l|}
        \hline
        $\ol{u_1}$ & $\ol{u_2}\Pref u_2$\\
        \hline
        $u_1$ & $u_2\Pref \ol{u_2}$\\
        \hline
        \end{tabular}.
    \end{itemize}

\item $\NetEWC6$ is characterized by having a link from feature $U_2$ to any other feature.
   The CP tables of $\NetEWC6$ are:

    \begin{itemize}
    \item feature $U_2$ has  the CP table
        \begin{tabular}{|l|}
        \hline
        $\ol{u_2}\Pref u_2$\\
        \hline
        \end{tabular};

    \item features $F$ different from $U_2$ have the CP table

    \smallskip
        \begin{tabular}{|c||l|}
        \hline
        $\ol{u_2}$ & $f \Pref \ol{f}$\\
        \hline
        $u_2$ & $\ol{f} \Pref f$\\
        \hline
        \end{tabular}.
    \end{itemize}
\end{itemize}

Observe that $\MNetExistsWeakCond(\Phi)$ is acyclic, binary, its indegree is three, and can be computed in polynomial time from $\phi$.
Moreover, the class of \MCPNets ${\{\MNetExistsWeakCond(\Phi)\}}_{\Phi}$ derived from formulas $\Phi$ of the specified kind and according to the reduction shown above is polynomially connected.
It is possible to show that $\Phi=(\exists X)(\forall Y)\lnot\phi(X,Y)$ is valid if and only if $\MNetExistsWeakCond(\Phi)$ has a majority optimal outcome.

\begin{lemma}\label{lemma_properties_of_ExistsWeakCondorcet_reduction}
Let $\Phi=(\exists X)(\forall Y)\lnot\phi(X,Y)$ be a quantified Boolean formula, where $\phi(X,Y)$ is a $3$CNF Boolean formula, defined over two disjoint sets $X$ and $Y$ of Boolean variables.
Then, $\Phi$ is valid if and only if $\MNetExistsWeakCond(\Phi)$ has a majority optimal outcome.
\end{lemma}
\begin{proof}[Proof (sketch).]
The intuition at the base of the proof is to put in relationship truth assignments over the variable set $X$ with outcomes of $\MNetExistsWeakCond(\Phi)$.
In particular, given an assignment $\sigma_X$ over $X$, the associated outcome is $\beta_{\sigma_X}$, where $\sigma_X$ is encoded over the feature set $\SetFeat{V}$ in the usual way, and all other features have non-overlined values.
We show first that outcomes not in the form of a $\beta_{\sigma_X}$ are not majority optimal.
Then, we show that if there is an assignment $\sigma_X$ such that $(\forall Y)\lnot\phi(X/\sigma_X,Y)$ is valid, then $\beta_{\sigma_X}$ is majority optimal.
On the other hand, if there is no assignment~$\sigma_X$ for which $(\forall Y)\lnot\phi(X/\sigma_X,Y)$ is valid, then none of the outcomes $\beta_{\sigma_X}$ is majority optimal.
Details of the proof are given at page~\pageref{page_pointer:properties_of_ExistsWeakCondorcet_reduction_detailed}.
\end{proof}

We now prove that deciding whether an \MCPNet has a majority optimal outcome is \SigmaPh{2}.

\begin{theorem}\label{theo_exists_weak_condorcet_mcpnets_hardness}
Let $\MNet{M}$ be an \MCPNet.
Then, deciding whether there is a majority optimal outcome in $\MNet{M}$ is \SigmaPh2.
Hardness holds even on polynomially connected classes of acyclic binary \MCPNets with indegree at most three and at most six agents.
\end{theorem}
\begin{proof}
We prove the hardness of \ExistsMajorityOptimal by showing a reduction from QBF$^{\mathit{CNF}}_{2,\forall,\lnot}$.
Let $\Phi = (\exists X)(\forall Y)\linebreak[0]\lnot\phi(X,Y)$ be an instance of QBF$^{\mathit{CNF}}_{2,\forall,\lnot}$, and consider the $6$\CPNet $\MNetExistsWeakCond(\Phi)$.
By \cref{lemma_properties_of_ExistsWeakCondorcet_reduction}, $\Phi$ is a ``yes''-instance of QBF$^{\mathit{CNF}}_{2,\forall,\lnot}$ if and only if there is majority optimal outcome in $\MNetExistsWeakCond(\Phi)$.
\end{proof}

By combining the two above results, we immediately conclude that deciding the existence of majority optimal outcomes over polynomially connected classes of acyclic \MCPNets is \SigmaPc{2}.

\begin{corollary}
Let $\mathcal{C}$ be a polynomially connected class of acyclic \MCPNets.
Let $\MNet{M}\in\mathcal{C}$ be an \MCPNet.
Then, deciding whether there is a majority optimal outcome in $\MNet{M}$ is \SigmaPc2.
\end{corollary}

\subsection{Complexity of majority optimums on \texorpdfstring{$m$}{m}CP-nets}\label{sec_complexity_majority_optimum}

We now focus  on majority optimum outcomes.
In particular, the problems analyzed are deciding whether an outcome is majority optimum, and deciding whether an \MCPNet has a majority optimum outcome.
We first consider the problem of deciding whether an outcome is majority optimum in an \MCPNet.
We prove that this problem is \PiPc{2}.
More formally, consider the following problem.

\medbreak

\begin{tabular}{rl}
  \textit{Problem:} & \IsMajorityOptimum\\
  \textit{Instance:} & An \MCPNet $\MNet{M}$, and an outcome $\alpha\in\OutMNet{\MNet{M}}$.\\
  \textit{Question:} & Is $\alpha$ majority optimum in $\MNet{M}$?
\end{tabular}

\medbreak
The following result shows that, on preference profiles represented via an \NP-representations scheme, deciding whether an outcome is majority optimum is feasible in \PiP{2}.

\begin{theorem}\label{theo_is_strong_condorcet_generic_membership}
Let $\Profile{P}$ be a preference profile defined over the same combinatorial domain and represented via an \NP-representation scheme, and let $\alpha$ be an outcome.
Deciding whether $\alpha$ is majority optimum in $\Profile{P}$ is feasible in \PiP2.
\end{theorem}
\begin{proof}
We prove the statement by showing that deciding whether $\alpha$ is \emph{not} majority optimum in $\Profile{P}$ is feasible in \SigmaP2.
If $\alpha$ is majority optimum, then there is an outcome $\beta$ such that $\alpha\not\PrefMajorityProfile{\Profile{P}}\beta$.
Therefore, in order to prove that $\alpha$ is not majority optimum, it suffices to guess $\beta$, and then check that $\alpha\not\PrefMajorityProfile{\Profile{P}}\beta$.
Observe that guessing $\beta$ requires an \NP machine, and then checking $\alpha\not\PrefMajorityProfile{\Profile{P}}\beta$ is feasible in \CoNP (see \cref{theo_majority_query_generic_membership}), which can be carried out by an oracle.
Therefore, the overall procedure is feasible in \SigmaP2.
\end{proof}

To prove the \PiPh{2}{}ness of \IsMajorityOptimum, we use a reduction from QBF$^{\mathit{CNF}}_{2,\forall,\lnot}$.
Consider the following construction.
Let $\Phi = (\exists X)(\forall Y)\lnot\phi(X,Y)$ be a quantified formula where $\phi(X,Y)$ is a $3$CNF Boolean formula defined over two disjoint sets $X=\{x_1,\dots,x_{n_X}\}$ and $Y=\{y_1,\dots,y_{n_Y}\}$ of Boolean variables, and whose set of clauses is $C=\{c_1,\dots,c_m\}$.
From $\phi(X,Y)$, we define the $3$\CPNet $\MNetIsStrongCond(\Phi)=\tuple{\NetISC1,\NetISC2,\NetISC3}$ in the following way.

\smallskip

\noindent The features of $\MNetIsStrongCond(\Phi)$ are:
\begin{itemize}
\item all the features of a net $\NetCNFSumm(\phi)$ (defined in \cref{sec_summ_formula_net}) in which, in this case, we distinguish two variable feature sets $\SetFeat{V}=\{V_i^T,V_i^F\mid x_i\in X\}$ and $\SetFeat{W}=\{W_i^T,W_i^F\mid y_i\in Y\}$ (recall that $\SetFeat{P}$ and $\SetFeat{D}$ are the sets of literal and clause features, respectively, and $\SetFeat{A}$ is the set of features of the conjunctive interconnecting net embedded in $\NetCNFSumm(\phi)$);

\item all the features of set $\SetFeat{V'}=\{V_i'\mid x_i\in X\}$;

\item all the features of the set $\SetFeat{B}$ which are the features $B_i$ of a disjunctive interconnecting net (defined in \cref{sec_interconnecting_net}) $\NetInterOR(|\SetFeat{V'}|+|\SetFeat{W}|+|\SetFeat{P}|+|\SetFeat{D}|+|\SetFeat{A}|)$ and its apex is feature $B$ (observe that features $B_i$ are distinct from features $A_i$ of the conjunctive interconnecting net $\NetInterAND(m)$ embedded in $\NetCNFSumm(\phi)$).
\end{itemize}

To summarize, all the features of $\MNetIsStrongCond(\Phi)$ are:
$\SetFeat{V} \cup \SetFeat{V'} \cup \SetFeat{W} \cup \SetFeat{P} \cup \SetFeat{D} \cup \SetFeat{A} \cup \SetFeat{B} \cup \{U_1,U_2\}$.

The \CPNets of $\MNetIsStrongCond(\Phi)$ are:
\begin{itemize}
\item $\NetISC1$ is composed by a net $\NetCNFSumm(\phi)$ with its features, links, and CP tables, and a direct net $\NetDirect(\gamma)$ (see \cref{sec_direct_net}), where $\gamma$ is defined over features in $\SetFeat{V'}\cup\SetFeat{B}$ and assigns non-overlined values to all of them.

\item $\NetISC2=\NetDirect(\ol{\alpha})$, with $\ol{\alpha}$ defined over all the features of $\MNetIsStrongCond(\Phi)$, and having overlined values only for features $U_1$ and $U_2$.

\item $\NetISC3=\tuple{\FeatNet{\NetISC3},\LinkNet{\NetISC3}}$ is as follows (see \cref{fig:ExistsWeakCond_Hardness} for a schematic representation of the links of $\NetEWC3$, which is equivalent to $\NetISC3$).
    Links of $\NetISC3$ are the following:
    \begin{itemize}
    \item for each $x_i\in X$, $\{\{V_i^T,V_i'\},\{V_i^F,V_i'\}\} \subset \LinkNet{\NetISC3}$;

    \item a disjunctive interconnecting $\NetInterAND(|\SetFeat{V'}|+|\SetFeat{W}|+|\SetFeat{P}|+|\SetFeat{D}|+|\SetFeat{A}|)$ over feature set $\SetFeat{B}$ which is connected to the features in $\SetFeat{V'}\cup \SetFeat{W}\cup \SetFeat{P}\cup \SetFeat{D}\cup \SetFeat{A}$;

    \item $\{\{B,U_1\},\{U_1,U_2\}\}\subset\LinkNet{\NetISC3}$.
    \end{itemize}

    CP tables of $\NetISC3$ are the following:
    \begin{itemize}
    \item features $F\in(\SetFeat{V}\cup\SetFeat{W}\cup\SetFeat{P}\cup\SetFeat{D}\cup\SetFeat{A})$ have the CP tables
    \begin{tabular}{|l|}
    \hline
    $f\Pref \ol{f}$\\
    \hline
    \end{tabular};

    \item for each variable $x_i\in X$, feature $V_i'\in\SetFeat{V'}$ has the CP table

    \smallskip
    \begin{tabular}{|c||l|}
    \hline
    $\ol{v_i^T} \land \ol{v_i^F}$ & $\ol{v_i'}\Pref v_i'$\\
    \hline
    else & $v_i'\Pref \ol{v_i'}$\\
    \hline
    \end{tabular};

\smallskip
    \item features in $\SetFeat{B}$ of the interconnecting net have the usual CP tables;

    \item $U_1$ has the CP table

    \smallskip
    \begin{tabular}{|c||l|}
    \hline
    $\ol{b}$ & $\ol{u_1}\Pref u_1$\\
    \hline
    $b$ & $u_1\Pref \ol{u_1}$\\
    \hline
    \end{tabular};

\smallskip
    \item $U_2$ has the CP table

    \smallskip
    \begin{tabular}{|c||l|}
    \hline
    $\ol{u_1}$ & $\ol{u_2}\Pref u_2$\\
    \hline
    $u_1$ & $u_2\Pref \ol{u_2}$\\
    \hline
    \end{tabular}.
    \end{itemize}
\end{itemize}

Observe that $\MNetIsStrongCond(\Phi)$ is acyclic, binary, its indegree is three, and can be computed in polynomial time from $\phi$.
Moreover, the class of \MCPNets ${\{\MNetIsStrongCond(\Phi)\}}_{\Phi}$ derived from formulas $\Phi$ of the specified kind and according to the reduction shown above is polynomially connected.
The following result shows that $\Phi=(\exists X)(\forall Y)\lnot\phi(X,Y)$ is valid if and only if a particular outcome of $\MNetIsStrongCond(\Phi)$ is majority optimum.

\begin{lemma}\label{lemma_properties_of_IsStrongCondorcet_reduction}
Let $\Phi=(\exists X)(\forall Y)\lnot\phi(X,Y)$ be a quantified Boolean formula, where $\phi(X,Y)$ is a $3$CNF Boolean formula, defined over two disjoint sets $X$ and $Y$ of Boolean variables.
Then, $\Phi$ is valid if and only if $\MNetIsStrongCond(\Phi)$ does not have a majority optimum outcome.
In particular, when $\MNetIsStrongCond(\Phi)$ has a majority optimum outcome it is the outcome $\ol{\alpha}\in\OutMNet{\MNetIsStrongCond(\Phi)}$ assigning overlined values only to features $U_1$ and $U_2$.
\end{lemma}
\begin{proof}[Proof (sketch).]
The intuition at the base of the proof is to put in relationship truth assignments over the variable set $X$ with outcomes of $\MNetIsStrongCond(\Phi)$.
In particular, given an assignment $\sigma_X$ over $X$, the associated outcome is $\beta_{\sigma_X}$, where $\sigma_X$ is encoded over the feature set $\SetFeat{V}$ in the usual way, and all other features have non-overlined values.
We show first that $\ol{\alpha}$ majority dominates any other outcome $\beta$ that is not in the form of a $\beta_{\sigma_X}$ outcome, and hence none of them is majority optimum.
Moreover, for all such outcomes $\beta_{\sigma_X}$, $\beta_{\sigma_X}$ does not majority dominates $\ol{\alpha}$, and hence, again, none of them is majority optimum.
So, only $\ol{\alpha}$ is candidate to be majority optimum.
Then, we show that if there is an assignment $\sigma_X$ such that $(\forall Y)\lnot\phi(X/\sigma_X,Y)$ is valid, then $\ol{\alpha}\not\PrefMajorityMNet{\MNetIsStrongCond(\Phi)}\beta_{\sigma_X}$, and hence $\ol{\alpha}$ is not majority optimum, which implies that $\MNetIsStrongCond(\Phi)$ does not have any majority optimum outcome.
On the other hand, if there is no assignment $\sigma_X$ for which $(\forall Y)\lnot\phi(X/\sigma_X,Y)$ is valid, then, for all the outcomes $\beta_{\sigma_X}$, $\ol{\alpha}\PrefMajorityMNet{\MNetIsStrongCond(\Phi)}\beta_{\sigma_X}$, which implies that $\ol{\alpha}$ is majority optimum.
Details of the proof are given at page~\pageref{page_pointer:properties_of_IsStrongCondorcet_reduction_detailed}.
\end{proof}

The next result shows that deciding whether an outcome is majority optimum in an \MCPNet is \PiPh{2}.

\begin{theorem}\label{theo_is_strong_condorcet_mcpnets_hardness}
Let $\MNet{M}$ be an \MCPNet, and let $\alpha\in\OutMNet{\MNet{M}}$ be an outcome.
Then, deciding whether $\alpha$ is majority optimum in $\MNet{M}$ is \PiPh2.
Hardness holds even on polynomially connected classes of acyclic binary \MCPNets with indegree at most three and at most three agents.
\end{theorem}
\begin{proof}
We prove the statement by showing that deciding whether $\alpha$ is \emph{not} majority optimum is \SigmaPh2, and we do this by exhibiting a reduction from QBF$^{\mathit{CNF}}_{2,\forall,\lnot}$ to the complement problem to \IsMajorityOptimum.
Let $\Phi = (\exists X)(\forall Y)\linebreak[0]\lnot\phi(X,Y)$ be an instance of QBF$^{\mathit{CNF}}_{2,\forall,\lnot}$, and consider the $3$\CPNet $\MNetIsStrongCond(\Phi)$, and the outcome $\ol{\alpha}$ in which only the values of features $U_1$ and $U_2$ are overlined.
By \cref{lemma_properties_of_IsStrongCondorcet_reduction}, $\Phi$ is a ``yes''-instance of QBF$^{\mathit{CNF}}_{2,\forall,\lnot}$ if and only if $\alpha$ is \emph{not} majority optimum in $\MNetIsStrongCond(\Phi)$.
\end{proof}

Note that, with respect to the number of agents, the above result is optimal.
Indeed, since majority dominance and Pareto dominance are equivalent on $m$\CPNets with $m \leq 2$, it is not possible to show the \PiPh2{}ness of \IsMajorityOptimum on $m$\CPNets with $m < 3$.

By combining the two previous results, we immediately conclude that deciding whether an outcome is majority optimum over polynomially connected classes of acyclic \MCPNets is \PiPc{2}.

\begin{corollary}
Let $\mathcal{C}$ be a polynomially connected class of acyclic \MCPNets.
Let $\MNet{M}\in\mathcal{C}$ be an \MCPNet, and let $\alpha\in\OutMNet{\MNet{M}}$ be an outcome.
Then, deciding whether $\alpha$ is majority optimum in $\MNet{M}$ is \PiPc2.
\end{corollary}

To conclude, we study the complexity of deciding whether an \MCPNet has a majority optimum outcome.
We show that the problem is \PiPh{2} and belongs to \DP{2}.
More formally, consider the following problem.

\medbreak

\begin{tabular}{rl}
  \textit{Problem:} & \ExistsMajorityOptimum\\
  \textit{Instance:} & An \MCPNet $\MNet{M}$.\\
  \textit{Question:} & Does $\MNet{M}$ have a majority optimum outcome?
\end{tabular}

\medbreak

We prove that deciding whether a preference profile represented via an \NP-representation scheme has a majority optimum outcome is feasible in $\DP2$.
The following \lcnamecref{lemma_not_exists_one_weak_condorcet_that_is_not strong_membership} is an intermediate result to show the membership in $\DP2$.

\begin{lemma}\label{lemma_not_exists_one_weak_condorcet_that_is_not strong_membership}
Let $\Profile{P}$ be a preference profile defined over the same combinatorial domain and represented via an \NP-representation scheme.
Then, deciding whether $\Profile{P}$ does not have majority optimal outcomes that are not also majority optimum is feasible in $\PiP2$.
\end{lemma}
\begin{proof}
To prove the statement of the lemma, we show that the complement task to the one in the statement is feasible in $\SigmaP2$.
First, we guess two different outcomes $\alpha$ and $\beta$, which is feasible in \NP.
Then, through a \CoNP oracle call we check whether $\alpha$ is actually majority optimal (see \cref{theo_is_weak_condorcet_generic_membership}), then through another oracle call in \CoNP we check that $\alpha\not\PrefMajorityProfile{\Profile{P}}\beta$ (see \cref{theo_majority_query_generic_membership}).
If the answer to the latter oracle call is ``yes'', then it means that $\alpha$ is a majority optimal outcome that is not also majority optimum (because $\alpha$ does not majority dominate $\beta$).
To conclude, observe that the overall procedure is feasible in $\SigmaP2$.
\end{proof}

\begin{theorem}\label{theo_exists_strong_condorcet_generic_membership}
Let $\Profile{P}$ be a preference profile defined over the same combinatorial domain and represented via an \NP-representation scheme.
Then, deciding whether $\Profile{P}$ has a majority optimum outcome is feasible in $\DP2$.
\end{theorem}
\begin{proof}
Let $\mathcal{O}$ be a combinatorial domain.
Let $A$ be the set of all the preference profiles over $\mathcal{O}$ having at least a majority optimal outcome.
Let $B$ be the set of all the preference profiles over $\mathcal{O}$ not having majority optimal outcomes that are not also majority optimum (observe that in $B$ there might be profiles with no majority optimal outcome at all, and profiles having a majority optimal outcome that is also majority optimum).
Clearly, the intersection $A \cap B$ is the set of all the preference profiles over $\mathcal{O}$ having a majority optimum outcome.
Observe that, if we focus only on \NP-representations schemes, then set $A$ can be decided in $\SigmaP2$ (see \cref{theo_exists_weak_condorcet_generic_membership}), and set $B$ can be decided in $\PiP2$ (see \cref{lemma_not_exists_one_weak_condorcet_that_is_not strong_membership}).
Therefore, deciding whether $\Profile{P}$ has a majority optimum outcome is feasible in \DP2.
\end{proof}

For a lower bound of the problem \ExistsMajorityOptimum we can exploit the construction used to show the lower bound of \IsMajorityOptimum, and obtain the following \lcnamecref{theo_exists_strong_condorcet_hardness}, which proves that on \MCPNets deciding the existence of a majority optimum outcome is \PiPh2.

\begin{theorem}\label{theo_exists_strong_condorcet_hardness}
Let $\MNet{M}$ be an \MCPNet.
Then, deciding whether $\MNet{M}$ has a majority optimum outcome is \PiPh2.
Hardness holds even on polynomially connected classes of acyclic binary \MCPNets with indegree at most three and at most three agents.
\end{theorem}
\begin{proof}
We prove the statement by showing a reduction from QBF$^{\mathit{CNF}}_{2,\forall,\lnot}$ to the complement problem to \ExistsMajorityOptimum.
Let $\Phi = (\exists X)(\forall Y)\linebreak[0]\lnot\phi(X,Y)$ be an instance of QBF$^{\mathit{CNF}}_{2,\forall,\lnot}$, and consider the $3$\CPNet $\MNetIsStrongCond(\Phi)$.
By \cref{lemma_properties_of_IsStrongCondorcet_reduction}, $\Phi$ is a ``yes''-instance of QBF$^{\mathit{CNF}}_{2,\forall,\lnot}$ if and only if $\MNetIsStrongCond(\Phi)$ has not a majority optimum outcome.
\end{proof}

By combining the two above results, we immediately conclude that deciding the existence of majority optimum outcomes over polynomially connected classes of acyclic \MCPNets is in \DP{2} and \PiPh{2}.

\begin{corollary}
Let $\mathcal{C}$ be a polynomially connected class of acyclic \MCPNets.
Let $\MNet{M}\in\mathcal{C}$ be an \MCPNet.
Then, deciding whether $\MNet{M}$ has a majority optimum outcome is \PiPh2 and in \DP2.
\end{corollary}

\section{Related work}\label{sec_related}

In this section, we describe the larger context of this work in the literature, especially its relationship to previous work on
compact representations of preferences, global and sequential voting, and the notion of $\mathcal{O}$-legality.

\subsection{Compact representations of preferences}\label{sec_representation_preferences}

The preferences of agents can be represented in different ways, and preference (representation) models can essentially be divided into
\emph{quantitative}  and \emph{qualitative} ones~\cite{Lang2004,Boutilier2004}.
The former preference models associate with each outcome a numerical value, which is the value of a utility function, and
preferences between outcomes are evaluated by comparing the utility values.
The latter preference models provide a (not necessarily complete) order over the outcomes, which can be represented in multiple ways, e.g.,
via a plain sequence of outcomes, or, more generally, via a binary relation over the outcomes, i.e., a set of ordered pairs of outcomes, in which the first outcome of the pair is considered preferred to the second.
The key point of qualitative preferences is that there is no precise quantification of the utility associated with the outcomes.
For this reason, it has been argued in the literature that qualitative preference are easier to be stated by humans~\cite{Boutilier2004,Brafman2009}, as it is not necessary to estimate the utility values, which can be quite challenging.
In this paper, we focus only on qualitative preference models, encoding preference relations.

A preference model is \emph{extensive}, if for any preference relation $\Ranking{R}$ that it is capable to represent, each outcome of the domain of $\Ranking{R}$ appears explicitly at least once in its representation of $\Ranking{R}$.
This implies that the size of its representation of $\Ranking{R}$  is at least linear in the number of outcomes of $\Ranking{R}$.
So, for example, a sequence of all the outcomes, or listing all the pairs of a preference relation, are both extensive representations.
By definition, the number of possible outcomes in combinatorial domains is exponential in the number of features.
Hence, the extensive representation of agents' preferences over combinatorial domains becomes quickly infeasible and unrealistic when the number of features is more than just a few.
For this reason, compact formalisms to represent combinatorial preferences are needed~\cite{Lang2002,Lang2004,Lang_handbook}.
An ideal compact representation scheme for combinatorial preferences would be one such that the space required for the representation is polynomial in the number of the features (and not in the number of the outcomes, like for extensive representations).
Clearly, for information-theoretic reasons, it is not possible to have a preference model to compactly represent \emph{any} possible preference relation over a combinatorial domain.
Indeed, for a combinatorial domain characterized by $n$ features of two values each, i.e., \emph{binary features}, the number of all possible (complete) relations is $(2^n)!$, which is $O(2^{2^n})$.
Hence, no preference model could ever represent all the possible preference relations using only polynomial space in $n$~\cite{Lang_handbook}.
Nonetheless, if we want to represent preference relations showing particular structures or patterns, then it is possible to take advantage of these patterns and ``decompose'' the relation to summarize it into a concise representation.
In the literature, there has been considerable work on exploiting the structure of preferences to appropriately decompose them, and, e.g., Lang~\cite{Lang2004} gives a survey of different logical languages for compact
preference~representation (see also~\cite{Boutilier2004} for more references).

Among the preferences' structural patterns exploited to achieve compact representations, the (conditional) preference independence of the features is one of the most studied (see~\cite{Keeney1976,Boutilier2004,Lang2004,Lang_ai_magazine,Lang_handbook,Lang2007,Xia2007,Lang2009} and references therein).
Intuitively, the (conditional) preferential independence of the features implies that the preference relation between outcomes varying on specific features can be influenced by the values of some features and can be totally independent from the values of some other features. When features are (conditionally) preferentially independent, it roughly means that portions of the structure of the preference relation are replicated throughout the preference relation.
Therefore, these patterns in the preferences can be ``factorised'' to save space in the representation.

In the literature, \emph{(conditional) ceteris paribus preference statements} have been proposed several times to 
compactly represent preferences with (conditional) preferential independencies among (sets of) features~\cite{Lang2004,Boutilier2004}.
Moreover, a preference representation scheme should capture statements that are natural for agents to assess, and (conditional) ceteris paribus preference statements have several times been argued to be intuitive for users, as they resemble the way in which humans express their preferences and act upon them~\cite{Hansson1996,Boutilier1999,Hansson2002,Boutilier2004}.

Ceteris paribus means ``all else being equal'', and ceteris paribus statements were classically defined to be non-conditional~\cite{Hansson1996}.
Given a set of features $\SetFeat{V}\subseteq\SetFeat{F}$, a non-conditional ceteris paribus preference statement sounds like:
``Outcomes varying over $\SetFeat{V}$, and all else being equal, are ranked according to the following preference relation restricted over $\SetFeat{V}$''.
For example, a non-conditional ceteris paribus preference statement is ``I prefer a round table to a square one, all else being equal''~\cite{Hansson1996}.
In this example, the set $\SetFeat{V}$ contains the feature ``shape'' (of the table).
Observe that, here, the ``all else being equal'' does not refer to the fact that the values of the other features determine the preference relation restricted over $\SetFeat{V}$. Instead, it means that the given ceteris paribus statement allows to compare outcomes varying only on $\SetFeat{V}$: outcomes varying over $\SetFeat{V}$ \emph{and} other features cannot be compared via the given statement.

Boutilier et al.~\cite{Boutilier1999,Boutilier2004} extended the idea of ceteris paribus statements to conditional ceteris paribus preference statements.
In this case, given two disjoint sets of features $\SetFeat{V},\SetFeat{Z}\subseteq\SetFeat{F}$, a conditional ceteris paribus preference statement sounds like:
``Given the specific instantiation $\gamma$ of values for the features in $\SetFeat{Z}$, outcomes varying over $\SetFeat{V}$, and all else being equal, are ranked according to the following preference relation restricted over $\SetFeat{V}$''.
For example, a conditional ceteris paribus preference statement could be ``Given that the main course of the dinner is meat, I prefer a red wine to a white one, all else being equal''.
Here, $\SetFeat{Z}$ contains the feature ``main course'', while  $\SetFeat{V}$ contains the feature ``wine''.

Among the representation schemes proposed in the literature based on (conditional) ceteris paribus statements, one using propositional logics was proposed by Lang~\cite{Lang2004}, while Boutilier et al.~\cite{Boutilier1999,Boutilier2004} proposed \CPNets; for other representations based on (conditional) ceteris paribus statements, see references in~\cite{Lang2004,Boutilier2004}.
More details on \CPNets are given in the introduction and in the preliminaries.

\subsection{Global and sequential voting} \label{sec72}

The graphical structure of \CPNets shows that, in general, combinatorial preferences may exhibit dependencies between features.
Dependencies are certainly a critical characteristic to model, especially to attain compact representations,
however, they can become troublesome when combinatorial preferences are aggregated.
Whether dependencies are actually problematic or not depends on the specific ways in which agents' votes are collected.
These specific ways in which votes are collected are called \emph{voting protocols}~\cite{Conitzer2005}.
A voting protocol characterizes how a voting rule is implemented, i.e., it determines which information is elicited from the agents, and when this is done.
In the literature, a number of different voting protocols for combinatorial vote have been considered (see, e.g.,~\cite{Lang_ai_magazine,Lang_handbook,Lacy2000} and references therein).
Two of them are global and sequential voting~\cite{Lang2007,Lang_ai_magazine,Xia2007,Lang_handbook,Lang2009,Lacy2000}.

In \emph{global voting}, agents' preferences over the \emph{entire combinatorial domain} are collected at the \emph{same time}, and then a voting rule is applied to the entire preference profile to select the winner(s).
In case of ($m$)\CPNets, global voting consists in the agents communicating their whole \CPNets for vote aggregation.
In \emph{sequential voting}, votes are collected feature-by-feature.
In particular, agents express at the \emph{same time} their preferences over the \emph{individual} features, and votes for the different features are collected \emph{sequentially} in consecutive steps.
In this case, a voting rule determines the winning value $v$ for a feature, and, most importantly, agents are informed about $v$ before the next feature is considered for voting in the protocol. 
Note that sometimes, in sequential voting, instead of voting for a sequence of single features, agents may be asked in a step of the protocol to express their preferences over the combined values (i.e., the value vectors) of a set of features (see, e.g.,~\cite{Xia2007,Airiau2011}).
Global voting is the protocol of the voting semantics in \MCPNets.

Dependencies are not an issue in global voting, because agents communicate their preferences over the entire (exponential) combinatorial domain (or the entire \CPNets, if \CPNets are used to represent the agents' preferences), and hence all the information needed for the aggregation is available.
However, global voting can be expensive to implement and evaluate, especially if extensive representations are adopted or preference relations are extensively unfolded from the compact representation before any further processing. %
Strictly speaking, from a theoretical perspective the computational complexity of aggregating preferences represented via extensive schemes can be lower than then complexity of performing the same task over compactly represented preferences. However, this computational simplification is artificial, because the computational complexity of a problem is evaluated relative to the size of the input. For extensive representations, the input is huge, and hence the computational complexity of problems over this kind of input can be low. Nevertheless, processing input of remarkable size can be computationally challenging.
The computational burden of global voting can be limited by adopting 
sequential voting.
A benefit of 
sequential voting over global voting is lowering the communication complexity, i.e., the amount of information needed to be exchanged by the agents to implement the protocol~\cite{Conitzer2005,Lang_handbook,Lang_ai_magazine}.
In sequential voting, agents are enquired in consecutive steps about their preferences for individual features, therefore the information required to be exchanged is only the preferences ``projected'' over individual features.
However, feature dependencies can be quite detrimental for 
sequential voting, to the point that sub-optimal outcomes are selected---examples of this phenomenon are called \emph{multiple election paradoxes}---or agents can experience regret after voting~\cite{Lacy2000,Lang2007,Lang2009,Xia2007,Lang_ai_magazine,Lang_handbook}.
Intuitively, 
paradoxes may occur in sequential voting, because votes over the different features are collected separately, and this can clash with the different individual preferential dependencies that agents may have between features---more specifically, it might be the case that agents have to vote on a feature whose preferences depend on features for which it has not been voted yet~\cite{Brams1998}. 

Lacy and Niou~\cite{Lacy2000} showed that these paradoxes in sequential voting can be (partly) avoided if the considered preferences are separable, i.e., they do not have dependencies among features.
Intuitively, when represented via \CPNets, combinatorial preferences without dependencies among features do not have any edge between vertices.
Clearly, this is a very strong assumption, and it is unlikely to be met in practice~\cite{Lang2007,Xia2007,Xia2007a,Lang2009}.

\subsection{Overcoming multiple election paradoxes via \texorpdfstring{$\mathcal{O}$}{O}-legality}\label{sec73}

To overcome the strong limitation imposed by preference separability, Lang~\cite{Lang2007} proposed and investigated a weaker structural restriction of preferences, called $\mathcal{O}$-legality, which preserves the nice properties of sequential voting (when evaluated over separable preferences) on a wider class of combinatorial preferences.
Intuitively, a preference profile is $\mathcal{O}$-legal, if the dependencies among the features for all the agents comply with a common sequence.
More formally, if $\mathcal{O}=(F_1,\dots,F_m)$ is a sequence of all the features of the combinatorial domain, a preference profile $\mathcal{P}$ is \emph{$\mathcal{O}$-legal}, if for any agent $A$ and any two features $F_i$ and $F_j$, if $F_i$ precedes $F_j$ in $\mathcal{O}$, then $A$'s preferences for $F_i$ do not depend on $F_j$'s value~\cite{Xia2007a}.
The concept of $\mathcal{O}$-legality has an immediate translation over \CPNets.
Indeed, a profile of (preferences represented via) \CPNets is $\mathcal{O}$-legal, if $\mathcal{O}$ is a topological order for all the \CPNets' graphs of the profile.
Observe that the existence of a topological order for the features of a \CPNet imposes that the graph of the \CPNet is acyclic.

The intuitiveness of \CPNets to model combinatorial preferences, together with the convenient characterization of $\mathcal{O}$-legality in \CPNets, has largely encouraged the study of sequential voting in $\mathcal{O}$-legal (and acyclic) \CPNets.
In a first group of works, the sequential composition of the voting rule assumed that a feature order was given beforehand (and the voting rule was defined upon the given order)~\cite{Lang2007,Xia2007,Lang2009}.
Next, this idea was generalized by assuming that the sequential voting rule was not defined given the specific order~\cite{Xia2007a}, however, the existence of a shared topological order among the features was anyway required.
A study on how well solutions computed via sequential voting approximate the winning outcomes obtained via global voting was presented in \cite{Conitzer2012}.
Also various other works considered $\mathcal{O}$-legal \CPNets (see, e.g.,~\cite{Loreggia2018,Mattei2013,Maran2013,Cornelio2015,Grandi2014}). Among them, an interesting approach to preference aggregation over $\mathcal{O}$-legal \CPNets was proposed in \cite{Cornelio2015}, where ``probabilistic'' \CPNets were used to represent the result of the aggregation.

\subsection{Going beyond \texorpdfstring{$\mathcal{O}$}{O}-legality}
However, also $\mathcal{O}$-legality is somewhat demanding, because it imposes that there are no ``inversions'' in the preference dependencies.
For example, if in a profile of \CPNets encoding preferences for a dinner there were an agent whose choice of the main dish influences the choice of the wine and an other agent whose choice of the wine influences the choice of the main dish, then those \CPNets would not be $\mathcal{O}$-legal.
(Observe that the profile of \CPNets in \cref{fig:dinner_aggregation} is not $\mathcal{O}$-legal).
Hence, also assuming $\mathcal{O}$-legality is in the end quite restrictive~\cite{Xia2008,Li2011,Sikdar2017a}.
To go beyond the restrictions imposed by $\mathcal{O}$-legality, different approaches were proposed.

One of these approaches is generalizing the idea of sequential voting.
A voting agenda specifies the order in which the features have to be considered in sequential voting.
If the voting agenda does not clash with any of the feature dependencies of all the agents, then it is possible to avoid multiple election paradoxes.
In particular, if $\mathcal{O}$ is a shared topological order of the features of the \CPNets in a profile, then $\mathcal{O}$ is a voting agenda compatible with the feature dependencies of all the agents.
A generalization of this idea is to have a sequence of elections for sets of joint features that cannot be decomposed due to preferential dependencies of some of the agents.
In this way, the problem can be shifted to deciding suitable generalized voting agendas that do not clash with the agents' feature dependencies~\cite{Airiau2011}.

Another approach to overcome the limitations imposed by $\mathcal{O}$-legality is  hypercubewise preference aggregation~\cite{Xia2008}. 
This family of voting rules decomposes the preference aggregation task into two phases:
first, a (hypercubewise) dominance graph is built by applying local voting rules to set of outcomes differing only for the value of a single specific feature %
(these outcomes are the neighboring vertices of the extended preference graph of \CPNets, which, in case of binary features, constitute a hypercube---this is where the name of the voting rule comes from);
and then the winners are chosen from the hypercubewise dominance graph via choice sets functions, which may select dominating or undominated outcomes in the graph.
The idea of the hypercubewise aggregation is at the base of the definition of the hypercubewise Condorcet winner (i.e., the hypercubewise majority winner, which is the outcome majority dominating all its neighbors)~\cite{Xia2008,Conitzer2011}, also called the local Condercet winner, in which the hypercubewise dominance graph is obtained via majority voting.
The definition of the hypercubewise Condorcet winner is different from the definition of the standard, or global, Condorcet winner (the majority optimum outcome, in this paper) obtained from global voting~\cite{Xia2008,Conitzer2011}.
Also the dominance relation inferred from the hypercubewise dominance graph is different from the standard majority dominance relation obtained via global voting~\cite{Li2011}.
In~\cite{Li2011}, also other relations between local and global Condorcet winners are investigated.
It was shown that deciding the existence of hypercubewise Condorcet winners is \NPc~\cite{Conitzer2011,Li2011}.

A third approach proposes, similarly to the previous, to define new voting rules that can be applied also over profiles of non $\mathcal{O}$-legal \CPNets.
These new rules select the outcomes minimizing the total value of a loss function evaluated over the profile of \CPNets~\cite{Sikdar2017,Sikdar2017a}.

\subsection{Analysis of global voting over CP-nets}
Although the proposed approaches can deal with \CPNet profiles that are not $\mathcal{O}$-legal, they do not address the complexity analysis of global voting over (not necessarily $\mathcal{O}$-legal) acyclic \CPNets.
In fact, global voting over non-$\mathcal{O}$-legal acyclic \CPNets has not received as much attention as sequential voting, although it was explicitly stated in the literature that a theoretical comparison between global and sequential voting would have been highly promising~\cite{Lang2007}.

The first work studying global voting over (not necessarily $\mathcal{O}$-legal) acyclic \CPNets was~\cite{Rossi2004}, in which \MCPNets were defined.
Recall that the group dominance semantics of \MCPNets is global voting over a profile of \CPNets.
Voting schemes over \MCPNets were considered from an algorithmic perspective in~\cite{Rossi2004}, which gave a computational insight for global voting over \CPNets.
However, most of the algorithms considered in~\cite{Rossi2004} were brute-force.
Therefore, these algorithms gave only \ExpTime upper bounds for most of the global voting tasks over \CPNets, and no hardness result was provided in~\cite{Rossi2004}.

Algorithms exploiting SAT solvers to compute Pareto and majority optimal outcomes according to global voting over profiles of (not necessarily $\mathcal{O}$-legal) acyclic \CPNets were proposed in \cite{Li2010a} and \cite{Li2010}, respectively.
A SAT solver is used in \cite{Li2011} to compute, over profiles of even cyclic (and therefore also non $\mathcal{O}$-legal) \CPNets, majority optimal and majority optimum outcomes according to global voting, starting from hypercubewise weak Condorcet winners.
The approach of \cite{Li2011} was subsequently extended in \cite{Li2015} to consider also the possibility of multi-valued and incomplete \CPNets.

Despite the mentioned works advanced the study of global voting over (not necessarily $\mathcal{O}$-legal) acyclic \CPNets, they still did not provide precise complexity results.
As mentioned in the introduction, the precise complexity of these problems was actually reported as an open problem multiple times in the literature~\cite{Lang2007,Li2010,Li2010a,Li2011,Li2015,Sikdar2017}.
Our work is the first in the literature tackling directly the complexity analysis of dominance in \MCPNets (and hence the complexity of global voting over \CPNets).

\section{Conclusion}\label{sec_concl_future}
In this paper, we have carried out a thorough complexity analysis of the Pareto and majority semantics in \MCPNets.
Given the specific definitions of  group dominance in \MCPNets, these results characterizes also the complexity of Pareto and majority global voting over \CPNets, which was missing and asked for in the literature various times.
Unlike what is often assumed in the literature, in this work, we have not restricted the profiles of \CPNets to be $\mathcal{O}$-legal, which makes the results achieved here more general.
For the Pareto and the majority voting schemes, we have analyzed the problems of deciding dominance, optimal and optimum outcomes, and the existence of optimal and optimum outcomes.
We have shown completeness results for most cases, which means that we have provided tight lower bounds for problems that (up to date) did not have any explicit lower bound transcending the obvious hardness due to the dominance test over the underlying \CPNets.
Our hardness results are given for polynomially connected classes of binary acyclic ($m$)\CPNets.
This means that our hardness results extend to classes of ($m$)\CPNets encompassing the \CPNets here considered, and in particular also to general \MCPNets with partial \CPNets or multi-valued features.
The various problems analyzed here have been put at various levels of the polynomial hierarchy, and some of them are even tractable (in \PTIME or \LogSpace), which is quite interesting given that for most of these tasks only \ExpTime upper-bounds were known in the literature.

There are various possible directions for further research.
The lower bound for the problem of deciding the existence of majority optimum outcomes does not match the upper bound (\PiPh{2}{}ness and membership in \DP{2}, respectively).
Hence, it would be interesting to close this gap and find the precise complexity of the problem.
Furthermore, characterizing the complexity of preference aggregation when partial \CPNets are allowed to be part of \MCPNets would be interesting, since with (standard) \CPNets, indifference between outcomes are not allowed.
Having constraints on outcomes's feasibility is another interesting direction of investigation.
Without any constraint, \CPNets model agents' preferences when it is assumed that all outcomes are attainable.
However, this is not always the case.
During the aggregation precess, we should take into account what outcomes are feasible.
For example, to decide whether an outcome is majority dominated by another, we should check that the latter is actually feasible.
It will be interesting studying the case in which constraints are issued over the outcome domain prior the preference aggregation and the case in which constraints are considered after the aggregation.
A similar idea characterized the solution concepts in NTU cooperative games defined via constraints~\cite{constrained_games-conferenza,constrained_games-rivista}.
This approach could be merged with the definition of constrained \CPNets~\cite{Boutilier2004a,Prestwich2005},
Finally, it will also be interesting investigating structural restrictions on the structure of \CPNets, in the spirit of what was done in~\cite{bargset_kernel_nucleolus-conferenza,core_bargset_kernel-rivista,core_coalition_structures-conferenza,Ieong2005,Brafman2010a}, to identify broader classes of \CPNets where the dominance test is tractable, whereas, in general, over acyclic \CPNets the dominance test is \NPh.

\section*{Acknowledgments}
This work was supported by the UK EPSRC grants EP/J008346/1, EP/L012138/1, and EP/M025268/1, and by the Alan Turing Institute under the EPSRC grant EP/N510129/1.

\appendix

\section{Detailed proofs}\label{sec_detailed_proofs}

\subsection{Proofs for Section~\ref{sec_complexity_cpnets}}
\phantomsection\label{page_pointer:Pref_Net_CNF_new_detailed}

\smallskip
\noindent\textbf{\cref{lemma_Pref_Net_CNF_new}.}
\emph{%
Let $\phi(X)$ be a Boolean formula in $3$CNF defined over a set $X$ of Boolean variables, and let $\sigma_{X}$ be an assignment on $X$.
Let $\alpha_{\sigma_X}$ be the outcome of $\NetCNF(\phi)$ encoding $\sigma_X$ on the feature set $\SetFeat{V}$, and assigning non-overlined values to all other features, and let $\ol{\beta}$ be the outcome assigning overlined values to all and only variable and clause features.
Then:
\begin{itemize}
\item[(1)] There is an extension of $\sigma_{X}$ to $X$ satisfying $\phi(X)$ if and only if $\ol{\beta}\PrefNet{\NetCNF(\phi)}\alpha_{\sigma_X}$.
\item[(2)] There is no extension of $\sigma_{X}$ to $X$ satisfying $\phi(X)$ if and only if  $\ol{\beta}\IncompNet{\NetCNF(\phi)}\alpha_{\sigma_X}$.
\end{itemize}
}

\begin{proof}
We first  prove~(1).
\begin{itemize}
\item[$(\Rightarrow)$]
Assume that there is an extension $\sigma_{X}'$ of $\sigma_X$ to $X$ satisfying $\phi(X)$.
To prove that $\ol{\beta}\PrefNet{\NetCNF(\phi)}\alpha_{\sigma_X}$ we show that there is an improving flipping sequence from $\alpha_{\sigma_X}$ to $\ol{\beta}$.

Recall that $\sigma_{X}'$ is a complete assignment over $X$, and that if $\sigma_X$ is complete, then $\sigma_X' = \sigma_X$.
For each variable $x_i$ not defined in $\sigma_X$, if $\sigma_{X}'[x_i]=\valtrue$, then we flip feature $V_i^T$ from $v_i^T$ to $\ol{v_i^T}$, analogously if $\sigma_{X}'[x_i]=\valfalse$, then we flip feature $V_i^F$ from $v_i^F$ to $\ol{v_i^F}$.

For all literals $\ell_{j,k}$ evaluating to \valtrue in $\sigma_{X}'$, we flip the corresponding literal features $P_{j,k}$ from $p_{j,k}$ to $\ol{p_{j,k}}$.

Since $\sigma_{X}'$ is a satisfying assignment, for each clause $c_j$ of $\phi$, there is at least a literal $\ell_{j,k}$ evaluating to \valtrue in $\sigma_{X}'$.
For this reason, given any clause feature $D_j$, at this point of the flipping sequence, there is at least one literal feature $P_{j,k}$ with value $\ol{p_{j,k}}$, and hence we can flip $D_j$ from $d_j$ to $\ol{d_j}$.
We can do this for every clause feature.

Then, we flip to their overlined value all variable features that have not been flipped until now.
By the definition of the CP tables of literal features, we can flip all features $P_{j,k}$ having value $\ol{p_{j,k}}$ to $p_{j,k}$, because in the outcome having been built so far through the flips shown above, for all pairs of features $(V_i^T,V_i^F)$, their values are $\ol{v_i^T}\ol{v_i^F}$.

To conclude, observe that the obtained outcome is exactly $\ol{\beta}$, and hence $\ol{\beta}\PrefNet{\NetCNF(\phi)}\alpha_{\sigma_X}$.

\item[$(\Leftarrow)$]
Assume that $\ol{\beta}\PrefNet{\NetCNF(\phi)}\alpha_{\sigma_X}$.
We show that there is an extension $\sigma_{X}'$ of $\sigma_X$ to $X$ satisfying $\phi(X)$.
Since $\ol{\beta}\PrefNet{\NetCNF(\phi)}\alpha_{\sigma_X}$, there is an improving flipping sequence $\rho\colon\gamma_0\ImpFlip\dots\ImpFlip\gamma_z$ from $\gamma_0=\alpha_{\sigma_X}$ to $\gamma_z=\ol{\beta}$.
Consider the truth assignment~$\sigma_{X}'$ built as follows:
If there is an index $q$ such that $\gamma_q[V_i^T V_i^F]=\ol{v_i^T}v_i^F$, then $\sigma_X'[x_i]=\valtrue$; and if there is an index $q$ such that $\gamma_q[V_i^T V_i^F]={v_i^T}\ol{v_i^F}$, then $\sigma_X'[x_i]=\valfalse$.

We first show that $\sigma_X'$ is consistent, complete, and an extension of $\sigma_X$ to $X$.
Observe that, by the definition of the CP tables of variable features, and the fact that those features have no parents, once a variable feature is flipped to its overlined value, it cannot be flipped back.
Therefore, it cannot be the case that there are indices $q$ and $r$ for which there are variable features $V_i^T$ and $V_i^F$ such that $\gamma_q[V_i^T V_i^F]=\ol{v_i^T}v_i^F$ and $\gamma_r[V_i^T V_i^F]={v_i^T}\ol{v_i^F}$.
Hence, $\sigma_X'$ is consistent.
Moreover, we claim that, for any variable $x_i$, there is always an index $q$ such that $\gamma_q[V_i^T V_i^F]=\ol{v_i^T}v_i^F$ or $\gamma_q[V_i^T V_i^F]={v_i^T}\ol{v_i^F}$, which implies that $\sigma_X'$ is complete.
Indeed, if $x_i$ has a value in $\sigma_X$, then the index $q$ that we are looking for is $q = 0$ (because either $\gamma_0[V_i^T V_i^F]=\ol{v_i^T}v_i^F$ or $\gamma_0[V_i^T V_i^F]={v_i^T}\ol{v_i^F}$, by the definition of $\gamma_0 = \alpha_{\sigma_X}$).
Observe that this, along with the consistency of $\sigma_X'$ proven above, implies that $\sigma_X'[x_i] = \sigma_X[x_i]$ for each variable $x_i$ having a truth value in $\sigma_X$.
On the other hand, if $x_i$ has not a value in $\sigma_X$, then, since $\gamma_0[V_i^T V_i^F]={v_i^T}v_i^F$ and $\gamma_z[V_i^T V_i^F]=\ol{v_i^T}\ol{v_i^F}$, it must be the case that there is an index $q > 0$ such that $\gamma_q[V_i^T V_i^F]=\ol{v_i^T}v_i^F$ or $\gamma_q[V_i^T V_i^F]={v_i^T}\ol{v_i^F}$.
Hence, $\sigma_X'$ is complete.
To conclude, by all the properties above, $\sigma_X'$ is also an extension of $\sigma_X$ to $X$.

We now prove that $\sigma_X'$ satisfies $\phi(X)$ by showing that $\sigma_X'$ satisfies all clauses of $\phi(X)$.

Let $c_j$ be a clause of $\phi(X)$, and consider clause feature $D_j$.
Because $\gamma_0[D_j] = d_j$ and $\gamma_z[D_j]=\ol{d_j}$, there is an index $t$ such that $\gamma_t[D_j]=d_j$, $\gamma_{t+1}[D_j]=\ol{d_j}$, and $\gamma_t\ImpFlipVar{D_j}\gamma_{t+1}$.
Since flipping $D_j$ has to be an improving flip, it must be the case that there is a literal feature $P_{j,k}$ such that $\gamma_t[P_{j,k}]=\ol{p_{j,k}}$.
Therefore, since $\gamma_0[P_{j,k}] = p_{j,k}$, there is an index $s<t$ such that $\gamma_s[P_{j,k}]=p_{j,k}$, $\gamma_{s+1}[P_{j,k}]=\ol{p_{j,k}}$, and $\gamma_s\ImpFlipVar{P_{j,k}}\gamma_{s+1}$.
Now there are two cases: either (a) $\ell_{j,k}=x_i$, or (b) $\ell_{j,k}=\lnot x_i$.
For (a), since flipping $P_{j,k}$ has to be an improving flip, it must be the case that $\gamma_s[V_i^T V_i^F]=\ol{v_i^T}v_i^F$, therefore $\sigma_X'[x_i]=\valtrue$, and hence $\sigma_X'$ satisfies $c_j$.
For (b), again since flipping $P_{j,k}$ has to be an improving flip, it must be the case that $\gamma_s[V_i^T V_i^F]={v_i^T}\ol{v_i^F}$, therefore $\sigma_X'[x_i]=\valfalse$ and hence $\sigma_X'$ satisfies $c_j$.

Therefore, $\sigma_X'$ satisfies all clauses of $\phi$, and hence $\sigma_X'$ is an extension of $\sigma_X$ to $X$ satisfying $\phi(X)$.
\end{itemize}

We now prove~(2).
We know that $\ol{\beta}\IncompNet{\NetCNF(\phi)}\alpha_{\sigma_X}$ if and only if $\alpha_{\sigma_X}\not\PrefNet{\NetCNF(\phi)}\ol{\beta}$ \emph{and} $\ol{\beta}\not\PrefNet{\NetCNF(\phi)}\alpha_{\sigma_X}$.
First, observe that $\alpha_{\sigma_X}\not\PrefNet{\NetCNF(\phi)}\ol{\beta}$ is always true, as there is no improving flipping sequence from $\ol{\beta}$ to $\alpha_{\sigma_X}$.
Indeed, the values of variable features in $\ol{\beta}$ cannot be flipped to their non-overlined values according to their CP tables in $\NetCNF(\phi)$, because they are the most preferred values and variable features do not have parents.
So, $\ol{\beta}\IncompNet{\NetCNF(\phi)}\alpha_{\sigma_X}$ if and only if $\ol{\beta}\not\PrefNet{\NetCNF(\phi)}\alpha_{\sigma_X}$.
Hence, showing that there is \emph{no} extension of $\sigma_X$ to $X$ satisfying $\phi(X)$ if and only if $\ol{\beta}\IncompNet{\NetCNF(\phi)}\alpha_{\sigma_X}$ is equivalent to showing that there is \emph{no} extension of $\sigma_X$ to $X$ satisfying $\phi(X)$ if and only if $\ol{\beta}\not\PrefNet{\NetCNF(\phi)}\alpha_{\sigma_X}$.
However, we have already shown this in~(1).
\end{proof}

\subsection{Proofs for Section~\ref{sec_complexity_pareto_voting}}
\phantomsection\label{page_pointer:properties_of_IsParetoOptimal_reduction_detailed}

\smallskip
\noindent\textbf{\cref{lemma_properties_of_IsParetoOptimal_reduction}.}
\emph{Let $\phi(X)$ be a $3$CNF Boolean formula, and let $\alpha$ be the outcome of $\MNetIsParOpt(\phi)$ assigning non-overlined values to all features.
Then, $\phi(X)$ is satisfiable if and only if $\alpha$ is not Pareto optimal in $\MNetIsParOpt(\phi)$.}
\begin{proof}
To prove the statement of the lemma, we first show the two following properties.

\medbreak

\noindent\textbf{Property~\ref*{lemma_properties_of_IsParetoOptimal_reduction}.(1).} \emph{If $\phi(X)$ is unsatisfiable, and $\beta\in\OutMNet{\MNetIsParOpt(\phi)}$ is an outcome such that $\beta\neq\alpha$, then $\beta\PrefNet{\NetIPO1}\alpha$ implies that $\beta\not\PrefNet{\NetIPO2}\alpha$, and $\beta\PrefNet{\NetIPO2}\alpha$ implies that $\beta\not\PrefNet{\NetIPO1}\alpha$.}

\begin{adjustwidth}{1.5em}{}
\begin{proof}
By inspection of the proof of \cref{lemma_Pref_Net_CNF_new}, since $\phi$ is unsatisfiable, there is no improving flipping sequence in $\NetIPO1$ that from $\alpha$ arrives to an outcome in which the values of \emph{all} the clause features of the net $\NetCNF(\phi)^a$ are overlined.
For this reason, by the definition of the interconnecting net $\NetInterAND(m)$, in $\NetIPO1$ there is no improving flipping sequence that from $\alpha$ arrives to an outcome in which the value of the apex of $\NetInterAND(m)$ is overlined.
This implies, moreover, that in $\NetIPO1$ there is no improving flipping sequence that from $\alpha$ arrives to an outcome in which the value of any of the features of the net $\NetCNF(\phi)^b$ is overlined.
So, any improving flipping sequence in $\NetIPO1$ from $\alpha$ arrives to outcomes in which values of features of $\NetCNF(\phi)^a$ are overlined, while values of feature of $\NetCNF(\phi)^b$ are \emph{non}-overlined.

Symmetrically, since $\phi$ is unsatisfiable, any improving flipping sequence in $\NetIPO2$ from $\alpha$ arrives to outcomes in which values of features of $\NetCNF(\phi)^a$ are \emph{non}-overlined, while values of feature of $\NetCNF(\phi)^b$ are overlined.

Now, assume that $\beta$ is such that $\beta\PrefNet{\NetIPO1}\alpha$, hence in $\NetIPO1$ there is an improving flipping sequence from $\alpha$ to $\beta$.
From what we have said, $\beta$ is such that values of features of $\NetCNF(\phi)^a$ are overlined, while values of feature of $\NetCNF(\phi)^b$ are \emph{non}-overlined.
Therefore, $\beta$ cannot be reached through an improving flipping sequence in $\NetIPO2$.
Thus, $\beta\not\PrefNet{\NetIPO2}\alpha$.

Symmetrically, it can be shown that if $\beta\PrefNet{\NetIPO2}\alpha$, then $\beta\not\PrefNet{\NetIPO1}\alpha$.
\end{proof}
\end{adjustwidth}

\medbreak

\noindent\textbf{Property~\ref*{lemma_properties_of_IsParetoOptimal_reduction}.(2).} \emph{If $\phi(X)$ is satisfiable, then outcome $\beta$ assigning overlined values to all variable and clause features of $\MNetIsParOpt(\phi)$ and to all features of $\NetInterAND(m)$ is such that $\beta\PrefNet{\NetIPO1}\alpha$ and $\beta\PrefNet{\NetIPO2}\alpha$.}

\begin{adjustwidth}{1.5em}{}
\begin{proof}
Since $\phi$ is satisfiable, by inspection of the proof of \cref{lemma_Pref_Net_CNF_new}, there is an improving flipping sequence in $\NetIPO1$ that from $\alpha$ arrives to an outcome $\alpha_1$ in which the values of \emph{all} variable and clause features of $\NetCNF(\phi)^a$ are overlined.
By definition of the interconnecting net $\NetInterAND(m)$, in $\NetIPO1$ there is also an improving flipping sequence that from $\alpha_1$ arrives to an outcome $\alpha_2$ in which the values of all features (and also of the apex) of $\NetInterAND(m)$ are overlined.
This allows $\alpha_2$ to be further improved by a flipping sequence to an outcome $\alpha_3$ in which the values of \emph{all} variable and clause features of $\NetCNF(\phi)^b$ are overlined, because $\phi$ is satisfiable.
Observe that $\alpha_3 = \beta$.

Symmetrically, it can be shown that in $\NetIPO2$ there exists an improving flipping sequence from $\alpha$ to the very same~$\beta$.
Therefore, $\beta\PrefNet{\NetIPO1}\alpha$ and $\beta\PrefNet{\NetIPO2}\alpha$.
\end{proof}
\end{adjustwidth}

\medbreak

\noindent
We now show that $\phi(X)$ is satisfiable if and only if $\alpha$ is not a Pareto optimal outcome of $\MNetIsParOpt(\phi)$.

\begin{itemize}
\item[$(\Rightarrow)$]
Assume that $\phi(X)$ is satisfiable.
Then, by Property~\ref*{lemma_properties_of_IsParetoOptimal_reduction}.(2), there is an outcome $\beta$ that is preferred to $\alpha$ by all agents of $\MNetIsParOpt(\phi)$.
Therefore, $\beta\PrefParetoMNet{\MNetIsParOpt(\phi)}\alpha$, and hence $\alpha$ is not Pareto optimal in $\MNetIsParOpt(\phi)$.

\item[$(\Leftarrow)$]
Assume that $\phi(X)$ is not satisfiable.
Then, by Property~\ref*{lemma_properties_of_IsParetoOptimal_reduction}.(1), there is no outcome $\beta$ that is preferred to $\alpha$ by all agents of $\MNetIsParOpt(\phi)$.
Therefore, $\alpha$ is Pareto optimal in $\MNetIsParOpt(\phi)$.\qedhere
\end{itemize}
\end{proof}

\subsection{Proofs for Section \ref{sec_complexity_majority_voting}}
\phantomsection\label{page_pointer:Pref_Net_CNF_Summ_new_detailed}

\noindent\textbf{\cref{lemma_Pref_Net_CNF_Summ_new}.}
\emph{Let $\phi(X)$ be a Boolean formula in $3$CNF defined over a set $X$ of Boolean variables, and let $\sigma_X$ be an assignment on $X$.
	Let $\alpha_{\sigma_X}$ be the outcome of $\NetCNFSumm(\phi)$ encoding $\sigma_X$ on the feature set $\SetFeat{V}$, and assigning non-overlined values to all other features.
	Let $\ol{\beta}$ be an outcome of $\NetCNFSumm(\phi)$ such that $\ol{\beta}[U_1 U_2] = \ol{u_1}\ol{u_2}$, assigning any value to the features of $\SetFeat{V}$, and assigning non-overlined values to all other features.
	Then:
	\begin{itemize}
		\item[(1)] There is an extension of $\sigma_X$ to $X$ satisfying $\phi(X)$ if and only if $\ol{\beta}\PrefNet{\NetCNFSumm(\phi)}\alpha_{\sigma_X}$;
		\item[(2)] There is no extension of $\sigma_X$ to $X$ satisfying $\phi(X)$ if and only if $\ol{\beta}\IncompNet{\NetCNFSumm(\phi)}\alpha_{\sigma_X}$.
	\end{itemize}}
	
	\begin{proof}
		We first prove~(1).
		\begin{itemize}
			\item[$(\Rightarrow)$]
			Assume that there is an extension $\sigma_X'$ of $\sigma_X$ to $X$ satisfying $\phi(X)$.
			We show that $\ol{\beta}\PrefNet{\NetCNFSumm(\phi)}\alpha_{\sigma_X}$ by exhibiting an improving flipping sequence from $\alpha_{\sigma_X}$ to $\ol{\beta}$.
			
			We claim that, by \cref{lemma_Pref_Net_CNF_new}, outcome $\beta$ (which is different from $\ol{\beta}$) assigning overlined values to all variable and clause features is such that $\beta\PrefNet{\NetCNFSumm(\phi)}\alpha_{\sigma_X}$.
			Indeed, features in $\SetFeat{V}\cup\SetFeat{P}\cup\SetFeat{D}$ are linked in $\NetCNFSumm(\phi)$ through the very same links of a net $\NetCNF(\phi)$.
			Moreover, the value $u_1$ assigned to feature $U_1$ in $\alpha_{\sigma_X}$ selects in the CP tables of features $\SetFeat{V}\cup\SetFeat{P}\cup\SetFeat{D}$ specific preference rankings that are equivalent to those in the CP tables of features $\SetFeat{V}\cup\SetFeat{P}\cup\SetFeat{D}$ in $\NetCNF(\phi)$.
			Now, since $\beta\PrefNet{\NetCNFSumm(\phi)}\alpha_{\sigma_X}$, there is an improving flipping sequence from $\alpha_{\sigma_X}$ to $\beta$.
			Then, we can flip the values of all features of the interconnecting net $\NetInterAND(m)$ (including the apex) and of $U_2$.
			Now we flip the value of feature $U_1$ from $u_1$ to $\ol{u_1}$.
			Recall that in $\beta$, all variable features have overlined values, and their values have not been flipped after outcome $\beta$ was reached in the improving flipping sequence.
			Therefore, since the value of $U_1$ is $\ol{u_1}$, given the CP table of variable features, we can flip features in $\SetFeat{V}$ to any configuration of values (even leaving everything as it is), and in particular we can flip them to match the values of features of $\SetFeat{V}$ in $\ol{\beta}$ .
			Next, all literal features can be flipped to their non-overlined values (recall that $U_1$ has value $\ol{u_1}$ now).
			Then, we flip all clause features to their non-overlined values, and after this, in the proper order, we can flip the features of the interconnecting net to their non-overlined values.
			Observe that the obtained outcome is precisely $\ol{\beta}$, and hence $\ol{\beta}\PrefNet{\NetCNFSumm(\phi)}\alpha_{\sigma_X}$.
			
			\item[$(\Leftarrow)$]
			Assume that $\ol{\beta}\PrefNet{\NetCNFSumm(\phi)}\alpha_{\sigma_X}$.
			We show that there is an extension $\sigma_X'$ of $\sigma_X$ to $X$ satisfying $\phi(X)$.
			Since $\ol{\beta}\PrefNet{\NetCNFSumm(\phi)}\alpha_{\sigma_X}$, there is an improving flipping sequence $\rho\colon\gamma_0\ImpFlip\dots\ImpFlip\gamma_z$ from $\gamma_0=\alpha_{\sigma_X}$ to $\gamma_z=\ol{\beta}$.
			
			Since $\gamma_0[U_1] = u_1$ and $\gamma_z[U_1] = \ol{u_1}$, there must be an index $s$ in which $\gamma_s[U_1]=u_1$, $\gamma_{s+1}[U_1]=\ol{u_1}$, and $\gamma_s\ImpFlipVar{U_1}\gamma_{s+1}$.
			Moreover, because $U_1$ has no parents, in the sequence $\rho$, feature $U_1$ can be flipped only once.
			Therefore, for all $p < s$, $\alpha_p[U_1] = u_1$, and for all $p > s$, $\alpha_p[U_1] = \ol{u_1}$.
			
			We claim that we can assume \Wlog that all variable features have overlined values in $\gamma_s$.
			Indeed, if this is not the case, we can always modify as follows an improving flipping sequence $\rho$, from $\alpha_{\sigma_X}$ to $\ol{\beta}$, to obtain an improving flipping sequence $\rho'$, from $\alpha_{\sigma_X}$ to $\ol{\beta}$, satisfying the required assumption.
			In particular, consider all variable features $F$ having non-overlined values just before feature $U_1$ is flipped.
			We can flip all of them before flipping $U_1$ to $\ol{u_1}$, and, after having flipped $U_1$, we can flip them all back to the values they had before.
			Clearly, the new sequence satisfies the required assumption, and, moreover, it is still improving, and it is still a sequence from $\alpha_{\sigma_X}$ to $\ol{\beta}$.
			
			Consider the truth assignment $\sigma_X'$ built as follows:
			If there is an index $p < s$ such that $\gamma_p[V_i^T V_i^F]=\ol{v_i^T}v_i^F$, then $\sigma_X'[x_i]=\valtrue$; if there is an index $p < s$ such that $\gamma_p[V_i^T V_i^F]={v_i^T}\ol{v_i^F}$, then $\sigma_X'[x_i]=\valfalse$.
			
			We first show that $\sigma_X'$ is consistent, complete, and an extension of $\sigma_X$ to $X$.
			Observe that, by the definition of the CP tables of variable features, before the $s$-th step of the sequence, once a variable feature is flipped to its overlined value, it cannot be flipped back (this may happen only after the $s$-th step).
			Therefore, it cannot be the case that there are indices $p < s$ and $q < s$ for which there are variable features $V_i^T$ and $V_i^F$ such that $\gamma_p[V_i^T V_i^F]=\ol{v_i^T}v_i^F$ and $\gamma_q[V_i^T V_i^F]={v_i^T}\ol{v_i^F}$.
			Hence, $\sigma_X'$ is consistent.
			Moreover, we claim that, for any variable $x_i$, there is always an index $p < s$ such that $\gamma_p[V_i^T V_i^F]=\ol{v_i^T}v_i^F$ or $\gamma_p[V_i^T V_i^F]={v_i^T}\ol{v_i^F}$, which implies that $\sigma_X'$ is complete.
			Indeed, if $x_i$ has a value in $\sigma_X$, then the index $p$ that we are looking for is $p = 0$ (because either $\gamma_0[V_i^T V_i^F]=\ol{v_i^T}v_i^F$ or $\gamma_0[V_i^T V_i^F]={v_i^T}\ol{v_i^F}$, by the definition of $\gamma_0 = \alpha_{\sigma_X}$).
			Observe that this and the consistency of $\sigma_X'$ proven above imply that $\sigma_X'[x_i] = \sigma_x[x_i]$ for each variable $x_i$ having a truth value in $\sigma_X$.
			On the other hand, if $x_i$ has not a value in $\sigma_X$, then, since $\gamma_0[V_i^T V_i^F]={v_i^T}v_i^F$, and we are assuming that $\gamma_s[V_i^T V_i^F]=\ol{v_i^T}\ol{v_i^F}$, it must be the case that there is an index $p < s$ such that $\gamma_p[V_i^T V_i^F]=\ol{v_i^T}v_i^F$ or $\gamma_p[V_i^T V_i^F]={v_i^T}\ol{v_i^F}$.
			Hence, $\sigma_X'$ is complete.
			To conclude, by all the properties above, $\sigma_X'$ is an extension of $\sigma_X$ to $X$.
			
			We now show that $\sigma_X'$ satisfies $\phi(X)$ by showing that $\sigma_X'$ satisfies all the clauses of $\phi(X)$.
			
			Because $\gamma_0[U_2] = u_2$ and $\gamma_z[U_2]=\ol{u_2}$, there must be an index $t$ such that $\gamma_t[U_2]=u_2$, $\gamma_{t+1}[U_2]=\ol{u_2}$, and $\gamma_t\ImpFlipVar{U_2}\gamma_{t+1}$.
			(Indices $s$ and $t$ are not in any particular relationship; it could be $s < t$ but also $t < s$.)
			By the definition of the CP tables,  for $\gamma_t\ImpFlipVar{U_2}\gamma_{t+1}$ to be an improving flip, it must be the case that, for the apex $A$ of the interconnecting net, $\gamma_t[A]=\ol{a}$.
			This requires that there is an index $r < t$, for which in $\gamma_r$ all clause features have their overlined values (because the conjunctive interconnecting net is linked to the set of clause features).
			
			Let $c_j$ be any clause of $\phi(X)$, and consider feature $D_j$.
			Because $\gamma_0[D_j] = d_j$ and $\gamma_r[D_j]=\ol{d_j}$, there is an index $q < r$ such that $\gamma_q[D_j]=d_j$, $\gamma_{q+1}[D_j]=\ol{d_j}$, and $\gamma_q\ImpFlipVar{D_j}\gamma_{q+1}$.
			Since flipping $D_j$ has to be an improving flip, it must be the case that there is a literal feature $P_{j,k}$ such that $\gamma_q[P_{j,k}]=\ol{p_{j,k}}$.
			Therefore, since $\gamma_0[P_{j,k}] = p_{j,k}$ and $\gamma_q[P_{j,k}]=\ol{p_{j,k}}$, there is an index $p < q$ such that $\gamma_p[P_{j,k}]=p_{j,k}$, $\gamma_{p+1}[P_{j,k}]=\ol{p_{j,k}}$, and $\gamma_p\ImpFlipVar{P_{j,k}}\gamma_{p+1}$.
			By the definition of the CP table of $P_{j,k}$, it must also be the case that $p < s$.
			This means that, for all the literal features that in $\rho$ change their value from non-overlined to overlined, their flipping happens before the $s$-th step.
			Hence, if two different literal features linked to the same variable features $V_i^T$ and $V_i^F$ flip before the $s$-th step, then their flipping is based on consistent values assigned to $V_i^T$ and $V_i^F$.
			
			Now there are two cases: either (a) $\ell_{j,k}=x_i$, or (b) $\ell_{j,k}=\lnot x_i$.
			For (a), since flipping $P_{j,k}$ has to be an improving flip, it must be the case that $\gamma_p[V_i^T V_i^F]=\ol{v_i^T}v_i^F$, therefore $\sigma_X'[x_i]=\valtrue$, and hence $\sigma_X'$ satisfies $c_j$.
			For (b), again since flipping $P_{j,k}$ has to be an improving flip, it must be the case that $\gamma_p[V_i^T V_i^F]={v_i^T}\ol{v_i^F}$, therefore $\sigma_X'[x_i]=\valfalse$, and hence $\sigma_X'$ satisfies $c_j$.
			
			Therefore, if $\ol{\beta}\PrefNet{\NetCNFSumm(\phi)}\alpha_{\sigma_X}$, then there is an extension $\sigma_X'$ of $\sigma_X$ to $X$ satisfying $\phi(X)$.
		\end{itemize}
		
		We now prove~(2).
		We know that $\ol{\beta}\IncompNet{\NetCNFSumm(\phi)}\alpha_{\sigma_X}$ if and only if $\alpha_{\sigma_X}\not\PrefNet{\NetCNFSumm(\phi)}\ol{\beta}$ \emph{and} $\ol{\beta}\not\PrefNet{\NetCNFSumm(\phi)}\alpha_{\sigma_X}$.
		First, observe that $\alpha_{\sigma_X}\not\PrefNet{\NetCNFSumm(\phi)}\ol{\beta}$ is always true, because there is no improving flipping sequence from $\ol{\beta}$ to $\alpha_{\sigma_X}$, as the value $\ol{u_1}$ of feature $U_1$ in $\ol{\beta}$ cannot be flipped in $\NetCNFSumm(\phi)$, because it is the most preferred value of $U_1$, and $U_1$ does not have parents.
		So, $\ol{\beta}\IncompNet{\NetCNFSumm(\phi)}\alpha_{\sigma_X}$ if and only if $\ol{\beta}\not\PrefNet{\NetCNFSumm(\phi)}\alpha_{\sigma_X}$.
		Hence, showing that there is \emph{no} extension of $\sigma_X$ to $X$ satisfying $\phi(X)$ if and only if $\ol{\beta}\IncompNet{\NetCNFSumm(\phi)}\alpha_{\sigma_X}$ is equivalent to showing that there is \emph{no} extension of $\sigma_X$ to $X$ satisfying $\phi(X)$ if and only if $\ol{\beta}\not\PrefNet{\NetCNFSumm(\phi)}\alpha_{\sigma_X}$.
		However, we have already shown this in~(1).
	\end{proof}

\smallskip \phantomsection\label{page_pointer:properties_of_ExistsWeakCondorcet_reduction_detailed}
\noindent\textbf{\cref{lemma_properties_of_ExistsWeakCondorcet_reduction}.}
\emph{Let $\Phi=(\exists X)(\forall Y)\lnot\phi(X,Y)$ be a quantified Boolean formula, where $\phi(X,Y)$ is a $3$CNF Boolean formula, defined over two disjoint sets $X$ and $Y$ of Boolean variables.
Then, $\Phi$ is valid if and only if $\MNetExistsWeakCond(\Phi)$ has a majority optimal outcome.}

\begin{proof}
The intuition at the base of the proof is to put in relationship truth assignments over the variable set $X$ with outcomes of $\MNetExistsWeakCond(\Phi)$.
In particular, given an assignment $\sigma_X$ over $X$, the associated outcome is $\beta_{\sigma_X}$, where $\sigma_X$ is encoded over the feature set $\SetFeat{V}$ in the usual way, and all other features have non-overlined values.
We show first that outcomes not in the form of a $\beta_{\sigma_X}$ are not Majority optimal.
Then, we show that if there is an assignment $\sigma_X$ such that $(\forall Y)\lnot\phi(X/\sigma_X,Y)$ is valid, then $\beta_{\sigma_X}$ is Majority optimal.
On the other hand, if there is no assignment $\sigma_X$ for which $(\forall Y)\lnot\phi(X/\sigma_X,Y)$ is valid, then none of the outcomes $\beta_{\sigma_X}$ is Majority optimal.

To prove the statement of the lemma, we have to analyze the majority dominance relationships between outcome pairs.
To organize this task, we define the following sets of outcomes:
\begin{itemize}
\item $O_d={O_d}'\cup {O_d}''\cup {O_d}'''$, where
\begin{itemize}
\item ${O_d}'=\{\beta\in\OutMNet{\MNetExistsWeakCond(\Phi)}\mid (\exists F)(F\in(\SetFeat{V'}\cup\SetFeat{W}\cup\SetFeat{P}\cup\SetFeat{D}\cup\SetFeat{A}\cup\SetFeat{B})\land \beta[F]=\ol{f})\}$;
\item ${O_d}''=\{\beta\in\OutMNet{\MNetExistsWeakCond(\Phi)}\mid \beta[U_1 U_2]\neq u_1 u_2\}$;
\item ${O_d}'''=\{\beta\in\OutMNet{\MNetExistsWeakCond(\Phi)}\mid (\exists i)(\beta[V_i^T V_i^F]=\ol{v_i^T}\ol{v_i^F})\}$.
\end{itemize}
\item $O_c=\{\beta\in\OutMNet{\MNetExistsWeakCond(\Phi)}\mid \beta\notin O_d\}$.
\end{itemize}
Clearly, $O_d$ and $O_c$ constitute a partition of $\OutMNet{\MNetExistsWeakCond(\Phi)}$.
On the contrary, ${O_d}'$, ${O_d}''$, and ${O_d}'''$, do not constitute a partition of $O_d$, because they are not disjoint.
We show that only outcomes of a subset $S$ (whose detailed characterization will be given toward the end of the proof) of $O_c$ might be Majority optimal.
We do so by showing that (1) all outcomes in $(O_d\cup O_c)\setminus S$ are majority dominated by some other outcome, and hence they are not Majority optimal; and that (2) all outcomes in $S$, which might be empty, are not majority dominated, and hence they are Majority optimal.
Therefore, $\MNetExistsWeakCond(\Phi)$ has a majority optimal outcome if and only if $S$ is non-empty.

We recall that, since $\MNetExistsWeakCond(\Phi)$ is a $6$\CPNet, if $\alpha$ and $\beta$ are two outcomes, $|\SetAgentsPrefMNet{\MNetExistsWeakCond(\Phi)}(\beta,\alpha)|\geq 4$ implies that $\beta\PrefMajorityMNet{\MNetExistsWeakCond(\Phi)}\alpha$.

\medbreak

\noindent\textbf{Property~\ref*{lemma_properties_of_ExistsWeakCondorcet_reduction}.(1).} \emph{Let $\beta'\in {O_d}'$ be an outcome.
Then, $\beta'$ is not Majority optimal in $\MNetExistsWeakCond(\Phi)$.}

\begin{adjustwidth}{1.5em}{}
\begin{proof}
There are two cases: either (1) $\beta'[U_1 U_2] \neq u_1 u_2$, or (2) $\beta'[U_1 U_2] = u_1 u_2$:

\medskip \noindent (1)
Let $\gamma$ be the outcome assigning non-overlined values to all features in $(\SetFeat{V}\cup\SetFeat{V'}\cup\SetFeat{W}\cup\SetFeat{P}\cup\SetFeat{D}\cup\SetFeat{A}\cup\SetFeat{B})$, and such that $\gamma[U_1 U_2] = \beta'[U_1 U_2]$.
We prove that $\gamma\PrefMajorityMNet{\MNetExistsWeakCond(\Phi)}\beta'$.
Consider net $\NetEWC3$.
The following is an improving flipping sequence from $\beta'$ to $\gamma$, showing that $\gamma\PrefNet{\NetEWC3}\beta'$.
We flip to their non-overlined value, in the following order, all features in $\SetFeat{V}$, $\SetFeat{V'}$, $\SetFeat{W}$, $\SetFeat{P}$, $\SetFeat{D}$, and $\SetFeat{A}$, having an overlined value in $\beta'$.
Next, we flip in the proper order all features in $\SetFeat{B}$ having an overlined value in $\beta'$ to their non-overlined value.
The outcome obtained is precisely $\gamma$.
Moreover, since $\NetEWC4 = \NetEWC3$, $\gamma\PrefNet{\NetEWC4}\beta'$, as well.
Now, observe that by the definition of $\NetEWC5$, $\gamma\PrefNet{\NetEWC5}\beta'$.
We have seen that there are three agents preferring $\gamma$ to $\beta'$.

Now there are two cases: either (a) $\beta'[U_1] = \ol{u_1}$, or (b) $\beta'[U_1] = u_1$:

(a)
Let us  focus on $\NetEWC1$.
The following is an improving flipping sequence from $\beta'$ to $\gamma$ in $\NetEWC1$, showing that $\gamma\PrefNet{\NetEWC1}\beta'$.
Since $\beta'[U_1] = \ol{u_1}$, we can flip to their non-overlined value all features in $\SetFeat{V}$ and $\SetFeat{W}$ having an overlined value in $\beta'$.
Next, we flip to their non-overlined value all features in $\SetFeat{P}$, $\SetFeat{D}$, $\SetFeat{A}$ (in the proper order), $\SetFeat{V'}$, and $\SetFeat{B}$, having an overlined value in $\beta'$.
The outcome reached is exactly $\gamma$.
Therefore, in~(1)(a), there are four agents preferring $\gamma$ to $\beta'$, and hence $\gamma\PrefMajorityMNet{\MNetExistsWeakCond(\Phi)}\beta'$.

(b)
As we are in~(1)(b), $\beta'[U_1] = u_1$ and $\beta'[U_1 U_2] \neq u_1 u_2$, hence $\beta'[U_1 U_2] = u_1 \ol{u_2}$.
Let us focus on $\NetEWC2$.
The improving flipping sequence from $\beta'$ to $\gamma$ in $\NetEWC1$  in~(1)(a) is also an improving flipping sequence from $\beta'$ to $\gamma$ in $\NetEWC2$ (because $\beta'[U_2] = \ol{u_2}$).
Therefore, in~(1)(b), there are four agents preferring~$\gamma$ to $\beta'$, and hence $\gamma\PrefMajorityMNet{\MNetExistsWeakCond(\Phi)}\beta'$.

\medskip \noindent (2)
Let $\gamma'$ be the outcome assigning an overlined value only to $U_1$.
We prove that $\gamma'\PrefMajorityMNet{\MNetExistsWeakCond(\Phi)}\beta'$.
Consider net $\NetEWC1$.
By the definition of the net, we can flip $U_1$ from $u_1$ to $\ol{u_1}$.
Once this is done, by performing the improving flipping sequence from $\beta'$ to $\gamma$ in $\NetEWC1$ exhibited in~(1)(a), we reach $\gamma'$ in this case.
Hence, $\gamma'\PrefNet{\NetEWC1}\beta'$.
Consider now net $\NetEWC3$.
Since $\beta'\in {O_d}'$, there is a feature $F\in(\SetFeat{V'}\cup\SetFeat{W}\cup\SetFeat{P}\cup\SetFeat{D}\cup\SetFeat{A}\cup\SetFeat{B})$ such that $\beta'[F] = \ol{f}$.
By the definition of the disjunctive interconnecting net embedded in $\NetEWC3$, there is an improving flipping sequence from $\beta'$ to an outcome in which $U_1$ has an overlined value (we flip the features in the interconnecting net until we flip its apex, and then we flip $U_1$ to its overlined value).
At this point, by performing the improving flipping sequence from $\beta'$ to $\gamma$ in $\NetEWC3$ shown in~(1) we reach $\gamma'$ in this case.
Hence, $\gamma'\PrefNet{\NetEWC3}\beta'$.
Since $\NetEWC4 = \NetEWC3$, $\gamma'\PrefNet{\NetEWC4}\beta'$, as well.
Finally, by the definition of $\NetEWC5$, $\gamma'\PrefNet{\NetEWC5}\beta'$.
Therefore, $\gamma'\PrefMajorityMNet{\MNetExistsWeakCond(\Phi)}\beta'$, and hence $\beta'$ is not Majority optimal.
\end{proof}
\end{adjustwidth}

\medbreak

\noindent\textbf{Property~\ref*{lemma_properties_of_ExistsWeakCondorcet_reduction}.(2).} \emph{Let $\beta''\in {O_d}''$ be an outcome.
Then, $\beta''$ is not Majority optimal in $\MNetExistsWeakCond(\Phi)$.}

\begin{adjustwidth}{1.5em}{}
\begin{proof}
By Property~\ref*{lemma_properties_of_ExistsWeakCondorcet_reduction}.(1) we can focus on those outcomes $\beta''$ assigning non-overlined values to all features in $\SetFeat{V'}\cup\SetFeat{W}\cup\SetFeat{P}\cup\SetFeat{D}\cup\SetFeat{A}\cup\SetFeat{B}$.
There are three cases: (1) $\beta'[U_1 U_2] = \ol{u_1}u_2$, (2) $\beta'[U_1 U_2] = u_1\ol{u_2}$, or (3) $\beta'[U_1 U_2] = \ol{u_1}\ol{u_2}$:

\medskip\noindent~(1)
Let $\gamma$ be the outcome such that, for all features $F\notin\{U_1,U_2\}$, $\gamma[F] = \beta''[F]$, and $\gamma[U_1 U_2] = \ol{u_1}\ol{u_2} \neq \beta''[U_1 U_2]$.
By the definition of $\NetEWC2$, $\NetEWC3$, $\NetEWC4$, and $\NetEWC5$, $\gamma\PrefNet{\NetEWC2}\beta''$, $\gamma\PrefNet{\NetEWC3}\beta''$, $\gamma\PrefNet{\NetEWC4}\beta''$, $\gamma\PrefNet{\NetEWC5}\beta''$, respectively.
Therefore, $\gamma\PrefMajorityMNet{\MNetExistsWeakCond(\Phi)}\beta''$, and $\beta''$ is not Majority optimal.

\medskip\noindent~(2)
Let $\gamma$ be the outcome such that, for all features $F\notin\{U_1,U_2\}$, $\gamma[F] = \beta''[F]$, and $\gamma[U_1 U_2] = u_1 u_2 \neq \beta''[U_1 U_2]$.
We show that $\gamma\PrefMajorityMNet{\MNetExistsWeakCond(\Phi)}\beta''$.
Consider net $\NetEWC1$.
Since we are assuming that all features in $\SetFeat{A}$ have non-overlined value in $\beta''$, we can flip $U_2$ from $\ol{u_2}$ to $u_2$.
Hence, $\gamma\PrefNet{\NetEWC1}\beta''$.
By the definition of $\NetEWC3$, $\NetEWC4$, and $\NetEWC5$, $\gamma\PrefNet{\NetEWC3}\beta''$, $\gamma\PrefNet{\NetEWC4}\beta''$, $\gamma\PrefNet{\NetEWC5}\beta''$, respectively.
Therefore, $\gamma\PrefMajorityMNet{\MNetExistsWeakCond(\Phi)}\beta''$, and $\beta''$ is not Majority optimal.

\medskip\noindent~(3)
Let $\gamma$ be the outcome such that, for all features $F\notin\{U_1,U_2\}$, $\gamma[F] = \beta''[F]$, and $\gamma[U_1 U_2] = u_1 \ol{u_2} \neq \beta''[U_1 U_2]$.
We show that $\gamma\PrefMajorityMNet{\MNetExistsWeakCond(\Phi)}\beta''$.
Consider net $\NetEWC2$.
Since we are assuming that all features in $\SetFeat{A}$ have non-overlined value in $\beta''$, we can flip $U_1$ from $\ol{u_1}$ to $u_1$.
Hence, $\gamma\PrefNet{\NetEWC2}\beta''$.
Moreover, we are assuming also that all features in $\SetFeat{B}$ have non-overlined value in $\beta''$, which implies that in $\NetEWC3$ we can flip $U_1$ from $\ol{u_1}$ to $u_1$.
Therefore, $\gamma\PrefNet{\NetEWC3}\beta''$.
Since $\NetEWC3 = \NetEWC4$, $\gamma\PrefNet{\NetEWC4}\beta''$, as well.
Finally, by the definition of $\NetEWC6$, $\gamma\PrefNet{\NetEWC6}\beta''$ (as we are assuming $\beta''[U_2] = \ol{u_2}$).
Thus, $\gamma\PrefMajorityMNet{\MNetExistsWeakCond(\Phi)}\beta''$, and $\beta''$ is not Majority optimal.
\end{proof}
\end{adjustwidth}

\medbreak

\noindent\textbf{Property~\ref*{lemma_properties_of_ExistsWeakCondorcet_reduction}.(3).} \emph{Let $\beta'''\in {O_d}'''$ be an outcome.
Then, $\beta'''$ is not Majority optimal in $\MNetExistsWeakCond(\Phi)$.}

\begin{adjustwidth}{1.5em}{}
\begin{proof}
By Properties~\ref*{lemma_properties_of_ExistsWeakCondorcet_reduction}.(1) and~\ref*{lemma_properties_of_ExistsWeakCondorcet_reduction}.(2), we can focus on those outcomes $\beta'''$ assigning non-overlined values to all features in $\SetFeat{V'}\cup\SetFeat{W}\cup\SetFeat{P}\cup\SetFeat{D}\cup\SetFeat{A}\cup\SetFeat{B}$, and such that $\beta'''[U_1 U_2] = u_1 u_2$.
Since $\beta'''\in {O_d}'''$, there is a pair of features $(V_i^T, V_i^F)$ such that $\beta'''[V_i^T V_i^F] = \ol{v_i^T}\ol{v_i^F}$.
Let $\gamma$ be the outcome such that, for all features $F\notin\{V_i^T,V_i^F,U_1,U_2\}$, $\gamma[F] = \beta'''[F]$, $\gamma[V_i^T V_i^F] = v_i^T v_i^F \neq \beta'''[V_i^T V_i^F]$, and $\gamma[U_1 U_2] = \ol{u_1} u_2 \neq \beta'''[U_1 U_2]$.
We show that $\gamma\PrefMajorityMNet{\MNetExistsWeakCond(\Phi)}\beta'''$.

Consider net $\NetEWC1$.
We can flip feature $U_1$ from $u_1$ to $\ol{u_1}$.
Then, we can flip features $V_i^T$ and $V_i^F$ to their non-overlined value.
The reached outcome is precisely $\gamma$, and hence $\gamma\PrefNet{\NetEWC1}\beta'''$.

Consider now net $\NetEWC3$.
Since $\beta'''[V_i^T V_i^F] = \ol{v_i^T}\ol{v_i^F}$, we can flip $V_i'$ from $v_i'$ to $\ol{v_i'}$.
Then, we flip $V_i^T$ and $V_i^F$ to their non-overlined value.
At this point, since $\beta'''[U_1 U_2] = u_1 u_2$, the improving flipping sequence in $\NetEWC3$ exhibited in Case~(2) of Property~\ref*{lemma_properties_of_ExistsWeakCondorcet_reduction}.(1), in which the flips of features in $\SetFeat{V}$ are ignored, is an improving flipping sequence from $\beta'''$ to $\gamma$ in this case.
Hence, $\gamma\PrefNet{\NetEWC3}\beta'''$.
Since $\NetEWC3 = \NetEWC4$, $\gamma\PrefNet{\NetEWC4}\beta'''$, as well.

Finally, by the definition of $\NetEWC5$, $\gamma\PrefNet{\NetEWC5}\beta'''$.
Therefore, $\gamma\PrefMajorityMNet{\MNetExistsWeakCond(\Phi)}\beta'''$.
\end{proof}
\end{adjustwidth}

\medbreak

\noindent
The three properties above show that any outcome $\beta\in O_d$ is not Majority optimal.
Consider now outcomes in $O_c$.
Observe that, since $O_d$ and $O_c$ are disjoint, all outcomes $\beta_c\in O_c$ are characterized by having all features in $\SetFeat{V'}\cup\SetFeat{W}\cup\SetFeat{P}\cup\SetFeat{D}\cup\SetFeat{A}\cup\SetFeat{B}$ with non-overlined values, $\beta_c[U_1 U_2] = u_1 u_2$, and, for all variables $x_i\in X$, $\beta_c[V_i^T V_i^F] \neq \ol{v_i^T} \ol{v_i^F}$.

Given an assignment $\sigma_X$ for the variables in $X$, we define $\beta_{\sigma_X}\in O_c$ as the outcome encoding $\sigma_X$ over the feature set~$\SetFeat{V}$ as usual.
Denote by $\NonWitnessSet_c$ the set of all \emph{complete} assignments $\sigma_X$ over $X$ such that $(\forall Y)\lnot\phi(X/\sigma_X,Y)$ is \emph{not} valid (i.e., such that $\phi(X/\sigma_X,Y)$ is satisfiable).
Let $\NonWitnessSet$ be the set of all (partial or complete) assignments $\sigma_X$ over $X$ such that there is an extension of $\sigma_X$ to $X$ belonging to $\NonWitnessSet_c$, and let $\WitnessSet$ be the set of all (partial or complete) assignments over $X$ not belonging to $\NonWitnessSet$.
Recall that if $\sigma_X$ is a complete assignment over $X$, then $\sigma_X$ itself is the unique extension of $\sigma_X$ to $X$.
Given the above definitions, $O_c^{\WitnessSet}=\{\beta_{\sigma_X}\in O_c\mid \sigma_X\in\WitnessSet\}$ and $O_c^{\NonWitnessSet}=\{\beta_{\sigma_X}\in O_c\mid \sigma_X\in\NonWitnessSet\}$ constitute a partition of $O_c$.
This implies that, if $\beta\in O_c$ is an outcome, then there is a (partial or complete) assignment $\sigma_X$ over the variables in $X$ such that $\beta = \beta_{\sigma_X}$.

We show that all outcomes in $O_c^{\NonWitnessSet}$ are not Majority optimal, and that all outcomes in $O_c^{\WitnessSet}$ are majority optimal (i.e., $O_c^{\WitnessSet}$ is the set $S$ mentioned earlier).

\medbreak

\noindent\textbf{Property~\ref*{lemma_properties_of_ExistsWeakCondorcet_reduction}.(4).} \emph{Let $\beta_c\in O_c^{\NonWitnessSet}$ be an outcome.
Then, $\beta_c$ is not Majority optimal in $\MNetExistsWeakCond(\Phi)$.}

\begin{adjustwidth}{1.5em}{}
\begin{proof}
Let $\ol{\alpha}$ be the outcome assigning overlined values only to $U_1$ and $U_2$.
We show that $\ol{\alpha}\PrefMajorityMNet{\MNetExistsWeakCond(\Phi)}\beta_c$.
Let $\sigma_X\in\NonWitnessSet$ be the (partial or complete) assignment over $X$ such that $\beta_c = \beta_{\sigma_X}$.

Consider the net $\NetEWC1$.
First, note that, for any outcome $\beta_c\in O_c$, $\beta_c$ assigns non-overlined values to all features in $\SetFeat{V'}\cup\SetFeat{B}$.
Therefore, the part of net $\NetEWC1$ over feature sets $\SetFeat{V'}$ and $\SetFeat{B}$ does not play an active role in any improving flipping sequence (if exists) either from $\ol{\alpha}$ to $\beta_c$, or from $\beta_c$ to $\ol{\alpha}$, because in $\NetEWC1$, features in $\SetFeat{V'}\cup\SetFeat{B}$ have no parents, and they have already their most preferred values in $\ol{\alpha}$ and $\beta_c$.

Consider now the non-quantified formula $\phi(X,Y)$.
If we consider the set $X\cup Y$ of all the Boolean variables in~$\phi$, the assignment $\sigma_X$ is a partial assignment over $X\cup Y$.
Since $\sigma_X\in\NonWitnessSet$, there is an extension $\sigma_X'$ of~$\sigma_X$ to $X$ such that $(\forall Y)\lnot\phi(X/\sigma_X',Y)$ is not valid (i.e., $\phi(X/\sigma_X',Y)$ is satisfiable), and hence there is an extension of~$\sigma_X$ to $X\cup Y$ satisfying $\phi$.
Therefore, by \cref{lemma_Pref_Net_CNF_Summ_new}, $\ol{\alpha}\PrefNet{\NetEWC1}\beta_c$.
Observe that the differences between $\NetEWC1$ and~$\NetEWC2$ are only in the roles of features $U_1$ and $U_2$, which are exchanged.
Hence, by (an adaptation of) \cref{lemma_Pref_Net_CNF_Summ_new}, $\ol{\alpha}\PrefNet{\NetEWC2}\beta_c$, as well.

Consider now $\NetEWC5$.
The following is an improving flipping sequence from $\beta_c$ to $\ol{\alpha}$.
We flip $U_1$ from $u_1$ to $\ol{u_1}$, and then $U_2$ from $u_2$ to $\ol{u_2}$.
After this, we flip to their non-overlined value all features in $\SetFeat{V}$ having overlined values in $\beta_c$.
The outcome reached is precisely $\ol{\alpha}$, and hence $\ol{\alpha}\PrefNet{\NetEWC5}\beta_c$.

Finally, consider net $\NetEWC6$.
The following is an improving flipping sequence from $\beta_c$ to $\ol{\alpha}$.
We flip $U_1$ from $u_1$ to $\ol{u_1}$ (remember that $\beta_c[U_2] = u_2$).
Next, we flip $U_2$ from $u_2$ to $\ol{u_2}$.
To conclude, we flip to their non-overlined value all features in $\SetFeat{V}$ having overlined values in $\beta_c$.
The outcome reached is precisely $\ol{\alpha}$, and hence $\ol{\alpha}\PrefNet{\NetEWC6}\beta_c$.
Therefore, $\ol{\alpha}\PrefMajorityMNet{\MNetExistsWeakCond(\Phi)}\beta_c$, and $\beta_c$ is not Majority optimal.
\end{proof}
\end{adjustwidth}

\medbreak

\noindent The goal of the next properties is to show that outcomes in $O_c^{\WitnessSet}$ are Majority optimal.
We prove this by showing that any outcome $\beta_c\in O_c^{\WitnessSet}$ is not majority dominated by any other outcome.
Note that, since $\MNetExistsWeakCond(\Phi)$ is a $6$\CPNet, for any two outcomes $\alpha$ and $\beta$, if $|\SetAgentsPrefMNet{\MNetExistsWeakCond(\Phi)}(\beta,\alpha)|\leq 3$, then $\beta\not\PrefMajorityMNet{\MNetExistsWeakCond(\Phi)}\alpha$.

\medbreak

\noindent\textbf{Property~\ref*{lemma_properties_of_ExistsWeakCondorcet_reduction}.(5).} \emph{Let $\beta_c\in O_c^{\WitnessSet}$ and $\gamma$ be two outcomes such that there is a feature $F\in(\SetFeat{W}\cup\SetFeat{P}\cup\SetFeat{D}\cup\SetFeat{A})$ for which $\gamma[F] \neq \beta_c[F]$.
Then, $\gamma\not\PrefMajorityMNet{\MNetExistsWeakCond(\Phi)}\beta_c$.}

\begin{adjustwidth}{1.5em}{}
\begin{proof}
Since $\gamma[F] \neq \beta_c[F]$, and $\beta_c[F] = f$ because $\beta_c \in O_c^{\WitnessSet}\subseteq O_c$, it must be the case that $\gamma[F] = \ol{f}$.
Therefore, by the definition of $\NetEWC3$, $\NetEWC4$, and $\NetEWC5$, $\gamma\not\PrefNet{\NetEWC3}\beta_c$, $\gamma\not\PrefNet{\NetEWC4}\beta_c$, and $\gamma\not\PrefNet{\NetEWC5}\beta_c$.
Thus, $|\SetAgentsPrefMNet{\MNetExistsWeakCond(\Phi)}(\gamma,\beta_c)|\leq 3$, and hence $\gamma\not\PrefMajorityMNet{\MNetExistsWeakCond(\Phi)}\beta_c$.
\end{proof}
\end{adjustwidth}

\medbreak

\noindent\textbf{Property~\ref*{lemma_properties_of_ExistsWeakCondorcet_reduction}.(6).} \emph{Let $\beta_c\in O_c^{\WitnessSet}$ and $\gamma$ be two outcomes such that there is a feature $F\in(\SetFeat{V'}\cup\SetFeat{B})$ for which $\gamma[F] \neq \beta_c[F]$.
Then, $\gamma\not\PrefMajorityMNet{\MNetExistsWeakCond(\Phi)}\beta_c$.}

\begin{adjustwidth}{1.5em}{}
\begin{proof}
Since $\gamma[F] \neq \beta_c[F]$, and $\beta_c[F] = f$ because $\beta_c \in O_c^{\WitnessSet}\subseteq O_c$, it must be the case that $\gamma[F] = \ol{f}$.
Therefore, by the definition of $\NetEWC1$, $\NetEWC2$, and $\NetEWC5$, $\gamma\not\PrefNet{\NetEWC1}\beta_c$, $\gamma\not\PrefNet{\NetEWC2}\beta_c$, and $\gamma\not\PrefNet{\NetEWC5}\beta_c$.
Thus, $|\SetAgentsPrefMNet{\MNetExistsWeakCond(\Phi)}(\gamma,\beta_c)|\leq 3$, and hence $\gamma\not\PrefMajorityMNet{\MNetExistsWeakCond(\Phi)}\beta_c$.
\end{proof}
\end{adjustwidth}

\medbreak

\noindent\textbf{Property~\ref*{lemma_properties_of_ExistsWeakCondorcet_reduction}.(7).} \emph{Let $\beta_c\in O_c^{\WitnessSet}$ and $\gamma$ be two outcomes such that $\gamma[U_1 U_2] \neq \beta_c[U_1 U_2] = u_1 u_2$.
Then, $\gamma\not\PrefMajorityMNet{\MNetExistsWeakCond(\Phi)}\beta_c$.}

\begin{adjustwidth}{1.5em}{}
\begin{proof}
Consider first net $\NetEWC3$.
Since $\beta_c \in O_c^{\WitnessSet}\subseteq O_c$, all features in $\SetFeat{V'}\cup\SetFeat{W}\cup\SetFeat{P}\cup\SetFeat{D}\cup\SetFeat{A}\cup\SetFeat{B}$ have a non-overlined value in $\beta_c$, and there is no pair of features $(V_i^T,V_i^F)$ such that $\beta_c[V_i^T V_i^F] = \ol{v_i^T}\ol{v_i^F}$.
Hence, the disjunctive interconnecting net embedded in $\NetEWC3$ cannot be exploited in any improving flipping sequence that aims at reaching an outcome in which $U_1$ or $U_2$ has an overlined value.
Hence, for any outcome $\gamma$ such that $\gamma[U_1] = \ol{u_1}$ or $\gamma[U_2] = \ol{u_2}$, $\gamma\not\PrefNet{\NetEWC3}\beta_c$.
By $\NetEWC3 = \NetEWC4$, for any outcome $\gamma$ such that $\gamma[U_1] = \ol{u_1}$ or $\gamma[U_2] = \ol{u_2}$, $\gamma\not\PrefNet{\NetEWC4}\beta_c$, as well.

Since $\gamma[U_1 U_2] \neq \beta_c[U_1 U_2] = u_1 u_2$, there are the three cases (1) $\gamma[U_1 U_2] = u_1 \ol{u_2}$, (2) $\gamma[U_1 U_2] = \ol{u_1} u_2$, or (3)~$\gamma[U_1 U_2] = \ol{u_1} \ol{u_2}$:

\medskip \noindent (1)
Let us focus on $\NetEWC1$.
First note that, for any outcome $\beta_c\in O_c$, $\beta_c$ assigns non-overlined values to all features in $\SetFeat{V'}\cup\SetFeat{B}$, and by Property~\ref*{lemma_properties_of_ExistsWeakCondorcet_reduction}.(7) we can assume that $\gamma$ assigns non-overlined values to features in $\SetFeat{V'}\cup\SetFeat{B}$.
Therefore, the part of net $\NetEWC1$ over feature sets $\SetFeat{V'}$ and $\SetFeat{B}$ does not play an active role in any improving flipping sequence (if exists) either from $\gamma$ to $\beta_c$, or from $\beta_c$ to $\gamma$ because, in $\NetEWC1$, features in $\SetFeat{V'}\cup\SetFeat{B}$ have no parents, and they have already their most preferred values in $\beta_c$ and $\gamma$.

Let $\sigma_X\in\WitnessSet$ be the (partial or complete) assignment over $X$ such that $\beta_c = \beta_{\sigma_X}$.
Since $\sigma_X\in\WitnessSet$, there is no extension $\sigma_X'$ of $\sigma_X$ to $X$ such that $(\forall Y)\lnot\phi(X/\sigma_X',Y)$ is not valid, (i.e., such that $\phi(X/\sigma_X',Y)$ is satisfiable).
Consider the non-quantified formula $\phi(X,Y)$.
If we consider the set $X\cup Y$ of all the Boolean variables in $\phi$, the assignment $\sigma_X$ is a partial assignment over $X\cup Y$.
Since $\phi(X/\sigma_X,Y)$ is not satisfiable, there is no extension of $\sigma_X$ to $X\cup Y$ satisfying $\phi$.
Therefore, by inspection of the proof of \cref{lemma_Pref_Net_CNF_Summ_new}, in $\NetEWC1$ there is no improving flipping sequence from $\beta_c$ to an outcome in which $U_2$ has an overlined value, hence $\gamma\not\PrefNet{\NetEWC1}\beta_c$.

From this, and from what we have already shown for nets $\NetEWC3$ and $\NetEWC4$, it follows that $|\SetAgentsPrefMNet{\MNetExistsWeakCond(\Phi)}(\gamma,\beta_c)|\leq 3$, and hence $\gamma\not\PrefMajorityMNet{\MNetExistsWeakCond(\Phi)}\beta_c$.

\medskip \noindent (2)
If we focus on $\NetEWC2$, a similar argument to the one used in~(1) for net $\NetEWC1$ shows that $\gamma\not\PrefNet{\NetEWC2}\beta_c$ (simply observe that the roles on $U_1$ and $U_2$ are exchanged).
Again, from this, and from what we have already shown for nets $\NetEWC3$ and $\NetEWC4$, it follows that $|\SetAgentsPrefMNet{\MNetExistsWeakCond(\Phi)}(\gamma,\beta_c)|\leq 3$, and hence $\gamma\not\PrefMajorityMNet{\MNetExistsWeakCond(\Phi)}\beta_c$.

\medskip \noindent (3)
By combining the discussions in~(1) and~(2), it is possible to show that $\gamma\not\PrefNet{\NetEWC1}\beta_c$ and $\gamma\not\PrefNet{\NetEWC2}\beta_c$ (and that $\gamma\not\PrefNet{\NetEWC3}\beta_c$ and $\gamma\not\PrefNet{\NetEWC4}\beta_c$).
Thus, $|\SetAgentsPrefMNet{\MNetExistsWeakCond(\Phi)}(\gamma,\beta_c)|\leq 2$, and hence $\gamma\not\PrefMajorityMNet{\MNetExistsWeakCond(\Phi)}\beta_c$.
\end{proof}
\end{adjustwidth}

\medbreak

\noindent\textbf{Property~\ref*{lemma_properties_of_ExistsWeakCondorcet_reduction}.(8).} \emph{Let $\beta_c\in O_c^{\WitnessSet}$ and $\gamma$ be two outcomes such that there is a feature $F\in\SetFeat{V}$ for which $\gamma[F] \neq \beta_c[F]$.
Then, $\gamma\not\PrefMajorityMNet{\MNetExistsWeakCond(\Phi)}\beta_c$.}

\begin{adjustwidth}{1.5em}{}
\begin{proof}
Let $\Change = \{F\in\SetFeat{V}\mid \beta_c[F] \neq \gamma[F]\}$ be the set of all variable features in $\SetFeat{V}$ changing value from $\beta_c$ to $\gamma$.
Let $\Up = \{F\in\SetFeat{V}\mid \beta_c[F] = f \land \gamma[F] = \ol{f}\}$ be the subset of $\Change$ containing the variable features in $\SetFeat{V}$ changing their value from non-overlined in $\beta_c$ to overlined in $\gamma$.
Let $\Down = \{F\in\SetFeat{V}\mid \beta_c[F] = \ol{f} \land \gamma[F] = f\}$ be the subset of $\Change$ containing the variable features in $\SetFeat{V}$ changing their value from overlined in $\beta_c$ to non-overlined in $\gamma$.
Clearly, $\Up$ and $\Down$ constitute a partition of $\Change$, and, since from the statement of this property, we assume that $\Change \neq \emptyset$, it must be the case that $(\Up \cup \Down) \neq \emptyset$.
There are the two cases (1) $\Up \neq \emptyset $, or (2) $\Up = \emptyset$:

\medskip \noindent (1)
Since there are variable features in $\SetFeat{V}$ changing their value from non-overlined in $\beta_c$ to overlined in $\gamma$, by the definition of $\NetEWC3$, $\NetEWC4$, and $\NetEWC5$, $\gamma\not\PrefNet{\NetEWC3}\beta_c$, $\gamma\not\PrefNet{\NetEWC4}\beta_c$, and $\gamma\not\PrefNet{\NetEWC5}\beta_c$.
Therefore, $|\SetAgentsPrefMNet{\MNetExistsWeakCond(\Phi)}(\gamma,\beta_c)|\leq 3$, and hence $\gamma\not\PrefMajorityMNet{\MNetExistsWeakCond(\Phi)}\beta_c$.

\medskip \noindent (2)
Since $\Up = \emptyset$ and $\Change \neq \emptyset$, it must be the case that $\Down \neq \emptyset$, and hence there are variable features in $\SetFeat{V}$ changing their value from overlined in $\beta_c$ to non-overlined in $\gamma$.
Moreover, by Property~\ref*{lemma_properties_of_ExistsWeakCondorcet_reduction}.(7), we can assume that $\gamma[U_1 U_2] = u_1 u_2$.
Consider net $\NetEWC1$.
Observe that feature $U_1$ in $\NetEWC1$ has no parents.
Hence, once $U_1$ is flipped from $u_1$ to $\ol{u_1}$, it cannot be flipped back.
Therefore, since $\beta_c[U_1] = \gamma[U_1] = u_1$, in any improving flipping sequence in $\NetEWC1$ from $\beta_c$ to $\gamma$ (if exists), feature $U_1$ cannot be flipped at all.
However, by the definition of the CP tables in $\NetEWC1$, when $U_1$ has value $u_1$, variable features can be flipped only from non-overlined to overlined.
Hence, from $\Down \neq \emptyset$ it follows that $\gamma\not\PrefNet{\NetEWC1}\beta_c$.
A similar argument (but focused on $U_2$) shows that $\gamma\not\PrefNet{\NetEWC2}\beta_c$ and $\gamma\not\PrefNet{\NetEWC6}\beta_c$.
Therefore, $|\SetAgentsPrefMNet{\MNetExistsWeakCond(\Phi)}(\gamma,\beta_c)|\leq 3$, and hence $\gamma\not\PrefMajorityMNet{\MNetExistsWeakCond(\Phi)}\beta_c$.
\end{proof}
\end{adjustwidth}

\medbreak

\noindent We are now ready to prove that $\Phi=(\exists X)(\forall Y)\lnot\phi(X,Y)$ is valid if and only if $\MNetExistsWeakCond(\Phi)$ has a majority optimal outcome.

\begin{itemize}
\item[$(\Rightarrow)$]
Assume that $\Phi$ is valid, hence there is an assignment $\sigma_X$ for the variables in $X$ such that $\sigma_X\in\WitnessSet$.
By Properties~\ref*{lemma_properties_of_ExistsWeakCondorcet_reduction}.(5), \ref*{lemma_properties_of_ExistsWeakCondorcet_reduction}.(6), \ref*{lemma_properties_of_ExistsWeakCondorcet_reduction}.(7), and~\ref*{lemma_properties_of_ExistsWeakCondorcet_reduction}.(8), $\beta_{\sigma_X}$ is Majority optimal in $\MNetExistsWeakCond(\Phi)$, and hence $\MNetExistsWeakCond(\Phi)$ has a majority optimal outcome.

\item[$(\Leftarrow)$]
Assume that $\Phi$ is not valid, hence there is no assignment in $\WitnessSet$, and so $O_c^{\WitnessSet}$ is empty.
By Properties~\ref*{lemma_properties_of_ExistsWeakCondorcet_reduction}.(1), \ref*{lemma_properties_of_ExistsWeakCondorcet_reduction}.(2), \ref*{lemma_properties_of_ExistsWeakCondorcet_reduction}.(3), and~\ref*{lemma_properties_of_ExistsWeakCondorcet_reduction}.(4), all outcomes in $O_d\cup O_c^{\NonWitnessSet}$ are not Majority optimal, and, since $O_c^{\WitnessSet}$ is empty, $\MNetExistsWeakCond(\Phi)$ does not have a majority optimal outcome.\qedhere
\end{itemize}
\end{proof}

\medbreak

\phantomsection\label{page_pointer:properties_of_IsStrongCondorcet_reduction_detailed}
\noindent\textbf{\cref{lemma_properties_of_IsStrongCondorcet_reduction}.}
\emph{Let $\Phi=(\exists X)(\forall Y)\lnot\phi(X,Y)$ be a quantified Boolean formula, where $\phi(X,Y)$ is a $3$CNF Boolean formula, defined over two disjoint sets $X$ and $Y$ of Boolean variables.
Then, $\Phi$ is valid if and only if $\MNetIsStrongCond(\Phi)$ does not have a majority optimum outcome.
In particular, when $\MNetIsStrongCond(\Phi)$ has a majority optimum outcome it is the outcome $\ol{\alpha}\in\OutMNet{\MNetIsStrongCond(\Phi)}$ assigning overlined values only to features $U_1$ and $U_2$.}

\begin{proof}
The intuition at the base of the proof is to put in relationship truth assignments over the variable set $X$ with outcomes of $\MNetIsStrongCond(\Phi)$.
In particular, given an assignment $\sigma_X$ over $X$, the associated outcome is $\beta_{\sigma_X}$, where $\sigma_X$ is encoded over the feature set $\SetFeat{V}$ in the usual way, and all other features have non-overlined values.
We show first that $\ol{\alpha}$ majority dominates any other outcome $\beta$ that is not in the form of a $\beta_{\sigma_X}$ outcome, and hence none of them is Majority optimum.
Moreover, for all such outcomes $\beta_{\sigma_X}$, $\beta_{\sigma_X}$ does not majority dominates $\ol{\alpha}$, and hence, again, none of them is Majority optimum.
So, only $\ol{\alpha}$ is candidate to be Majority optimum.
Then, we show that if there is an assignment $\sigma_X$ such that $(\forall Y)\lnot\phi(X/\sigma_X,Y)$ is valid, then $\ol{\alpha}\not\PrefMajorityMNet{\MNetIsStrongCond(\Phi)}\beta_{\sigma_X}$, and hence $\ol{\alpha}$ is not Majority optimum, which implies that $\MNetIsStrongCond(\Phi)$ does not have any majority optimum outcome.
On the other hand, if there is no assignment $\sigma_X$ for which $(\forall Y)\lnot\phi(X/\sigma_X,Y)$ is valid, then, for all the outcomes $\beta_{\sigma_X}$, $\ol{\alpha}\PrefMajorityMNet{\MNetIsStrongCond(\Phi)}\beta_{\sigma_X}$, which implies that $\ol{\alpha}$ is majority optimum.

To prove the statement of the lemma we have to analyze the majority dominance relationship between $\ol{\alpha}$ and the other outcomes.
To organize this task, let us define the following sets of outcomes:
\begin{itemize}
\item $O_d={O_d}'\cup {O_d}''\cup {O_d}'''$, where
\begin{itemize}
\item ${O_d}'=\{\beta\in\OutMNet{\MNetIsStrongCond(\Phi)}\mid (\exists F)(F\in(\SetFeat{V'}\cup\SetFeat{W}\cup\SetFeat{P}\cup\SetFeat{D}\cup\SetFeat{A}\cup\SetFeat{B})\land \beta[F]=\ol{f})\}$;
\item ${O_d}''=\{\beta\in\OutMNet{\MNetIsStrongCond(\Phi)}\mid \beta[U_1 U_2]\neq u_1 u_2\}$;
\item ${O_d}'''=\{\beta\in\OutMNet{\MNetIsStrongCond(\Phi)}\mid (\exists i)(\beta[V_i^T V_i^F]=\ol{v_i^T}\ol{v_i^F})\}$.
\end{itemize}
\item $O_c=\{\beta\in\OutMNet{\MNetIsStrongCond(\Phi)}\mid \beta\notin O_d\}$.
\end{itemize}
Clearly, $O_d$ and $O_c$ constitute a partition of $\OutMNet{\MNetIsStrongCond(\Phi)}$.
On the contrary, ${O_d}'$, ${O_d}''$, and ${O_d}'''$, do not constitute a partition of $O_d$ because they are not disjoint.
Note that $\ol{\alpha}\in {O_d}''$.
We show that $\ol{\alpha}$ is the only outcome that might be Majority optimum.
We do so by showing that (1) all outcomes different from $\ol{\alpha}$ in $O_d\cup O_c$, but a specific subset $S$ (whose detailed characterization will be given toward the end of the proof) of outcomes of $O_c$, are majority dominated by $\ol{\alpha}$, which means that all outcomes in $(O_d\cup O_c)\setminus(\{\ol{\alpha}\}\cup S)$ are Majority optimum; and that (2) all outcomes in $S$, which might be empty, neither majority dominate, nor are majority dominated by, $\ol{\alpha}$, and hence they are not Majority optimum.
Therefore, $\ol{\alpha}$ is Majority optimum if and only if $S$ is empty.

Note that, by the definition of $\NetISC2$, $\ol{\alpha}\PrefNet{\NetISC2}\beta$ for any outcome $\beta\neq\ol{\alpha}$, and hence, since $\MNetIsStrongCond(\Phi)$ is a $3$\CPNet, if $\ol{\alpha}\PrefNet{\NetISC1}\beta$ or $\ol{\alpha}\PrefNet{\NetISC3}\beta$, then $\ol{\alpha}\PrefMajorityMNet{\MNetIsStrongCond(\Phi)}\beta$.

\medbreak

\noindent\textbf{Property~\ref*{lemma_properties_of_IsStrongCondorcet_reduction}.(1).} \emph{Let $\beta'\in {O_d}'$ be an outcome.
Then, $\ol{\alpha}\PrefMajorityMNet{\MNetIsStrongCond(\Phi)}\beta'$.}

\begin{adjustwidth}{1.5em}{}
\begin{proof}
Since $\beta'\in {O_d}'$, let $F\in(\SetFeat{V'}\cup\SetFeat{W}\cup\SetFeat{P}\cup\SetFeat{D}\cup\SetFeat{A}\cup\SetFeat{B})$ be a feature such that $\beta'[F] = \ol{f}$.
Because $\ol{\alpha}[F] = f$, $\beta'\neq \ol{\alpha}$, and hence $\ol{\alpha}\PrefNet{\NetISC2}\beta'$.
Consider net $\NetISC3$.
The following is an improving flipping sequence from $\beta'$ to $\ol{\alpha}$ in $\NetISC3$, showing that $\ol{\alpha}\PrefNet{\NetISC3}\beta'$, and hence that $\ol{\alpha}\PrefMajorityMNet{\MNetIsStrongCond(\Phi)}\beta'$.
Since $\beta'[F] = \ol{f}$, by the definition of the disjunctive interconnecting net embedded in $\NetISC3$, we can flip in the proper order some features of the interconnecting net to their overlined values, until we flip its apex $B$.
Once the apex has an overlined value, we flip to their overlined value $U_1$ and then $U_2$ (if they do not have an overlined value already).
After this, we flip to their non-overlined values all features in $\SetFeat{V}\cup\SetFeat{W}\cup\SetFeat{P}\cup\SetFeat{D}\cup\SetFeat{A}$ having overlined values in $\beta'$.
Next, we flip to their non-overlined values all features in $\SetFeat{V'}$ having overlined values, and subsequently those in $\SetFeat{B}$ having overlined values.
\end{proof}
\end{adjustwidth}

\medbreak

\noindent\textbf{Property~\ref*{lemma_properties_of_IsStrongCondorcet_reduction}.(2).} \emph{Let $\beta''\in {O_d}''$ be an outcome different from $\ol{\alpha}$.
Then, $\ol{\alpha}\PrefMajorityMNet{\MNetIsStrongCond(\phi)}\beta''$.}

\begin{adjustwidth}{1.5em}{}
\begin{proof}
Since we are assuming $\beta''\neq \ol{\alpha}$, it holds that $\ol{\alpha}\PrefNet{\NetISC2}\beta''$.
There are the three cases (1) $\beta''[U_1 U_2] = \ol{u_1} \ol{u_2}$, or (2)~$\beta''[U_1 U_2] = \ol{u_1} u_2$, or (3) $\beta''[U_1 U_2] = u_1 \ol{u_2}$:

\medskip \noindent (1)
Let us focus on net $\NetISC3$.
The following is an improving flipping sequence from $\beta''$ to $\ol{\alpha}$ in $\NetISC3$, showing that $\ol{\alpha}\PrefNet{\NetISC3}\beta''$, and hence that $\ol{\alpha}\PrefMajorityMNet{\MNetIsStrongCond(\Phi)}\beta''$.
We can flip all features, but $U_1$ and $U_2$, to their non-overlined values in a proper sequence which could be $\SetFeat{V}$, $\SetFeat{V'}$, $\SetFeat{W}$, $\SetFeat{P}$, $\SetFeat{D}$, $\SetFeat{A}$, and $\SetFeat{B}$.

\medskip \noindent (2)
Let us focus again on net $\NetISC3$.
The following is an improving flipping sequence from $\beta''$ to $\ol{\alpha}$ in $\NetISC3$, showing that $\ol{\alpha}\PrefNet{\NetISC3}\beta''$, and hence that $\ol{\alpha}\PrefMajorityMNet{\MNetIsStrongCond(\Phi)}\beta''$.
Since $\beta''[U_1 U_2] = \ol{u_1}u_2$, we can flip $U_2$ from $u_2$ to $\ol{u_2}$.
Now we are again in the case in which $\beta''[U_1 U_2] = \ol{u_1}\ol{u_2}$.
Hence, there is an improving flipping sequence to $\ol{\alpha}$ as shown in~(1).

\medskip \noindent (3)
Let us consider net $\NetISC1$.
The following is an improving flipping sequence from $\beta''$ to $\ol{\alpha}$ in $\NetISC1$, showing that $\ol{\alpha}\PrefNet{\NetISC1}\beta''$, and hence that $\ol{\alpha}\PrefMajorityMNet{\MNetIsStrongCond(\Phi)}\beta''$.
We can flip $U_1$ from $u_1$ to $\ol{u_1}$.
Then, in the order $\SetFeat{V}$, $\SetFeat{W}$, $\SetFeat{P}$, $\SetFeat{D}$, $\SetFeat{A}$, $\SetFeat{V'}$, and $\SetFeat{B}$, we flip features having overlined values  to their non-overlined values.
\end{proof}
\end{adjustwidth}

\medbreak

\noindent\textbf{Property~\ref*{lemma_properties_of_IsStrongCondorcet_reduction}.(3).} \emph{Let $\beta'''\in {O_d}'''$ be an outcome.
Then, $\ol{\alpha}\PrefMajorityMNet{\MNetIsStrongCond(\Phi)}\beta'''$.}

\begin{adjustwidth}{1.5em}{}
\begin{proof}
Since $\beta'''\in {O_d}''''$, there is a pair of features $(V_i^T,V_i^F)$ such that $\beta'''[V_i^T V_i^F] = \ol{v_i^T}\ol{v_i^F}$.
Because $\ol{\alpha}[V_i^T V_i^F] = v_i^T v_i^F$, $\beta'''\neq \ol{\alpha}$, and hence $\ol{\alpha}\PrefNet{\NetISC2}\beta'''$.
Consider net $\NetISC3$.
We show that there is an improving flipping sequence from $\beta'''$ to $\ol{\alpha}$ in $\NetISC3$, proving that $\ol{\alpha}\PrefNet{\NetISC3}\beta'''$, and hence that $\ol{\alpha}\PrefMajorityMNet{\MNetIsStrongCond(\Phi)}\beta'''$.

Since $\beta'''[V_i^T V_i^F] = \ol{v_i^T}\ol{v_i^F}$, if $\beta'''[V_i'] = v_i'$, we can flip the value of $V_i'$ from $v_i'$ to $\ol{v_i'}$.
This bring us in the case in which there is a feature $F$ in $\SetFeat{V'}$ with overlined value.
We have already shown in the proof of Property~\ref*{lemma_properties_of_IsStrongCondorcet_reduction}.(1) that in $\NetISC3$ there exists an improving flipping sequence from this outcome to $\ol{\alpha}$.
\end{proof}
\end{adjustwidth}

\medbreak

\noindent
The three properties above show that any outcome $\beta\in O_d$ that is different from $\ol{\alpha}$ is majority dominated by $\ol{\alpha}$.
This proves that any such outcome $\beta$ is not Majority optimum, and that $\ol{\alpha}$ majority dominates all of them.

Let us now consider outcomes in $O_c$.
Observe that, since $O_d$ and $O_c$ are disjoint, all outcomes $\beta_c\in O_c$ are characterized by having all features in $\SetFeat{V'}\cup\SetFeat{W}\cup\SetFeat{P}\cup\SetFeat{D}\cup\SetFeat{A}\cup\SetFeat{B}$ with non-overlined values, $\beta_c[U_1 U_2] = u_1 u_2$, and, for all variables $x_i\in X$, $\beta_c[V_i^T V_i^F] \neq \ol{v_i^T} \ol{v_i^F}$.

Given an assignment $\sigma_X$ for the variables in $X$, we define $\beta_{\sigma_X}\in O_c$ as the outcome encoding $\sigma_X$ over the feature set $\SetFeat{V}$ as usual.
Let us denote by $\NonWitnessSet_c$ the set of all \emph{complete} assignments $\sigma_X$ over $X$ such that $(\forall Y)\lnot\phi(X/\sigma_X,Y)$ is \emph{not} valid (i.e., such that $\phi(X/\sigma_X,Y)$ is satisfiable).
Let $\NonWitnessSet$ be the set of all (partial or complete) assignments $\sigma_X$ over $X$ such that there is an extension of $\sigma_X$ to $X$ belonging to $\NonWitnessSet_c$, and let $\WitnessSet$ be the set of all (partial or complete) assignments over $X$ not belonging to $\NonWitnessSet$.
Remember that if $\sigma_X$ is a complete assignment over $X$, then $\sigma_X$ itself is the unique extension of $\sigma_X$ to $X$.
Given the above definitions, $O_c^{\WitnessSet}=\{\beta_{\sigma_X}\in O_c\mid \sigma_X\in\WitnessSet\}$, and $O_c^{\NonWitnessSet}=\{\beta_{\sigma_X}\in O_c\mid \sigma_X\in\NonWitnessSet\}$ constitute a partition of $O_c$.
This implies that, if $\beta\in O_c$ is an outcome, then there is a (partial or complete) assignment $\sigma_X$ over the variables in $X$ such that $\beta = \beta_{\sigma_X}$.

We show that all outcomes in $O_c^{\NonWitnessSet}$ are  majority dominated by $\ol{\alpha}$, and that all outcomes in $O_c^{\WitnessSet}$ are not majority dominated by $\ol{\alpha}$ and do not majority dominate $\ol{\alpha}$ (i.e., $O_c^{\WitnessSet}$ is the set $S$ mentioned earlier).

For the following two properties it is useful to note that, for any outcome $\beta_c\in O_c$, $\beta_c$ assigns non-overlined values to all features in $\SetFeat{V'}\cup\SetFeat{B}$, and also $\ol{\alpha}$ assigns non-overlined values to features in $\SetFeat{V'}\cup\SetFeat{B}$.
Therefore, the part of net $\NetISC1$ over feature sets $\SetFeat{V'}$ and $\SetFeat{B}$ does not play an active role in any improving flipping sequence (if exists) either from $\ol{\alpha}$ to $\beta_c$, or from $\beta_c$ to $\ol{\alpha}$ because, in $\NetISC1$, features in $\SetFeat{V'}\cup\SetFeat{B}$ have no parents, and they have already their most preferred values in $\ol{\alpha}$ and $\beta_c$.

\medbreak

\noindent\textbf{Property~\ref*{lemma_properties_of_IsStrongCondorcet_reduction}.(4).} \emph{Let $\beta_c\in O_c^{\NonWitnessSet}$ be an outcome.
Then, $\ol{\alpha}\PrefMajorityMNet{\MNetIsStrongCond(\Phi)}\beta_c$.}

\begin{adjustwidth}{1.5em}{}
\begin{proof}
Let $\sigma_X\in\NonWitnessSet$ be the (partial or complete) assignment over $X$ such that $\beta_c = \beta_{\sigma_X}$.

Let us focus on net $\NetISC1$.
Consider now the non-quantified formula $\phi(X,Y)$.
If we consider the set $X\cup Y$ of all the Boolean variables in $\phi$, the assignment $\sigma_X$ is a partial assignment over $X\cup Y$.
Since $\sigma_X\in\NonWitnessSet$, there is an extension $\sigma_X'$ of $\sigma_X$ to $X$ such that $(\forall Y)\lnot\phi(X/\sigma_X',Y)$ is not valid (i.e., such that $\phi(X/\sigma_X',Y)$ is satisfiable), and hence there is an extension of $\sigma_X$ to $X\cup Y$ satisfying $\phi$.
Therefore, by \cref{lemma_Pref_Net_CNF_Summ_new}, $\ol{\alpha}\PrefNet{\NetISC1}\beta_c$.
Since $\ol{\alpha}\in {O_d}''$ and $\beta_c\in {O_c}$, $\beta_c\neq \ol{\alpha}$, and hence $\ol{\alpha}\PrefNet{\NetISC2}\beta_c$.
Thus, $\ol{\alpha}\PrefMajorityMNet{\MNetIsStrongCond(\Phi)}\beta_c$.
\end{proof}
\end{adjustwidth}

\medbreak

\noindent\textbf{Property~\ref*{lemma_properties_of_IsStrongCondorcet_reduction}.(5).} \emph{Let $\beta_c\in O_c^{\WitnessSet}$ be an outcome.
Then, $\ol{\alpha}\not\PrefMajorityMNet{\MNetIsStrongCond(\Phi)}\beta_c$ and $\beta_c\not\PrefMajorityMNet{\MNetIsStrongCond(\Phi)}\ol{\alpha}$.}

\begin{adjustwidth}{1.5em}{}
\begin{proof}
Let $\sigma_X\in\WitnessSet$ be the (partial or complete) assignment over $X$ such that $\beta_c = \beta_{\sigma_X}$.
Since $\sigma_X\in\WitnessSet$, there is no extension $\sigma_X'$ of $\sigma_X$ to $X$ such that $(\forall Y)\lnot\phi(X/\sigma_X',Y)$ is not valid, (i.e., such that $\phi(X/\sigma_X',Y)$ is satisfiable).

Now consider net $\NetISC1$.
We claim that $\beta_c\IncompNet{\NetISC1}\ol{\alpha}$.
Consider the non-quantified formula $\phi(X,Y)$.
If we consider the set $X\cup Y$ of all the Boolean variables in $\phi$, the assignment $\sigma_X$ is a partial assignment over $X\cup Y$.
Since $\phi(X/\sigma_X,Y)$ is not satisfiable, there is no extension of $\sigma_X$ to $X\cup Y$ satisfying $\phi$.
Therefore, by \cref{lemma_Pref_Net_CNF_Summ_new}, $\beta_c\IncompNet{\NetISC1}\ol{\alpha}$.

Now, since $\ol{\alpha}\in {O_d}''$ and $\beta_c\in {O_c}$, $\beta_c\neq \ol{\alpha}$, and hence $\ol{\alpha}\PrefNet{\NetISC2}\beta_c$.
To conclude, let us now focus on net $\NetISC3$.
By $\beta_c\in O_c$, all features in $\SetFeat{B}$ have non-overlined values, $\beta_c[U_1 U_2] = u_1 u_2$, and there is no pair of features $(V_i^T,V_i^F)$ such that $\beta_c[V_i^T V_i^F] = \ol{v_i^T}\ol{v_i^F}$.
Since in $\NetISC3$ feature $U_1$ is attached to the apex $B$ of the interconnecting net, and $U_1$ can be flipped from $u_1$ to $\ol{u_1}$ only when the apex $B$ of the interconnecting net has an overlined value, there is no improving flipping sequence from $\beta_c$ to $\ol{\alpha}$, and hence $\ol{\alpha}\not\PrefNet{\NetISC3}\beta_c$.

To summarize, we showed that $\beta_c\IncompNet{\NetISC1}\ol{\alpha}$, $\ol{\alpha}\PrefNet{\NetISC2}\beta_c$, and $\ol{\alpha}\not\PrefNet{\NetISC3}\beta_c$.
Therefore, $\ol{\alpha}\not\PrefMajorityMNet{\MNetIsStrongCond(\Phi)}\beta_c$, because $|\SetAgentsPrefMNet{\MNetIsStrongCond(\Phi)}(\ol{\alpha},\beta_c)| < 2$, and $\beta_c\not\PrefMajorityMNet{\MNetIsStrongCond(\Phi)}\ol{\alpha}$, because $|\SetAgentsPrefMNet{\MNetIsStrongCond(\Phi)}(\beta_c,\ol{\alpha})| < 2$.
\end{proof}
\end{adjustwidth}

\medbreak

\noindent We are now ready to prove that $\Phi=(\exists X)(\forall Y)\lnot\phi(X,Y)$ is valid if and only if $\ol{\alpha}$ is not majority optimum in $\MNetIsStrongCond(\Phi)$.

\begin{itemize}
\item[$(\Rightarrow)$]
Assume that $\Phi$ is valid.
By Properties~\ref*{lemma_properties_of_IsStrongCondorcet_reduction}.(1), \ref*{lemma_properties_of_IsStrongCondorcet_reduction}.(2), \ref*{lemma_properties_of_IsStrongCondorcet_reduction}.(3), and~\ref*{lemma_properties_of_IsStrongCondorcet_reduction}.(4), all outcomes in $O_d\cup O_c^{\NonWitnessSet}$ different from $\ol{\alpha}$ are not Majority optimum.
Moreover, from Property~\ref*{lemma_properties_of_IsStrongCondorcet_reduction}.(5), all outcomes in $O_c^{\WitnessSet}$ are not majority optimum, and $\ol{\alpha}$ does not majority dominate outcomes in $O_c^{\NonWitnessSet}$.
Hence, also $\ol{\alpha}$ is not Majority optimum.
Therefore, in $\MNetIsStrongCond(\Phi)$ there is no majority optimum outcome.

\item[$(\Leftarrow)$]
Assume that $\Phi$ is not valid, hence there is no assignment in $\WitnessSet$, and so $O_c^{\WitnessSet}$ is empty.
By Properties~\ref*{lemma_properties_of_IsStrongCondorcet_reduction}.(1), \ref*{lemma_properties_of_IsStrongCondorcet_reduction}.(2), \ref*{lemma_properties_of_IsStrongCondorcet_reduction}.(3), and~\ref*{lemma_properties_of_IsStrongCondorcet_reduction}.(4), all outcomes in $O_d\cup O_c^{\NonWitnessSet}$ different from $\ol{\alpha}$ are majority dominated by $\ol{\alpha}$.
Since $O_c^{\WitnessSet}$ is empty, $\ol{\alpha}$ majority dominates all other outcomes, which implies that in $\MNetIsStrongCond(\Phi)$ there is a majority optimum outcome, which is $\ol{\alpha}$.\qedhere
\end{itemize}
\end{proof}

\end{document}